\documentclass[twoside]{article}

\usepackage[preprint]{aistats2026}
\usepackage{hyperref}
\usepackage{booktabs}
\usepackage{enumerate}
\usepackage[round]{natbib}

\usepackage{bbm}
\usepackage{amsfonts}
\usepackage{amsthm}
\usepackage{color}
\usepackage{graphicx}
\usepackage{algorithm}
\usepackage{algpseudocode}

\newtheorem{theorem}{Theorem}[section]
\newtheorem{corollary}{Corollary}
\newtheorem{lemma}[theorem]{Lemma}
\newtheorem{proposition}{Proposition}
\newtheorem{remark}{Remark}
\newtheorem{definition}{Definition}[section]

\def\EE {{\mathbb E}}  
\def\RR {{\mathbb R}}

\def\rank {\mathrm{rank}}
\def\Pr {\mathrm{Pr}}

\def\bone {\mathbbm{1}}

\def\bk {\boldsymbol{k}}
\def\by {\boldsymbol{y}}
\def\bx {\boldsymbol{x}}
\def\bc {\boldsymbol{c}}
\def\bb {\boldsymbol{b}}

\def\ba {\boldsymbol{a}}
\def\bu {\boldsymbol{u}}
\def\be {\boldsymbol{e}}
\def\bv {\boldsymbol{v}}

\begin{document}

%

%

\twocolumn[

\aistatstitle{Differential Privacy of Gaussian Process Posterior Sampling}

\aistatsauthor{ Tomasz Maciazek}

\aistatsaddress{School of Mathematics, University of Bristol} ]

\begin{abstract}
We study the privacy of releasing functional posterior sample paths from a Gaussian process (GP) when the entire training set including covariates and responses is private. Unlike standard differential-privacy (DP) mechanisms that inject external noise, posterior sampling is intrinsically random and we show that this randomness provides useful privacy guarantees. We derive R\'enyi-DP guarantees separating privacy leakage through the posterior mean from a distinct channel induced by the data-dependent posterior covariance. The analysis identifies effective ridge regularisation and covariance scale as the principal privacy-controlling quantities and yields sharper guarantees in several regimes of practical interest as well as extensions to repeated and adaptive releases. Membership inference attacks confirm the predicted dependence on regularisation, covariance scale and the number of released paths. Utility experiments on downstream posterior sampling tasks identify noisy observation regimes where privacy-compatible regularisation preserves useful samples. Finally we identify large-data asymptotic regime in which the privacy parameter and posterior mean-square risk vanish simultaneously, yielding privacy for free. Together, these results provide a comprehensive characterisation of privacy and utility of GP posterior sampling.
\end{abstract}

\section{INTRODUCTION}\label{sec:intro}
Protecting individual training examples is a central concern in modern machine learning. Released models can leak membership, memorised training fragments, or partial reconstructions of sensitive examples and attributes \citep{ShokriStronatiSongShmatikov2017,mia_first_principles,CarliniLiuErlingssonEtAl2021,BalleCherubinHayes2022,HaimVardiYehudaiEtAl2022,property_inference}. Standard DP methods mitigate this by injecting randomness during training or release \citep{dwork07,rdp,hall13,AbadiChuGoodfellowEtAl2016}. However, recent work shows that randomness and regularisation already present in learning algorithms can also shape privacy leakage: regularisation can reduce membership-inference risk \citep{pmlr-v206-tan23b}, initialisation affects privacy--utility tradeoffs in overparameterised networks \citep{ye2023initialization}, and reconstruction attack success depends strongly on initialisation and regularisation \citep{HaimVardiYehudaiEtAl2022,BalleCherubinHayes2022,BuzagloHaimYehudaiEtAl2023}. This shows that intrinsic model properties may inhibit training-data inference even without explicit external DP noise.

We study this effect in Gaussian-process (GP) posterior sampling where randomness is intrinsic to the model. Given a training set $D$ the mechanism releases a function $f_D\sim\Pi_D=\mathcal{GP}(\mu_D,\sigma^2 k_D)$ from the standard GP posterior \citep{RasmussenWilliams2006}. Both $\mu_D$ and $k_D$ are data-dependent, so privacy leakage arises through both the posterior mean and covariance. The covariance leakage channel is largely underexplored in the private GP-literature -- it disappears under the common public-covariate assumption, but here it is central to our analysis. We do not treat $f_D$ as a private surrogate for $\mu_D$, but rather as the released Bayesian object itself. This provides a canonical setting for studying how posterior randomness and regularisation limit training-data inference. Our privacy guarantees extend to multiple draws by composition and to any downstream output computed from them by post-processing such as a Thompson-sampling action or a released pollution map in geospatial analysis. Posterior function draws are standard computational objects in Bayesian optimisation and GP-Thompson sampling \citep{pmlr-v70-chowdhury17a,HernandezLobatoRequeimaPyzer-KnappAspuru-Guzik2017} including parallelised variants \citep{pmlr-v84-kandasamy18a}. Posterior draws are also used in Bayesian workflows beyond optimisation where one repeatedly simulates from the posterior distribution and evaluates discrepancies or other posterior predictive checks \citep{GelmanMengStern1996,GabrySimpsonVehtariEtAl2019}. Posterior sample draws of latent Gaussian field models used in spatial statistics and environmental modelling for estimating excursion sets, exceedance probabilities or contour uncertainty given some noisy latent field observations are likewise basic inferential objects \citep{BolinLindgren,FrenchHoeting2016,DiggleRibeiro2007}. More broadly, related privacy considerations may arise in distributed Bayesian computation where posterior draws are communicated across parties \citep{bayes_big_data,parallel_mcmc}. Our work introduces a new route to privacy accounting for this broad class of workflows for which there was previously no direct DP guarantee for the released posterior sample paths.

\paragraph{Threat Model and Main Contributions}
The adversary observes $L\ge 1$ released posterior sample paths drawn from the GP posterior conditioned on a private training set $D$. The release is function-valued i.e., can be evaluated at arbitrary inputs. The posterior mean
$\mu_D$ and covariance $k_D$ are not directly released. We treat the kernel and hyperparameters as fixed and public and chosen independently of the private dataset.
\begin{enumerate}
\item We derive the first DP bounds for functional GP posterior sample-path release identifying two separate privacy leakage channels coming from posterior mean and data-dependent posterior covariance (Theorem~\ref{thm:main_rdp}). To our knowledge, the covariance-channel has not previously been fully characterised in private GP inference. 
\item We identify the effective ridge regularisation and covariance scale ($r$, $\sigma$) as the key privacy-controlling quantities and derive sharper DP bounds specialized to certain natural regimes, e.g. under RKHS-response or data-separation assumption. We also quantify privacy composition across multiple released posterior sample paths including specialisation to GP-Thompson sampling. 
\item We test the theory experimentally. Membership-inference success follows the predicted dependence on $r$ and $\sigma$ while utility experiments show that in noisy regimes privacy-compatible regularisation provides useful posterior samples with modest utility loss. We also prove that when $|D|\to\infty$ this utility tradeoff vanishes under mild regularity assumptions.
\end{enumerate}
We state our bounds in terms of R\'{e}nyi-DP guarantees for a single posterior draw which can be conveniently composed to bounds for multiple posterior draws and then translated to the approximate-DP bounds i.e., $(\varepsilon,\delta)$-DP, by the standard formulae. 

For posterior sample-path release, utility is inherently downstream-task dependent: the released function is useful only through post-processing. This is especially relevant in sensing applications where the measurement locations are themselves sensitive and the goal is to release a spatial-based decision  without revealing sensor locations. We study this through excursion-set estimation in Section~\ref{sec:excursion}. The experiments show a clear privacy-utility transition: in noisy-data regimes the utility-optimal posterior is already substantially regularised, placing the mechanism closer to the strong-privacy region of our bounds. In these regimes meaningful intrinsic privacy comes with modest downstream utility loss, whereas low-noise regimes expose a sharper cost of private calibration. The tradeoff effect is less pronounced for larger training sets and typically vanishes as $|D|\to\infty$ enabling privacy ``for free''.

\paragraph{Prior Work}
Posterior sampling has been linked to privacy in broader Bayesian settings. \citet{pmlr-v37-wangg15} show that one posterior sample can be differentially private ``for free'' under bounded log-likelihood assumptions, while \citet{JMLR:v18:15-257} develop a general Bayesian posterior-sampling framework for DP. Neither result directly applies to standard GP regression: the bounded log-likelihood condition of \citet{pmlr-v37-wangg15} fails because the latent field is unbounded, while \citet{JMLR:v18:15-257} require the log-likelihood to be uniformly Lipschitz in the data over all parameter values, which here are latent functions. \citet{GeumlekEtAl2017RDPPosteriorSampling} treat finite-dimensional conjugate exponential families and bounded-link GLMs, neither of which covers function-valued GP regression.

Prior GP-specific work does not analyse our release mechanism. \citet{pmlr-v84-smith18a,SmithEtAl2021DPSparseGP} only privatise the posterior mean by adding calibrated noise and assume public $X$ and private $y$. \citet{honkela2021gaussianprocessesdifferentialprivacy} protect both $X$ and $y$, but by explicitly perturbing a sparse variational GP approximation. Their guarantees concern a different problem of releasing a modified and privatised approximate posterior, not the intrinsic privacy of sampling from the exact unmodified posterior. \citet{ye2024leaveoneout} study membership leakage through KL-divergence of posterior marginals at a single test point which does not establish full-function DP guarantees. To our knowledge, we provide the first intrinsic DP guarantees for sampling an entire function from an exact GP posterior, quantifying both mean and covariance leakage.


\section{PROBLEM SETUP}\label{sec:setup}
We consider the setup where one aims to model an unknown function $f_*$ using a mean-zero GP prior with the covariance function $\sigma^2k:\ \Omega_X\times \Omega_X\to \RR$ where $\Omega_X\subset \mathbb{R}^d$ is the covariate domain and $\sigma>0$. Given noisy observations $\by=(y_1,\dots,y_n)^T$ at covariates $X=(\bx_1,\dots,\bx_n)$ we write $D=(X,\by)$. When conditioned on $D$ the posterior GP distribution is $\Pi_D=\mathcal{GP}(\mu_D,\sigma^2 k_D)$ where
\begin{align*}
\begin{split}
\mu_D\left(\bx\right) & =\bk_X(\bx)^T\left(K_{XX}^{(r)}\right)^{-1}\,\by, \quad r^2:=\lambda^2/\sigma^2
\\
k_D(\bx,\bx') & =k\left(\bx,\bx'\right)-\bk_X(\bx)^T\left(K_{XX}^{(r)}\right)^{-1}\bk_X(\bx')
\end{split}
\end{align*}
with $K_{XX}^{(r)}:=K(X,X)+r^2I$, $\bk_X(\bx):=K(X,\bx)$. Here, $\lambda^2$ is the noise variance assumed in the GP model, $\sigma^2$ is the kernel scale and $r^2$ is the \emph{effective ridge} sometimes referred to as the noise-to-signal variance ratio. Importantly, changing $\sigma$ while keeping $r$ fixed does not change the predictive mean $\mu_D$ -- it only rescales the posterior covariance. We separate out the kernel scale $\sigma^2$ since it is the parameter that directly controls the posterior randomness. When $k$ has bounded diagonal we rescale $k$ so that $\sup_{\bx\in\Omega_X}k(\bx,\bx)=1$ in which case $\sigma^2$ has the standard interpretation as the prior variance scale.

We denote by $D\sim D'$ the replace-one neighbouring relation between datasets. Consider the random mechanism that assigns to $D$ a single sample path from the $GP$-posterior 
\begin{equation}\label{def:dp_mechanism}
\mathcal{M}:\ D\mapsto f_D\sim \mathcal{GP}\left(\mu_D,\sigma^2 k_D\right).
\end{equation}
In this paper we determine approximate-DP guarantees of this mechanism. Since \(\mathcal{M}(D)\) is function-valued, we use the formalism of differential privacy for function-valued mechanisms developed in \cite{hall13}. More precisely, we view $f_D$ as a random element of the measurable space $(\mathcal F,\mathcal A)
=
\left(\mathbb R^{\Omega_X},\mathcal C\right)
$,
where $\mathbb R^{\Omega_X}$ is the space of all real-valued functions on $\Omega_X$, and $\mathcal C$ is the canonical cylinder $\sigma$-algebra which makes all evaluation maps $f\mapsto f(\bx)$ measurable. The mechanism $\mathcal{M}$ is $(\varepsilon,\delta)$-DP if for all $D\sim D'$ and all $S\in\mathcal C$
\begin{equation}\label{def:fun_dp}
\Pr\left[f_D\in S\right]\leq e^\varepsilon\,\Pr\left[f_{D'}\in S\right]+\delta.
\end{equation}
Crucially, it is enough to verify inequality \eqref{def:fun_dp} on finite evaluations, namely that 
\[
\Pr\left[f_D(X_T)\in S\right]\leq e^\varepsilon\,\Pr\left[f_{D'}(X_T)\in S\right]+\delta
\]
 for every finite test set $X_T$ and every $S$ being a measurable subset of $\RR^{|X_T|}$ \citep[see][Proposition 6]{hall13}. To prove the $(\epsilon,\delta)$-DP guarantee for finite sample-path evaluations we turn to R\'{e}nyi-DP (RDP) machinery. The mechanism $D\mapsto f_D(X_T)$ is $(\alpha,\varepsilon_\alpha)$-RDP \citep{rdp} iff for all $D\sim D'$ we have 
 \[
 D_\alpha\left(\Pi_D(X_T)\|\Pi_{D'}(X_T)\right)\leq \varepsilon_\alpha
 \] 
 where $\Pi_D(X_T)$ is the GP posterior distribution of $f_D(X_T)$ and $D_\alpha$ is the R\'{e}nyi divergence of order $\alpha>1$.  Having proved $(\alpha,\varepsilon_\alpha)$-RDP uniformly over all finite test sets $X_T$, we apply the RDP-to-DP conversion \citep[see][Theorem 20]{pmlr-v108-balle20a} for any $0<\delta<1$.
 
The RDP guarantees are also convenient when considering the mechanism that
releases $L$ independently sampled posterior draws $\mathcal{M}_L:\ D\mapsto \left(f_D^{(1)},\dots,f_D^{(L)}\right)$. If a single posterior draw satisfies $(\alpha,\varepsilon_\alpha)$-RDP uniformly over finite test sets, then by additive composition of RDP
\citep[Proposition~1]{rdp} $\mathcal{M}_L$ satisfies
\((\alpha,L\varepsilon_\alpha)\)-RDP. Consequently, $\mathcal{M}_L$ is 
\begin{equation}\label{eq:rdp_composition}
\left(
L\varepsilon_\alpha+\log\frac{\alpha-1}{\alpha}
 -\frac{\log\delta+\log\alpha}{\alpha-1},
\delta
\right)\text{-DP}.
\end{equation}
by Theorem 20 in \citep{pmlr-v108-balle20a}. This accounting applies \textit{mutatis mutandis} to adaptive releases where
the round-$l$ training set depends on previous draws as in Thompson sampling, see Appendix \ref{app:public_covariates}. The final approximate-DP guarantee is obtained by optimising the converted bound over the admissible orders $\alpha>1$ for which the bound is finite. This optimisation can lead to sub-linear growth of the final approximate-DP bound in $L$, see Section \ref{sec:mia}. The full function-valued release guarantee also permits lazy evaluation where the mechanism can answer $f_D(X_1)$ and later sample $f_D(X_2)\mid f_D(X_1)$ for an adaptively chosen $X_2$. The output is the same as evaluation of one fixed posterior draw, so no extra privacy budget is consumed. 

\section{R\'{E}NYI-DP BOUNDS FOR GP POSTERIOR SAMPLING}\label{sec:bounds}
As in standard additive-noise DP mechanisms the privacy guarantees depend on a problem-specific sensitivity of the non-private output which must be bounded for the release rule at hand \citep[see e.g.][Proposition~3]{hall13}. However, for the GP
posterior-sampling mechanism, the released object is not a deterministic posterior mean plus data-independent external noise, but a draw from the full
data-dependent posterior distribution. As we show, the relevant quantities are then the posterior variance scale and the common-core posterior mean sensitivity defined below.
\begin{definition}
\label{def:Vn_Hn}
Let $D\sim D'$ and let $D_-=D\cap D'$ denote the common core and let
$X_-$ be the corresponding covariate set. Let $k_{D_-}$ be the posterior covariance kernel obtained after conditioning on $D_-$, as defined in Section \ref{sec:setup}. We define the posterior variance scale by
\[
V_n(r):=
\sup_{|X_-|=n-1}\,\sup_{x\in\Omega_X} k_{D_-}(x,x),
\]
and the common-core posterior mean sensitivity by
\[
\Delta_n(r):=
\sup_{D\sim D'} \|\mu_D-\mu_{D'}\|_{\mathcal H_-},
\]
where $\mathcal H_{-}$ denotes the reproducing kernel Hilbert space (RKHS) associated with the kernel $k_{D_-}$.
\end{definition}
For the RDP bounds presented in Theorem \ref{thm:main_rdp} below it is crucial to upper-bound $V_n(r)$ and $\Delta_n(r)$ as tightly as possible. 

\paragraph{Bounding $V_n(r)$.} In many cases $V_n(r)$ can be conveniently upper-bounded or even computed exactly in closed form. When $k(\bx,\bx)\leq 1$ for all $\bx$, then $V_n(r)\leq 1$. If additionally \(k(\bx,\bx')\ge0\) for all \(\bx,\bx'\) then
\[
V_n(r)\leq  \bar V_n(r):=1-\kappa^2\frac{n-1}{n-1+r^2},
\]
where $\kappa=\inf_{\bx,\bx'\in\Omega_X}k(\bx,\bx')$. If we also have $k(\bx,\bx)=1$ for all $\bx$, then $V_n(r)=  \bar V_n(r)$, see Proposition \ref{prop:general_positive_kernel_exact_Vn} in Appendix \ref{app:Vn_Deltan_bounds}. This applies, for example, to standard normalised Mat\'{e}rn and RBF kernels.

\paragraph{Bounding $\Delta_n(r)$.} The problem of finding bounds for $\Delta_n(r)$ is more subtle. Let us consider kernels with bounded diagonal and responses bounded in absolute value by a constant $M_Y>0$. Then, we have the generic bound (see Lemma \ref{lemma:boundedResp_Delta_bound})
\[
\Delta_n(r)\leq  \frac{M_Y}{r}\left(1+\frac{\sqrt{n-1}}{2r}\right)=\mathcal O(\sqrt{n}).
\]
Without additional information about the kernel $k$ one cannot improve the $\mathcal O(\sqrt{n})$-growth of $\Delta_n(r)$. Indeed, in Lemma \ref{lemma:Delta_sqrt_growth_example} we design a specific kernel for which $\Delta_n$ grows exactly at the $\sqrt{n}$-rate. On the other hand, there exist typical situations where $\Delta_n(r)$ is upper-bounded by a constant or even decays with $n$. For example, in Lemma \ref{lemma:Delta_constant_kernel} we show that for the constant kernel $k(\bx,\bx')\equiv1$ one has $\Delta_n(r)=\mathcal O(1/\sqrt{n})$. 

More generally, if the kernel $k$ is strictly positive-valued and has constant diagonal, then the existence of two distinct admissible datapoints with responses of opposite signs rules out any decay to zero of $\Delta_n(r)$ with $n$. Indeed, Lemma~\ref{lemma:delta_lower_constant} shows that under these broad conditions $\Delta_n(r)$ is bounded away from zero. Thus, the best one can hope for in this setting is $\Delta_n(r)=\mathcal O(1)$ -- Appendix \ref{app:Vn_Deltan_bounds} shows this is available under concrete and practically interpretable regularity conditions. A canonical response-side condition  is when the responses are given by a fixed RKHS function, namely $y_i = f_*(\bx_i)$ with $f_*\in\mathcal H_k$, see Lemma~\ref{lemma:Delta_rkhs_response_bound}. Then, $\Delta_n(r)\leq \|f_*\|_{\mathcal H_k}\, V_n(r)/(r^2+V_n(r))$. Complementarily, Appendix~\ref{app:Delta_kernel_sum_bounds} shows that imposing a pairwise separation condition $\|\bx_i-\bx_j\|\ge q$ on the training data also yields $\mathcal O(1)$-sensitivity for standard decaying kernels. Further examples are provided by the 1D exponential and purely diagonal kernels where Lemmas~\ref{lemma:Delta_exp_1D} and~\ref{lemma:Delta_purely_diagonal_kernel} give $\Delta_n(r)\le 2M_Y/r$ and $\Delta_n(r)\le 2\sqrt{2}M_Y/r$ respectively.

The above examples illustrate how the abstract sensitivity quantities can be controlled in concrete regimes. They are not additional assumptions of Theorem~\ref{thm:main_rdp} which applies completely generally.
\paragraph{Main RDP Bound.} We are ready to state the RDP bound which is uniform over all finite test sets $X_T$.
\begin{theorem}
\label{thm:main_rdp}
Let $D\sim D'$ be neighbouring datasets of size $n$ and let $\Pi_D(X_T)$ and $\Pi_{D'}(X_T)$ denote the posterior distributions of the random field on a test set $X_T$. Then for every finite $X_T$ and every $1<\alpha<1+r^2/V_n(r)$ one has
\begin{align}
\begin{split}\label{eq:rdp_bound}
& D_\alpha\left(\Pi_D(X_T)\middle\|\Pi_{D'}(X_T)\right)
\le
-\frac{\log\left(
1-\frac{\alpha(\alpha-1)\tau_n(r)^2}{1+\tau_n(r)}
\right)}{2(\alpha-1)}
\\
& + \frac{\alpha}{2}
\frac{\tau_n(r)+1}
{1-(\alpha-1)\tau_n(r)}\left(\frac{\Delta_n(r)}{\sigma}\right)^2,\ \tau_n(r):=\frac{V_n(r)}{r^2}.
\end{split}
\end{align}
\end{theorem}
Substituting upper bounds $V_n(r)\leq \widehat{V_n}(r)$ and $\Delta_n(r)\leq \widehat{\Delta_n}(r)$ in \eqref{eq:rdp_bound} yields an RDP bound for every $1<\alpha<1+r^2/\widehat{V_n}(r)$.
\begin{figure*}
  \centering
 \includegraphics[width=0.85\textwidth]{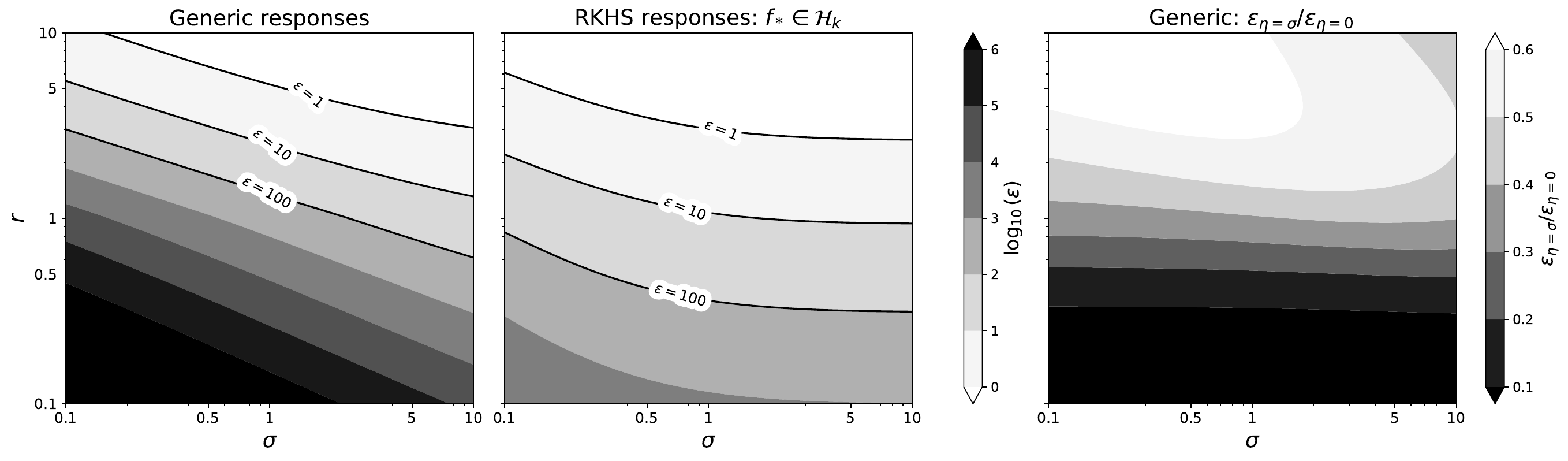}
  \caption{Examples of $(\epsilon,\delta)$-DP bounds constructed from Theorem \ref{thm:main_rdp} (log-log plots). We use $V_n(r)=\bar V_n(r)$, $M_Y=1$, $n=10^3$, $\delta=10^{-4}$, $\kappa^2=0.1$, $L=1$. {\bf Left:} the generic DP bound for $\Delta_n(r)$ from Lemma \ref{lemma:boundedResp_Delta_bound}. {\bf Middle:} more specialised setups lead to tighter DP bounds. It uses the tighter bounds in the RKHS response model from Lemma~\ref{lemma:Delta_rkhs_response_bound} with $\|f_*\|_{\mathcal H_k}=1$. {\bf{Right:}} the effect of using additional GP noise according to \eqref{def:mechanism_eta} with $\eta=\sigma$. This reduces $\varepsilon$ by a factor $1-10$ and the effect is strongest in less regularised regimes. }
  \label{fig:dp_plots}
\end{figure*}
 We prove Theorem~\ref{thm:main_rdp} in Appendix~\ref{app:rdp_proofs} using R\'{e}nyi divergence of multivariate Gaussians by \cite{GIL}. The key difficulty is bounding $D_\alpha$ uniformly when both the mean and covariance of the Gaussians change with the data in a complicated way. To this end, we derive specialised covariance-sandwich relations and use the rank-two structure of the replace-one dataset perturbation.
 
 The bound~\eqref{eq:rdp_bound} separates the roles of the posterior scale 
$\sigma$ and the effective ridge $r$. Increasing $\sigma$ only attenuates the mean-sensitivity term through $\Delta_n(r)/\sigma$. It does not affect 
the remaining term which captures leakage from the data-dependent posterior covariance. This distinguishes posterior sampling from standard DP mechanisms based on data-independent noise -- even in the limit $\sigma\to\infty$, the centred posterior draw can leak through $k_D$. Thus, meaningful intrinsic privacy also requires sufficient regularisation: Appendix~\ref{app:covariance_lower_bound} gives a strictly positive lower bound on covariance-induced RDP that is independent of $n$ and in the strongly regularised regime this lower bound has the same $r^{-4}$ scaling as our upper bound. Importantly, increasing $r$ reduces $V_n(r)/r^2$, enlarges the admissible range of $\alpha$, and decreases $\Delta_n(r)$. This behaviour is visible in Figure~\ref{fig:dp_plots} and is confirmed empirically in Section~\ref{sec:mia} where the membership-inference attack remains nontrivial at large $\sigma$ but collapses toward random guessing as $r$ increases.

In Appendix \ref{app:extra_gp_noise} we also show that one can enhance the intrinsic DP guarantees by adding noise to the released sample path via the mechanism
\begin{equation}\label{def:mechanism_eta}
D\mapsto f_D+g,\quad g\sim \mathcal{GP}(0,\eta^2k).
\end{equation}
By increasing $\eta$ one can obtain stronger privacy guarantees at fixed $\sigma, r$ (Figure~\ref{fig:dp_plots}), albeit typically at the cost of reduced sample-path utility.

The implementation of experiments from subsequent sections can be found on \url{https://github.com/tmaciazek/gaussian_process_dp}.

\section{PRIVACY ACCOUNTING AND MEMBERSHIP INFERENCE}
\label{sec:mia}
We use a LiRA-style MIA \citep{mia_first_principles} to show that the privacy effects predicted by the theory are observable under a concrete practical attack. We show that for a realistic adversary who can sample data from its prior distribution the attack success tracks the predicted dependence on $\sigma$ and $r$ where increasing $\sigma$ alone leaves residual membership signal and increasing $r$ drives the attack performance to random guessing. 

We consider the one-dimensional setup where the responses are generated as $y_i=f_{step}(x_i)$, $\Omega_X=[0,1]$ and we sample the covariates from the uniform distribution. The step-function is $f_{step}(x)=-1$ if $x<1/2$ and $f_{step}(x)=1$ otherwise. We use the exponential kernel $k(x,x')=\exp(-|x-x'|/\ell)$ for GP regression. The task is to decide whether the training set $D$ (with $|D|=10$) contained the point $z_0:=(x_0,y_0)=(1/2,1)$ given $L$ released posterior sample paths $f_D^{(1)},\dots,f_D^{(L)}$. To this end, we apply a version of LiRA which uses the observables 
\begin{gather*}
\hat f_D(x_0):=\frac{1}{L}\sum_{l=1}^L f_D^{(l)}(x_0),
\\
\hat v_D(x_0):=\frac{1}{L\sigma^2}\sum_{l=1}^L \left( f_D^{(l)}(x_0)-\hat f_D(x_0)\right)^2.
\end{gather*} 
Namely, we construct in- and out-distributions of $\left(\hat f_D(x_0),\hat v_D(x_0)\right)$ by sampling training sets $D$ that contain $z_0$ and $D$ that do not contain $z_0$. We then estimate the density of these in- and out-distributions (denoted by $\rho_{in}$ and $\rho_{out}$) and then apply the Neyman-Pearson binary hypothesis testing criterion 
\[
\log\frac{\rho_{in}\left(\hat f_D(x_0),\hat v_D(x_0)\right)}{\rho_{out}\left(\hat f_D(x_0),\hat v_D(x_0)\right)} > C, \quad C\in\mathbb R
\]
to decide whether an observed $\hat f_D, \hat v_D$ came from $D$ containing $z_0$. See \cite{mia_first_principles} for more details of the LiRA attack. We use a latent Gaussian mixture model to estimate $\rho_{in/out}$, see Appendix \ref{app:mia}.

Our choice of the noiseless regression function and the point $z_0$ is favourable to the adversary as probing the data around the jump location strongly affects the GP posterior sample paths. Thus, including $z_0$ creates a strong local membership signal. This is an appropriate stress-test for our DP claim which must control worst-case privacy leakage. We work with the 1D exponential kernel where by Proposition~\ref{prop:general_positive_kernel_exact_Vn} we have $V_n(r)=\bar V_n(r)$ with $\kappa=\exp(-1/\ell)$ and Corollary~\ref{coro:exp_1d} gives us a specialised tighter RDP bound.

For a fixed $C$ in the Neyman-Pearson criterion we evaluate the true-positive rate ($\mathrm{TPR}$) which is the fraction of members correctly classified as members and the false-positive rate ($\mathrm{FPR}$)  which is the fraction of non-members incorrectly classified as members. The ROC curves are obtained by varying $C$. We report \emph{excess} $\mathrm{TPR}@\mathrm{FPR}=10\%$ i.e., $TPR-10\%$ at $FPR=10\%$. Results at $\mathrm{FPR}=1\%$ are qualitatively similar and are deferred to Appendix~\ref{app:mia}.

\paragraph{Attack success vs. $\varepsilon$.} Figure~\ref{fig:mia} assumes $L=1$ and shows that the attack curves mirror our privacy bound. In the regime where the bound reaches single-digit $\varepsilon$, the LiRA-style attack is consistently driven close to random-guess performance, suggesting practical relevance of the bounds. This agrees with the recent observations \citep{LowyLiLiuEtAl2024LargeEpsilonMIA} that single-digit $\varepsilon$ can suffice to defend against MIA under a realistic attacker model (such as ours) in which the adversary does not have knowledge of $D_-$. When our bound remains large, large $\sigma$ alone leaves residual membership signal confirming that regularisation, not posterior scale alone, is crucial for good privacy. This is especially evident when considering the formal limit $\sigma\to\infty$ (with $r$-fixed) corresponding to releasing covariance-only information via the centred posterior draw $f_D\sim \mathcal{GP}(0,k_D)$ described by RDP bound \eqref{eq:rdp_bound} with $\Delta_n\to 0$.
\begin{figure}[t]
  \centering
 \includegraphics[width=.98\columnwidth]{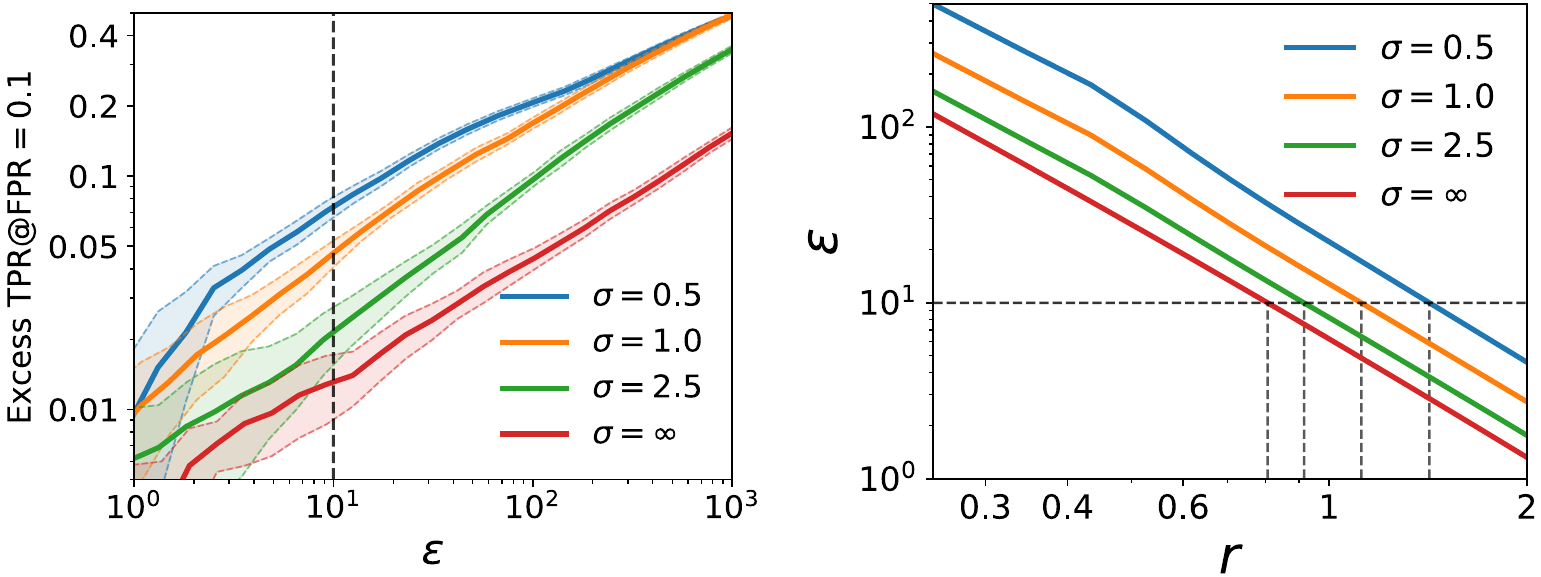}
 \caption{
MIA success and $(\varepsilon,\delta)$-DP bounds for one released posterior sample path ($\ell=1$, $L=1$, $n=10$, $\delta=10^{-3}$). 
{\bf Left:} Excess $\mathrm{TPR}@\mathrm{FPR}=0.1$ vs. $\varepsilon$ for different $\sigma$. Error bands show standard deviation over $10$ random seeds. {\bf Right:} Privacy bound versus regularisation $r$. Dashed guides mark the $\varepsilon=10$ threshold and the corresponding $r$ values.  Small $\varepsilon$ is reached only with sufficient regularisation, even in the covariance-only limit $\sigma=\infty$.
}
  \label{fig:mia}
\end{figure}
\paragraph{Privacy composition with $L$.} Figure~\ref{fig:mia_L} shows how MIA success scales with the number of released posterior sample paths $L$. As $L$ becomes large, the attack approaches the non-private benchmark corresponding to releasing the exact posterior functions $\mu_D$ and $k_D$. At the same time,  optimising the RDP-to-DP conversion after composition over $\alpha$ in \eqref{eq:rdp_composition} yields an initially sub-linear scaling of the final $(\varepsilon,\delta)$-DP bound with $L$ despite the RDP terms composing additively. Appendix~\ref{app:composition} shows that under representative privacy and model parameters the admissible number of posterior-path releases can reach hundreds or more. Interestingly, in~GP-Thompson sampling (GP-TS) established regret-calibrated covariance-scale schedules yield strongly sublinear privacy-cost growth $\varepsilon=\mathcal{O}(T^{2\nu/(2\nu+d)})$ with the number of TS rounds $T$ for Matérn-$\nu$ kernels, see Appendix~\ref{app:public_covariates}. This directly extends the results of \cite{pmlr-v238-ou24a} to GP-TS showing it is intrinsically differentially private at no loss in performance due to privacy and at an improved privacy cost scaling with $T$.

\begin{figure}[t]
  \centering
 \includegraphics[width=.9\columnwidth]{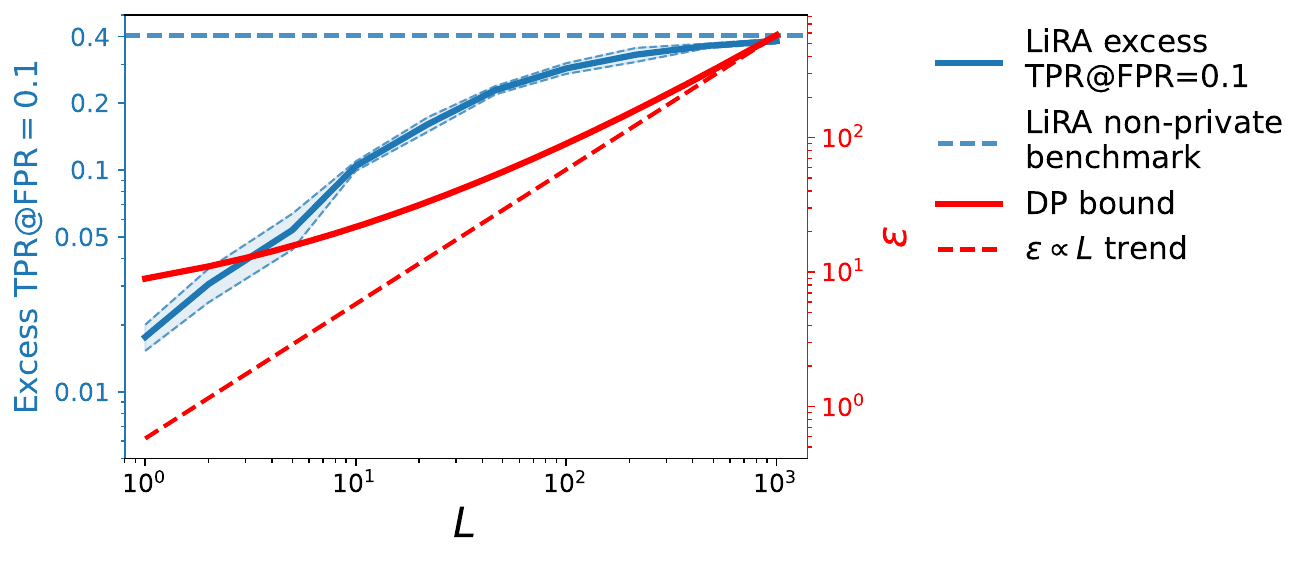}
\caption{
Effect of releasing $L$ posterior sample paths. The LiRA attack starts with random-guess effectiveness ($TPR\approx FPR$) at $L=1$ and strengthens with $L$ approaching the non-private benchmark while the $(\varepsilon,\delta)$-DP bound initially grows sub-linearly with $L$. Parameters: $\ell=1$, $r=1$, $\sigma=2$, $n=10$, $\delta=10^{-3}$. Error bands show SD over $10$ random seeds.
}
  \label{fig:mia_L}
\end{figure}

\section{PRIVACY VS. UTILITY}\label{sec:excursion}
In this section we test when GP posterior sampling can provide a useful privacy-preserving Bayesian output. The experiments isolate the privacy--utility mechanism suggested by Theorem~\ref{thm:main_rdp}: as observation noise makes stronger regularisation utility-optimal, the posterior-sampling release moves closer to the strong-privacy region while preserving downstream task utility. We begin with a simple one-dimensional setting where the privacy accounting is tighter and then consider more complex spatial setups. These experiments demonstrate the potential usefulness of the studied DP mechanism, but a broader assessment across tasks, modelling assumptions and more elaborate Bayesian workflows remains an important direction for future work.

We consider a GP-regression setting where an unknown latent random field $f_*$ is observed only through noisy data
$D=(X,\by)$, and inference is performed through a posterior GP,
$\Pi_D=\mathcal{GP}(\mu_D,\sigma^2 k_D)$ with hyperparameters
$\theta=(\sigma,r,\ell)$ corresponding to kernel scale, effective ridge and lengthscale. Abstractly, we aim to choose hyperparameters that minimise the average task-specific loss,
\begin{equation}\label{eq:unconstrained_theta}
\theta_* \in
\arg\min_{\theta}
\mathbb E_{f_*}E_{D\mid f_*}\!\left[
\mathcal L(\Pi_D,f_*)
\right].
\end{equation}
Given $\theta$ one evaluates the utility of the fitted posterior distribution $\mathcal U(\Pi_D,f_*)$. For ordinary regression the utility might be (negative) mean-squared error and the loss might be the predictive negative log-likelihood. Here we  turn to excursion-set estimation, a task common in geospatial and environmental modelling. This is a typical task where posterior draws naturally induce random excursion-set realizations by thresholding the posterior draws \citep{excursion_uncertainity}. For example, one may wish to publish a pollution map from measurements collected at private sensor locations $X$.

Given a threshold $t\in\mathbb R$, the true excursion set is $\Omega_t(f_*):=\{x\in\Omega_X:f_*(x)\ge t\}$. In our experiments we take $t=0$. The posterior assigns to each location $x$ the
excursion probability
\[
p_D(x)
:=
\mathbb P_{f\sim\Pi_D}\!\left[f(x)\ge t\right]
=
\Phi\!\left(
\frac{\mu_D(x)-t}
{\sigma\sqrt{k_D(x,x)}}
\right),
\]
where $\Phi$ is the CDF of normal distribution. In the synthetic experiments the true excursion set is known and we use BCE as the loss function:
\[
\mathcal L_{BCE}(\Pi_D,f_*)
=
\int_{\Omega_X}
\ell_{\rm BCE}\!\left(s_*(x),p_D(x)\right)\,dx
\]
where $\ell_{\rm BCE}(s,p):=-s\log p-(1-s)\log(1-p)$ and $s_*(x)=1$ when $f_*(x)\ge t$, and $s_*(x)=0$ otherwise. We approximate this integral on a fixed mesh. This oracle task-specific loss targets posterior excursion probabilities directly rather than general-purpose GP fit. The excursion set estimate is $\widehat\Omega$. The utility is the set accuracy measured by
intersection-over-union (IoU)
\[
\mathcal U_\theta(\Pi_D,f_*)=\mathrm{IoU}\left(\widehat\Omega,\Omega_t(f_*)\right)
:=
\frac{\left|\widehat\Omega\cap\Omega_t(f_*)\right|}
     {\left|\widehat\Omega\cup\Omega_t(f_*)\right|},
\]
where $|\Omega|$ denotes the volume of $\Omega$. We select the non-private excursion set estimate $\widehat\Omega_D(\theta_*)$ via the simple thresholded-probability criterion
\[
\widehat\Omega_D(\theta_*)
=
\{x\in\Omega_X:p_D(x)\ge C\},
\]
where the cutoff $C$ is chosen by validation. The private release uses hyperparameters selected by the constrained search
\begin{equation}\label{eq:constrained_theta}
\theta^{\rm priv}_L
\in
\arg\min_{\theta:\,\varepsilon_L(\theta)\le \varepsilon_0}
\mathbb E_{f_*,D}\!\left[
\mathcal L_{BCE}(\Pi_D,f_*)
\right],
\end{equation}
where $\varepsilon_L(\theta)$ is the $(\varepsilon,\delta)$-DP privacy bound for releasing $L$ independent
posterior sample paths and $\varepsilon_0$ is the target privacy budget. Given
these hyperparameters, the released randomized estimate is
\[
\widehat\Omega^{(L)}_D
=
\left\{
x\in\Omega_X:
\frac1L\sum_{l=1}^L
\mathbf 1\left\{f_D^{(l)}(x)\ge t\right\}
\ge c
\right\},
\]
where $c$ is fixed by validation and $f_D^{(\ell)}$ are i.i.d. from the GP posterior using $\theta=\theta^{\rm priv}_L$. 

This operational approach is in the spirit of the randomized decision-theoretic view of posterior sampling developed by
\citet{JMLR:v18:15-257}. We use posterior draws as a randomized answer to a downstream excursion-set
query: draw $f_D^{(1)},\dots,f_D^{(L)}$ from the GP posterior and release
$\widehat\Omega^{(L)}_D$. Since this is a post-processing of the sampled
paths, this set-valued release $\widehat\Omega^{(L)}_D$ inherits the same $(\varepsilon_L,\delta)$ privacy guarantee as the $L$ released posterior paths. The connection to randomised decision-making is conceptual rather than a direct application of the results of \citet{JMLR:v18:15-257} --
their assumptions do not cover the usual GP regression posterior which is why one needs to derive GP-specific bounds.

\paragraph{Experimental results.} 
We generate noisy observations by first sampling $f_*\sim\mathcal{GP}(0,k_{gen})$ with $k_{gen}(x,x')=\exp(-|x-x'|)$ and then feed $f_*$ into the bounded-response model
\begin{equation}\label{eq:response_model}
y_i=(1-M_\xi)\frac{f_*(x_i)}{\|f_*\|_\infty}+\xi_i, \quad \xi_i\overset{\mathrm{i.i.d.}}{\sim} P_\xi
\end{equation}
with $x_i$ uniformly distributed on $\Omega_X=[0,1]$, $P_\xi=\mathrm{Unif}([-M_\xi, M_\xi])$, $0<M_\xi<1$ and $\|f_*\|_\infty:=\sup_{\bx\in\Omega_X}|f_*(\bx)|$. Such a response model has amplitude noise-to-signal ratio $NSR:=M_\xi/(1-M_\xi)$. We use the exponential kernel for GP regression where Corollary~\ref{coro:exp_1d} provides tighter DP-bounds. We work with $L=1$, $|D|=100$ and the privacy budget $\varepsilon_0=10$ with $\delta=0.005$. The parameters $\theta_*$ and $\theta^{\rm priv}_L$ are selected via several rounds of grid-search minimisation according to Eqs. \eqref{eq:unconstrained_theta} and \eqref{eq:constrained_theta} on a fixed sample of $10^3$ pairs $(D,f_*)$. The thresholds $c,C$ are chosen via another grid-search on a separate validation set by optimising the expected IoU.

\begin{table}[t]
\centering
\setlength{\tabcolsep}{4pt}
\small
\begin{tabular}{lcccc}
\toprule
$M_\xi$ & $0.1$ & $0.3$ & $0.5$ & $0.6$ \\
\midrule
NSR & $0.11$ & $0.43$ & $1.00$ & $1.50$ \\
$\varepsilon$, unconstrained & $2460$ & $270$ & $36$ & $22$ \\
$d_{\rm eff}$, unconstrained & $51.2$ & $19.4$ & $8.39$ & $5.75$ \\
$d_{\rm eff}$, $\varepsilon<9$ & $3.09$ & $2.26$ & $2.18$ & $2.16$ \\
relative BCE increase & $116\%$ & $37.7\%$ & $12.0\%$ & $5.9\%$ \\
IoU, unconstrained & $0.949$ & $0.916$ & $0.871$ & $0.844$ \\
$1$-path IoU mean, $\varepsilon<9$ & $0.841$ & $0.828$ & $0.797$ & $0.765$ \\
relative IoU gap & $10.6\%$ & $8.8\%$ & $8.2\%$ & $8.6\%$ \\
\bottomrule
\end{tabular}
\caption{Median over $10^3$ draws of $(D,f_*)$-pairs across a range of
$\mathrm{NSR}$ values. The private candidate is the best
average-BCE-selected setting satisfying $\varepsilon<9$ and
$\delta=10^{-3}$. The $1$-path IoU is first averaged over $50$
posterior draws for each $(D,f_*)$-pair.}
\label{tab:noise_transition}
\end{table}

Table~\ref{tab:noise_transition} shows the transition across noise levels. As $M_\xi$ increases the unconstrained loss-optimum becomes naturally more regularised: both $\varepsilon$ and effective dimension $d_{\rm eff}:=\mathrm{tr}\left[K_{XX}\left(K_{XX}+r^2I\right)^{-1}\right]$ decrease. Thus, the cost of imposing $\varepsilon<9$ also decreases with the relative BCE cost falling from $116\%$ at $\mathrm{NSR}=0.11$ to $5.9\%$ at $\mathrm{NSR}=1.50$. The intermediate case $\mathrm{NSR}=1$ illustrates the typical tradeoff where the privacy constraint reduces $d_{\rm eff}$ from $8.39$ to $2.18$ while increasing median BCE by only $12\%$. The released object is the thresholded posterior sample path. Its median IoU remains high at $0.797$, only $8.2\%$ below the non-private posterior-probability benchmark. For more details see Appendix \ref{app:utility}. Appendix \ref{app:utility_2D} presents a 2D excursion set study with excursion set visualisations. As a real-data study, we apply analogous procedure to  \cite{HMLRPricePaidData} for 2018 Greater London leasehold-flat transactions. Transactions are aggregated onto a hexagonal lattice and GP regression is applied to the clipped and
rescaled centred responses, see Appendix~\ref{app:london} for details. Property prices and hexagon locations are private. Figure~\ref{fig:london} shows the  non-private posterior excursion boundary $\{x:p_D(x)\ge C\}$ and private one-path excursion boundaries. 

\paragraph{Privacy for free as $n\to\infty$.} The utility trade-off diminishes as $n$ grows. Table~\ref{tab:n_transition} in Appendix~\ref{app:privacyForFree} shows that in our synthetic-data experiment at $n=400$ and $\mathrm{NSR}=1$, the IoU gap drops to $4.2\%$. In Appendix~\ref{app:privacyForFree} we show that this effect is very general. Namely, by a direct combination of \cite{10.3150/07-BEJ5102} and \cite{LeGratiet2015} we obtain that, under mild regularity conditions, GP posterior risk $\EE_{f_D}[\|f_D-f_*\|_{L^2(P_X)}^2]\to 0$ in probability (over training data distribution) as $n\to\infty$ in  bounded-response model under growing regularisation $r_n \propto n^\gamma$ with $3/8<\gamma<1/2$. Consequently, one-path $\mathrm{IoU}\to 1$ in probability and our generic bounded-response DP-bound guarantees $\varepsilon \to 0$ as $n\to\infty$.
\begin{figure}[t]
  \centering
  \includegraphics[width=\columnwidth]{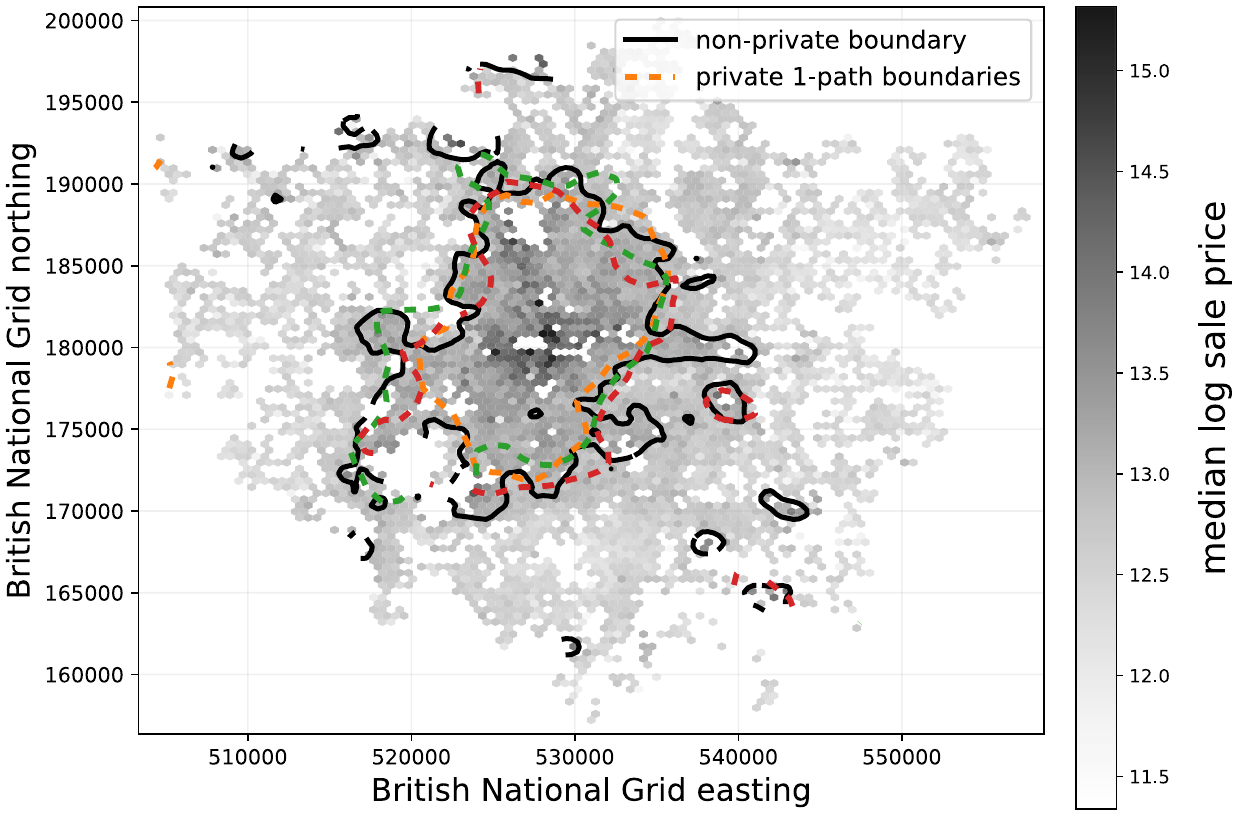}
  \caption{Greater London 2018 house-price excursion experiment. The task is to estimate the region where the median log sale
  price exceeds $13$ (about $\pounds 442\mathrm{k}$). The private posterior 
  draws (dashed lines) retain the dominant high-price region in central London, but can
  miss or smooth out smaller high-price islands outside central London
  illustrating the randomisation and regularisation induced by the
  privacy-constrained posterior.  $\varepsilon=9$, 
  $\delta=n^{-1.1}$ with $n=4390$ hexagons.}
  \label{fig:london}
\end{figure}

\section{DISCUSSION AND LIMITATIONS}\label{sec:discussion}
Our DP bounds assume a fixed kernel and fixed hyperparameters, including $r$, $\sigma$ and $\ell$. They apply directly when these choices are independent of
the private dataset, e.g. selected using public or non-private set-aside data. If hyperparameters are selected using the private training set, this must be privacy-accounted. Any RDP guarantee for the hyperparameter selection step can be combined with ours by composition.

Our DP bounds show that small-$\varepsilon$ guarantees require moderate to strong 
regularisation which may be restrictive for membership privacy. Importantly, larger privacy budgets can provide protection against more information-intensive threats such as full-example reconstruction. Indeed, reconstruction robustness (ReRo) bounds show that for diffuse priors on $d$-dimensional data reconstruction be mitigated even for $\varepsilon\propto d$ \citep[Propositions~6--7]{BalleCherubinHayes2022}. Consequently, for high-dimensional data such as images privacy budgets of order $10^3-10^4$ can still certify strong protection against data reconstruction. Thus, when the private data is high-dimensional, our DP bounds for GP posterior sampling can certify intrinsic protection against data reconstruction even under weaker regularisation, see also Appendix \ref{app:composition}.

Our bounds require $\Delta_n(r)<\infty$ and hence bounded responses. As is standard in DP, unbounded responses may yield unbounded sensitivity and response clipping provides a conventional mitigation. Theorem~\ref{thm:main_rdp} is completely general, but its practical strength depends on controlling $\Delta_n(r)$. While the generic bound from Lemma~\ref{lemma:boundedResp_Delta_bound} can be conservative, our RKHS-response and data-separation results in Appendix~\ref{app:Vn_Deltan_bounds} give much sharper $\mathcal O(1)$-bounds under natural structural assumptions that arise in a range of practical applications. Extending these to further kernels and response/data models remains future work.

\textbf{Conclusions\quad} We show that GP posterior sampling can provide intrinsic privacy
guarantees in regimes controlled by the posterior covariance and 
effective regularisation.  Membership-inference attacks follow the predicted dependence on
$r$, $\sigma$, and the number of released paths $L$. The utility experiments
show that for randomized excursion-set estimation privacy-compatible
regularisation can preserve utility even when it degrades posterior
calibration. In particular, noisy regimes in which regularisation is 
natural yield meaningful single-path privacy guarantees with useful released excursion
sets. Thus posterior randomness can be a genuine privacy resource, but not an unconditional one -- useful intrinsic privacy requires enough regularisation to control both posterior-mean and posterior-covariance leakage.


\bibliography{gp_privacy_refs}
\bibliographystyle{plainnat}

\clearpage
\appendix
\thispagestyle{empty}

\onecolumn
\aistatstitle{Supplementary Materials}

\section{PROOF OF THE RDP BOUND}
\label{app:rdp_proofs}

Throughout this section, let $D = D_- \cup \{(\bx,y)\}$ and $D' = D_- \cup \{(\bx',y')\}$ be replace-one neighbouring datasets, where $D_-:=\left((\bx_1,y_1),\dots,(\bx_{n-1},y_{n-1})\right)$ is their common core. Write also $X_-:=(\bx_1,\dots,\bx_{n-1})$ and $y_-:=(y_1,\dots,y_{n-1})^T$. Define the common-core posterior mean and covariance as
\begin{gather*}
\mu_{D_-}(X) := K(X,X_-)\left(K(X_-,X_-)+r^2I\right)^{-1}y_-,
\\
K_{D_-}(X,X') := K(X,X')-K(X,X_-)\left(K(X_-,X_-)+r^2I\right)^{-1}K(X_-,X').
\end{gather*}
For a finite test set \(X_T\), set
\begin{gather*}
C:=K_{D_-}(X_T,X_T),
\qquad
\bc(\bx):=k_{D_-}(X_T,\bx),
\qquad
v(\bx):=k_{D_-}(\bx,\bx).
\end{gather*}
The identities for the posterior mean and covariance 
\begin{align}
\begin{split}\label{eq:update_formula}
\mu_D(X_T)=\mu_{D_-}(X_T)+\frac{y-\mu_{D_-}(\bx)}{v(\bx)+r^2}\,\bc(\bx),
\qquad
\mu_{D'}(X_T)=\mu_{D_-}(X_T)+\frac{y'-\mu_{D_-}(\bx')}{v(\bx')+r^2}\,\bc(\bx'),
\\
K_D:=k_D(X_T,X_T)=C-\frac{\bc(\bx). \bc(\bx)^T}{v(\bx)+r^2},
\qquad
K_{D'}
=
C-\frac{\bc(\bx').\bc(\bx')^T}{v(\bx')+r^2}
\end{split}
\end{align}
are the standard one-step update formulas obtained by writing the full posterior expressions for $D=D_-\cup\{(\bx,y)\}$ and $D'=D_-\cup\{(\bx',y')\}$ from Section \ref{sec:setup} in block form relative to the common core $D_-$ and taking the corresponding Schur complement and using the block-matrix inverse identity. We also define the common-core posterior variance scale $V_{D_-}(r)$ and the uniform posterior variance scale $V_n(r)$ as 
\[
V_{D_-}(r):=\sup_{\bx\in\Omega_X} k_{D_-}(\bx,\bx), \qquad V_n(r):=\sup_{|D_-|=n-1} V_{D_-}(r).
\]
Since \(0\leq k_{D_-}(\bx,\bx)\leq k(\bx,\bx)\) and $k(\bx,\bx)\leq 1$, we have \(0\leq V_{D_-}(r),V_n(r)\leq 1\).
Finally, for the posterior mean we define the uniform common-core sensitivity
\[
\Delta_n(r)
:=
\sup_{|X_T|<\infty}\,\sup_{D\sim D'}
\left(\delta_T^T\,k_{D_-}(X_T,X_T)^{\dagger}\,\delta_T\right)^{1/2},
\qquad
\delta_T:=\mu_D(X_T)-\mu_{D'}(X_T),
\]
where $A^\dagger$ denotes the Moore--Penrose pseudoinverse. If $k_{D_-}(X_T,X_T)$ is invertible this reduces to the usual inverse.
\begin{lemma}[Covariance sandwich]
\label{lem:local_common_core_sandwich}
For every replace-one neighbouring pair $D,D'$ with common core $D_-$ and every finite test set $X_T$
\[
\frac{r^2}{V_{D_-}(r)+r^2}\,C
\;\preceq\;
K_D
\;\preceq\;
C,
\qquad
\frac{r^2}{V_{D_-}(r)+r^2}\,C
\;\preceq\;
K_{D'}
\;\preceq\;
 C.
\]
Consequently, one has
\[
\frac{r^2}{V_{D_-}(r)+r^2}\,K_{D'}
\;\preceq\;
K_D
\;\preceq\;
\left(1+\frac{V_{D_-}(r)}{r^2}\right)\,K_{D'}.
\]
Moreover, for $K_\alpha:=\alpha K_{D'}+(1-\alpha)K_D$ one has
\[
K_\alpha
\;\succeq\;
\frac{r^2-(\alpha-1)V_{D_-}(r)}{V_{D_-}(r)+r^2}\,C.
\]
Thus, $K_\alpha$ is positive semidefinite whenever $1<\alpha<1+\frac{r^2}{V_{D_-}(r)}$.
\end{lemma}

\begin{proof}
Consider first $K_D:=k_D(X_T,X_T)$. Since
\[
K_{D_-}(X_T\cup\{\bx\},X_T\cup\{\bx\})=
\begin{pmatrix}
C & \bc(\bx)\\
\bc(\bx)^T & v(\bx)
\end{pmatrix}
\]
is positive semidefinite, for every $\ba,\,t$ we have
\[
\begin{pmatrix}
\ba^T & t
\end{pmatrix}
\begin{pmatrix}
C & \bc(\bx)\\
\bc(\bx)^T & v(\bx)
\end{pmatrix}
\begin{pmatrix}
\ba \\
t
\end{pmatrix} 
 = \ba^T C \ba + 2t \ba^T\bc(\bx)+t^2v(\bx)
\geq 0.
\]
The discriminant of this polynomial in $t$ must be non-positive which is equivalent to 
\begin{equation}\label{ineq:C_ineq}
\ba^T.\left(\bc(\bx).\bc(\bx)^T\right)\ba=\left(\ba^T.\bc(\bx)\right)^2\leq v(\bx)\,\ba^TC\ba.
\end{equation}
Equivalently, $\bc(\bx).\bc(\bx)^T\preceq vC$. Hence by Equations \eqref{eq:update_formula}
\[
K_D
=
C-\frac{\bc(\bx). \bc(\bx)^T}{v(\bx)+r^2}
\succeq
C-\frac{v(\bx)}{v(\bx)+r^2}C
=
\frac{r^2}{v(\bx)+r^2}\,C, \quad K_{D'}\succeq\frac{r^2}{v(\bx')+r^2}\,C,
\]
where we applied the same argument to $\bx'$ to obtain the corresponding bound for $K_{D'}$. Since $v(\bx'),v(\bx)\le V_{D_-}(r)$, this yields
\begin{equation}\label{eq:KD_lower}
K_D
\succeq
\frac{r^2}{V_{D_-}(r)+r^2}\,C,\qquad K_{D'}
\succeq
\frac{r^2}{V_{D_-}(r)+r^2}\,C,
\end{equation}
The upper bound $K_D\preceq C$ is immediate from the identity for $K_D$. Indeed, by Equations \eqref{eq:update_formula}
\[
C - K_D=\frac{\bc(\bx). \bc(\bx)^T}{v(\bx)+r^2}\succeq 0.
\]
Combining this with \eqref{eq:KD_lower} yields
\[
K_D\preceq C
\preceq
\left(1+\tau_-(r)\right)K_{D'},\qquad \tau_-(r):=\frac{V_{D_-}(r)}{r^2}.
\]
This proves $K_D\preceq (1+\tau_-(r))K_{D'}$ and interchanging $D$ and $D'$ yields $K_{D'}\preceq (1+\tau_-(r))K_D$
 which is equivalent to
\[
\frac{1}{1+\tau_-(r)}\,K_{D'}
\preceq
K_D \preceq (1+\tau_-(r))K_{D'},\quad \frac{1}{1+\tau_-(r)}\,K_{D}
\preceq
K_{D'} \preceq (1+\tau_-(r))K_{D}.
\]
Finally, using the lower bound for $K_{D'}$ and the upper bound for $K_D$,
\begin{align*}
K_\alpha
&=
\alpha K_{D'}-(\alpha-1)K_D \succeq
\alpha \frac{r^2}{V_{D_-}(r)+r^2}\,C
-
(\alpha-1)C =
\left(
\frac{\alpha r^2}{V_{D_-}(r)+r^2}-(\alpha-1)
\right)C \\
&=
\frac{r^2-(\alpha-1)V_{D_-}(r)}{V_{D_-}(r)+r^2}\,C.
\end{align*}
This is positive semidefinite whenever $r^2-(\alpha-1)V_{D_-}(r)>0$.
\end{proof}

\begin{proposition}[RDP bound]
\label{prop:rdp_main}
Let $D,D'$ be replace-one neighbouring labelled datasets of size $n$. Let $\Pi_D(X_T)$ and $\Pi_{D'}(X_T)$ denote the (Gaussian) distributions of the released posterior sample vector on a finite test set $X_T$. Then for every
\[
1<\alpha<1+\frac{r^2}{V_n(r)},
\]
one has
\[
D_\alpha\left(\Pi_D(X_T)\|\Pi_{D'}(X_T)\right)
\le
\psi_\alpha\left(\frac{V_n(r)}{r^2}\right)
+
\frac{\alpha\left(V_n(r)+r^2\right)}
{2\sigma^2\left(r^2-(\alpha-1)V_n(r)\right)}
\,\Delta_n^2(r),
\]
where
\[
\psi_\alpha(\tau):=-\frac{\log\left(
1-\frac{\alpha(\alpha-1)\tau^2}{1+\tau}
\right)}{2(\alpha-1)}.
\]
\end{proposition}

\begin{proof}
Write
\[
\Pi_D(X_T)=\mathcal{N}(\mu_D,\sigma^2K_D),
\qquad
\Pi_{D'}(X_T)=\mathcal{N}(\mu_{D'},\sigma^2K_{D'}).
\]
For multivariate Gaussian measures the R\'{e}nyi divergence is given by the following formula \citep[see][Table 2]{GIL},
\begin{equation}\label{eq:gaussian_div}
D_\alpha(\Pi_D(X_T)\|\Pi_{D'}(X_T))
=
\frac{\alpha}{2\sigma^2}(\mu_D-\mu_{D'})^TK_\alpha^{-1}(\mu_D-\mu_{D'})
-
\frac{1}{2(\alpha-1)}
\log
\frac{\det(K_\alpha)}
{\det(K_D)^{1-\alpha}\det(K_{D'})^\alpha},
\end{equation}
where
\[
K_\alpha:=\alpha K_{D'}+(1-\alpha) K_D \succ 0.
\]

Let us first consider the finiteness condition in the Gaussian R\'enyi formula. The sandwich bounds from Lemma~\ref{lem:local_common_core_sandwich} imply that $K_\alpha\succ0$ whenever $C\succ0$ and $1<\alpha<1+r^2/V_n(r)$. In general, $K_D$, $K_{D'}$, and $K_\alpha$ may be singular as matrices on $\mathbb R^{|X_T|}$. However, the same sandwich bounds imply that on this range of $\alpha$ they have the same nullspace as $C$. From \eqref{eq:update_formula}
\[
\mu_D(X_T)-\mu_{D'}(X_T)=\frac{y-\mu_{D_-}(\bx)}{v(\bx)+r^2}\,\bc(\bx)-\frac{y'-\mu_{D_-}(\bx')}{v(\bx')+r^2}\,\bc(\bx').
\]
We also have that $\bc(\bx),\bc(\bx')\in\operatorname{range}(C)$. To see this, consider inequality \eqref{ineq:C_ineq} from Lemma \ref{lem:local_common_core_sandwich} which says that for any vector $\ba$ we have $\left(\ba^T\bc(\bx)\right)^2\leq v(\bx)\,\ba^TC\ba$. Taking $\ba\in\operatorname{ker}(C)$ gives $\left(\ba^T\bc(\bx)\right)^2\leq 0$, so we neccessarily have $\ba^T\bc(\bx)=0$ for any $\ba\in\operatorname{ker}(C)$. In other words, $\bc(\bx)\in\operatorname{ker}(C)^\perp$ for any $\bx$. Since $C$ is symmetric we have $\operatorname{ker}(C)^\perp=\operatorname{range}(C)$, thus $\bc(\bx)\in\operatorname{range}(C)$ for any $\bx$. Hence, $\mu_D(X_T)-\mu_{D'}(X_T)\in\operatorname{range}(C)$. Thus, the Gaussian densities $\Pi_D(X_T)$ and $\Pi_{D'}(X_T)$ are supported on the same affine subspace determined by $C$. So, the Gaussian R\'enyi formula applies and remains finite on this common support even when $K_\alpha$ is singular with ordinary inverses and determinants replaced by Moore-Penrose pseudoinverses and pseudodeterminants.

Next, we treat the two terms in \eqref{eq:gaussian_div} separately. Let us start with the determinant term. By Lemma~\ref{lem:local_common_core_sandwich},
\[
\frac{1}{1+\tau_-(r)}\,K_{D'}
\preceq
K_D
\preceq
(1+\tau_-(r))\,K_{D'},
\qquad
\tau_-(r):=\frac{V_{D_-}(r)}{r^2}.
\]
Next, conjugating the above inequality by $K_{D'}^{-1/2}$ and noting that $K_{D'}^{-1/2}K_{D'}K_{D'}^{-1/2}=\bone$ yields
\begin{equation}\label{eq:R_sandwich}
\frac{1}{1+\tau_-(r)}\,\bone
\preceq
K_{D'}^{-1/2}K_DK_{D'}^{-1/2}
\preceq
(1+\tau_-(r))\,\bone.
\end{equation}
Define $R:=K_{D'}^{-1/2}K_DK_{D'}^{-1/2}$. By \eqref{eq:R_sandwich}, the eigenvalues of $R$ all lie in the interval
\[
\left[\frac{1}{1+\tau_-(r)},\,1+\tau_-(r)\right]\subset \left[\frac{1}{1+\tau_n(r)},\,1+\tau_n(r)\right], \quad \tau_n(r):=\frac{V_n(r)}{r^2}.
\]
We also have 
\begin{equation}\label{eq:Kalph_identity}
K_\alpha
=
K_{D'}^{1/2}\left(\alpha I-(\alpha-1)R\right)K_{D'}^{1/2}
\end{equation}
Moreover, by \eqref{eq:update_formula}
\[
K_D-K_{D'}
=
\frac{\bc(\bx') .\bc(\bx')^T}{v(\bx')+r^2}
-
\frac{\bc(\bx). \bc(\bx)^T}{v(\bx)+r^2},
\]
so $\rank(K_D-K_{D'})\le 2$ and $K_D-K_{D'}$ has at most one positive and at most one negative eigenvalue. Note that $R-I=K_{D'}^{-1/2}(K_D-K_{D'})K_{D'}^{-1/2}$, thus we also have that $\rank(R-I)\le 2$ and $R-I$ has at most one positive and at most one negative eigenvalue. Let $\lambda_1,\dots,\lambda_m$ denote the eigenvalues of $R$. Using the identity \eqref{eq:Kalph_identity}, the determinant term becomes
\begin{align*}
&-\frac{1}{2(\alpha-1)}
\log\frac{\det K_\alpha}{(\det K_D)^{1-\alpha}(\det K_{D'})^\alpha}
=
-\frac{1}{2(\alpha-1)}
\log\frac{\det(\alpha I-(\alpha-1)R)}{(\det R)^{1-\alpha}}\\
&=\sum_{j=1}^m \phi_\alpha(\lambda_j),\qquad 
\phi_\alpha(\lambda):=
-\frac12\log\lambda
-\frac{1}{2(\alpha-1)}\log\left(\alpha-(\alpha-1)\lambda\right).
\end{align*}
By calculating first and second derivatives of $\phi_\alpha(\lambda)$ it is straightforward to see that $\phi_\alpha(\lambda)$ is convex on the interval $[\,1/(1+\tau_n(r)),\,1+\tau_n(r)\,]$ ($\phi''_\alpha(\lambda)>0$) and that it attains its minimum at $\lambda=1$ where $\phi_\alpha(1)=0$. Because $R-I$ is rank-two and has at most one positive and at most one negative eigenvalue, there is at most one eigenvalue of $R$ below $1$ and at most one above $1$. All the remaining eigenvalues of $R$ are equal to $1$. Call the eigenvalues different than $1$ by $\lambda_-\geq \frac{1}{1+\tau_n(r)}$ and $\lambda_+\leq 1+\tau_n(r)$, $\lambda_-\leq \lambda_+$. We have
\[
\sum_{j=1}^m \phi_\alpha(\lambda_j) = \phi_\alpha(\lambda_-) +\phi_\alpha(\lambda_+)\leq \phi_\alpha\left(\frac{1}{1+\tau_n(r)}\right)+\phi_\alpha(1+\tau_n(r))=-\frac{\log\left(
1-\frac{\alpha(\alpha-1)\tau_n(r)^2}{1+\tau_n(r)}
\right)}{2(\alpha-1)},
\]
where in the last inequality we have used the fact that $\phi_\alpha(\lambda)$ is a decreasing function for $\lambda<1$ and an increasing function when $\lambda>1$.

It remains to bound the mean-term in \eqref{eq:gaussian_div}. By Lemma~\ref{lem:local_common_core_sandwich},
\[
K_\alpha
\succeq
\frac{r^2-(\alpha-1)V_{D_-}(r)}{V_{D_-}(r)+r^2}\,C.
\]
Since $K_\alpha$ and $C$ have the same nullspace on the admissible range of $\alpha$, we may invert this inequality on their common support by taking the Moore-Penrose pseudoinverse of both sides to obtain
\[
K_\alpha^{\dagger}
\preceq
\frac{V_{D_-}(r)+r^2}
{r^2-(\alpha-1)V_{D_-}(r)}\,C^{\dagger}.
\]
Therefore, with $\delta_T:=\mu_D(X_T)-\mu_{D'}(X_T)$,
\[
\frac{\alpha}{2\sigma^2}\delta_T^TK_\alpha^{\dagger}\delta_T
\le
\frac{\alpha\left(k_{D_-}(\bx,\cdot)(r)+r^2\right)}
{2\sigma^2\left(r^2-(\alpha-1)V_{D_-}(r)\right)}
\,\delta_T^T C^{\dagger} \delta_T.
\]
Taking the supremum over all neighbouring pairs and all finite test sets \(X_T\), and using \(V_{D_-}(r)\le V_n(r)\), yields
\[
\frac{\alpha}{2\sigma^2}\delta_T^TK_\alpha^{\dagger}\delta_T
\le
\frac{\alpha\left(V_n(r)+r^2\right)}
{2\sigma^2\left(r^2-(\alpha-1)V_n(r)\right)}
\,\Delta_n^2(r).
\]

Combining the determinant and mean bounds proves the claim.
\end{proof}
The following lemma allows us to explicitly take the supremum over $X_T$ in the definition of $\Delta_n(r)$ as originally stated in Definition \ref{def:Vn_Hn}. 
\begin{lemma}[$\Delta_n(r)$ as RKHS norm]
\label{lem:Hn_RKHS_reduction}
Let $D = D_- \cup \{(\bx,y)\}$ and $D' = D_- \cup \{(\bx',y')\}$ be replace-one neighbouring datasets. Then $\mu_D-\mu_{D'}\in \mathcal H_{-}$, where $\mathcal H_{-}$ denotes the reproducing kernel Hilbert space associated with the common-core posterior covariance kernel $k_{D_-}$. For any finite test set $X_T\subset \Omega_X$ define $\delta_T:=\mu_D(X_T)-\mu_{D'}(X_T)$. Then
\[
\delta_T^T\,k_{D_-}(X_T,X_T)^{\dagger}\,\delta_T
\le
\|\mu_D-\mu_{D'}\|_{\mathcal H_{-}}^2.
\]
Equality holds whenever $\{\bx,\bx'\}\subset X_T$. Consequently,
\[
\Delta_n(r)=\sup_{D\sim D'}\|\mu_D-\mu_{D'}\|_{\mathcal H_{-}}.
\]
\end{lemma}
\begin{proof}
Fix \(D,D'\) and write $g:=\mu_D-\mu_{D'}$. By the one-step update formula \eqref{eq:update_formula}, if $\bx$ and $\bx'$ are the differing inputs in $D$ and $D'$, then
\[
g(\cdot)
=
\eta\,k_{D_-}(\cdot,\bx)-\eta'\,k_{D_-}(\cdot,\bx')
\]
for suitable scalars $\eta,\eta'$. Hence $g\in \mathcal H_{-}$. Next, we use the general standard fact that for any $h$ from RKHS of some kernel $\tilde k$ and for any finite sequence of distinct points $X_T\subset \Omega_X$ we have  \citep[see][Proposition 8]{hall13}
\begin{equation}\label{eq:interpolation}
h(X_T)^T \tilde K(X_T,X_T)^{-1} h(X_T)\leq \|h\|_{\mathcal H_{\tilde k}}^2.
\end{equation}
In the possibly singular case the same statement holds with the Moore-Penrose pseudoinverse. The equality in \eqref{eq:interpolation} holds when $h(\cdot) = \sum_{i=1}^m a_i \tilde k(\cdot,\bx_i)$ with $(\bx_1,\dots,\bx_m)\subset X_T$. To see this, expand $\|h\|_{\mathcal H_{\tilde k}}^2 = \sum_{i,j}a_ia_j \tilde k(\bx_j,\bx_i)$. On the other hand, denoting $G:=\tilde K(X_T,X_T)$ we have
\begin{align*}
h(X_T)^T G^\dagger h(X_T) & =\sum_{i,j}a_i a_j \tilde k(X_T,\bx_j)^T G^\dagger\tilde k(X_T,\bx_i)=\sum_{i,j}a_i a_j\be_j^T G G^\dagger G \be_i 
\\
& = \sum_{i,j}a_i a_j\be_j^TG \be_i = \sum_{i,j}a_i a_j\tilde k(\bx_j,\bx_i),
\end{align*}
where in the second equality we have used the fact that $\tilde k(X_T,\bx_i)=G\be_i$ with $\be_i$ being the $i$-th unit vector in $\mathbb R^{|X_T|}$, in the third equality we have used the Moore-Penrose property $G G^\dagger G=G$ and in the last equality we have used the fact that $\be_j^TG \be_i=G_{ji}=\tilde k(\bx_j,\bx_i)$.

The final inequality is obtained by putting $h = g$ and $\tilde k = k_{D_-}$ in \eqref{eq:interpolation}. This means that 
\[
\sup_{|X_T|<\infty}\delta_T^T\,k_{D_-}(X_T,X_T)^{\dagger}\,\delta_T = \|\mu_D-\mu_{D'}\|_{\mathcal H_{-}}^2,
\]
proving that $\Delta_n(r)^2=\sup_{D\sim D'}\|\mu_D-\mu_{D'}\|_{\mathcal H_{-}}^2$.
\end{proof}

\subsection{Adding prior GP noise to the released sample path}
\label{app:extra_gp_noise}

We now consider a simple modification of the release mechanism in which one first draws
\[
f_D(X_T)\sim \mathcal{N}\left(\mu_D(X_T),\sigma^2K_D\right),
\qquad
K_D:=k_D(X_T,X_T),
\]
from the posterior and then adds an independent Gaussian perturbation drawn from the GP prior
\[
g_T\sim \mathcal{N}\left(0,\eta^2K(X_T,X_T)\right),
\qquad \eta\ge0.
\]
The released vector is therefore
\[
\widetilde f_D(X_T):=f_D(X_T)+g_T.
\]
If \(\widetilde \Pi_D\) denotes the resulting sample path distribution, then
\[
\widetilde \Pi_D(X_T)
=
\mathcal{N}\left(\mu_D(X_T),\widetilde K_D\right),
\qquad
\widetilde K_D
:=
\sigma^2K_D+\eta^2K(X_T,X_T).
\]
The next proposition shows that the same proof strategy as in
Proposition~\ref{prop:rdp_main} yields an RDP bound for this modified
mechanism. The additional GP noise improves both the covariance
 term and the posterior-mean term of the original RDP bound.

\begin{proposition}[RDP bound with additional GP noise]
\label{prop:rdp_with_extra_gp_noise}
Let $D,D'$ be replace-one neighbouring labelled datasets of size $n$.
On a finite test set $X_T$ consider the modified release
\[
\widetilde f_D(X_T)
=
f_D(X_T)+g_T,
\qquad
g_T\sim \mathcal{N}\left(0,\eta^2K(X_T,X_T)\right),
\]
where $g_T$ is independent of $f_D(X_T)$. Denote the distribution of
$\widetilde f_D(X_T)$ by
\[
\widetilde\Pi_D(X_T)
=
\mathcal{N}\left(\mu_D(X_T),\widetilde K_D\right),
\qquad
\widetilde K_D
:=
\sigma^2K_D+\eta^2K(X_T,X_T),
\]
and define $\widetilde\Pi_{D'}(X_T)$ analogously. Then, for every
$\alpha$ such that
\[
1<\alpha<
1+\frac{1}{\widetilde\tau_n(r,\eta)},\qquad 
\widetilde\tau_n(r,\eta)
:=
\frac{\sigma^2V_n(r)}
{\sigma^2r^2+\eta^2\left(V_n(r)+r^2\right)}
\]
one has
\[
D_\alpha\left(
\widetilde\Pi_D(X_T)\middle\|\widetilde\Pi_{D'}(X_T)
\right)
\le
\psi_\alpha\left(\widetilde\tau_n(r,\eta)\right)
+
\frac{\alpha}{2}\frac{V_n(r)+r^2}
{
\sigma^2\left(r^2-(\alpha-1)V_n(r)\right)
+
\eta^2\left(V_n(r)+r^2\right)
}
\,\Delta_n^2(r).
\]
\end{proposition}

\begin{proof}
 Let $K_T:=K(X_T,X_T)$. Since
\[
C
=
K_T
-
K(X_T,X_-)
\left(K(X_-,X_-)+r^2I\right)^{-1}
K(X_-,X_T),
\]
we can write
\[
K_T=C+H,
\qquad
H\succeq0.
\]
Using the one-step update covariance identities \eqref{eq:update_formula}, we obtain
\begin{equation}\label{eq:noise_update}
\widetilde K_D
=
\eta^2H+(\sigma^2+\eta^2)C
-
\sigma^2\frac{\bc(\bx).\bc(\bx)^T}{v(\bx)+r^2},
\qquad
\widetilde K_{D'}
=
\eta^2H+(\sigma^2+\eta^2)C
-
\sigma^2\frac{\bc(\bx').\bc(\bx')^T}{v(\bx')+r^2}.
\end{equation}

As in the proof of Lemma~\ref{lem:local_common_core_sandwich}, the positive-semidefiniteness of block matrices involving $C,\bc(\bx),v(\bx)$ and $C,\bc(\bx),v(\bx')$ implies that $\bc(\bx).\bc(\bx)^T\preceq v(\bx)\,C$ and  $\bc(\bx').\bc(\bx')^T\preceq v(\bx')\,C$. Since $v(\bx),v(\bx')\le V_{D_-}(r)$, both noisy covariance matrices satisfy
\[
\widetilde K_D,\widetilde K_{D'}
\succeq
\eta^2H+
\left(
\eta^2+\sigma^2\frac{r^2}{V_{D_-}(r)+r^2}
\right)C.
\]
Since $\bc(\bx).\bc(\bx)^T, \bc(\bx').\bc(\bx')^T\succeq 0$, we also have
\[
\widetilde K_D,\widetilde K_{D'}
\preceq
\eta^2H+(\sigma^2+\eta^2)C.
\]
Next, define $\widetilde\tau_-(r,\eta)$ through
\[
1+\widetilde\tau_-(r,\eta)
=
\frac{\sigma^2+\eta^2}
{\eta^2+\sigma^2r^2/(V_{D_-}(r)+r^2)}.
\]
The preceding bounds imply the sandwich relation
\[
\frac{1}{1+\widetilde\tau_-(r,\eta)}\,\widetilde K_{D'}
\preceq
\widetilde K_D
\preceq
\left(1+\widetilde\tau_-(r,\eta)\right)\widetilde K_{D'}.
\]

Moreover, from \eqref{eq:noise_update} we also have
\[
\widetilde K_D-\widetilde K_{D'}
=
\sigma^2(K_D-K_{D'})
=
\sigma^2\left(
\frac{\bc(\bx').\bc(\bx')^T}{v(\bx')+r^2}
-
\frac{\bc(\bx).\bc(\bx)^T}{v(\bx)+r^2}
\right),
\]
and hence $\rank(\widetilde K_D-\widetilde K_{D'})\le2$. Thus, exactly as in the determinant argument from the proof of
Proposition~\ref{prop:rdp_main}, we get
\[
-
\frac{1}{2(\alpha-1)}
\log
\frac{\det(\widetilde K_\alpha)}
{\det(\widetilde K_D)^{1-\alpha}\det(\widetilde K_{D'})^\alpha}
\le
\psi_\alpha\left(\widetilde\tau_-(r,\eta)\right),
\]
where $\widetilde K_\alpha
:=
\alpha\widetilde K_{D'}+(1-\alpha)\widetilde K_D$. As before, in the singular case, this determinant expression is understood on the common support, with determinants replaced by pseudodeterminants.

Since $V_{D_-}(r)\le V_n(r)$, and the map
\[
V\mapsto
\frac{\sigma^2V}{\sigma^2r^2+\eta^2(v(\bx)+r^2)}
\]
is non-decreasing on $V\in [0,\infty)$, we have $\widetilde\tau_-(r,\eta)\le \widetilde\tau_n(r,\eta)$.  Using that $\psi_\alpha$ is non-decreasing in its argument, the determinant term is bounded by $2\,\psi_\alpha\left(\widetilde\tau_n(r,\eta)\right)$.

It remains to bound the mean term. By Lemma~\ref{lem:local_common_core_sandwich},
\[
K_\alpha
=
\alpha K_{D'}+(1-\alpha)K_D
\succeq
\frac{r^2-(\alpha-1)V_{D_-}(r)}
{V_{D_-}(r)+r^2}\,C.
\]
Therefore $\widetilde K_\alpha
=
\sigma^2K_\alpha+\eta^2K_T
=
\sigma^2K_\alpha+\eta^2(C+H)
$
satisfies
\begin{equation}\label{eq:tildeKalpha_bound}
\widetilde K_\alpha\succeq
\gamma
\,C,\qquad \gamma:=\sigma^2
\frac{r^2-(\alpha-1)V_{D_-}(r)}
{V_{D_-}(r)+r^2}+\eta^2
\end{equation}
The assumed range of $1<\alpha<1+1/\widetilde\tau_n(r,\eta)$ ensures that the scalar prefactor $\gamma>0$. Since with $\eta\neq 0$ the matrix $\widetilde K_\alpha$ may no longer have the same nullspace as $C$, we cannot apply the Moore-Penrose pseudoinverse directly to \eqref{eq:tildeKalpha_bound} as we did in the proof of Proposition \ref{prop:rdp_main}. However, for any matrix $A$ and any $\bu \in \operatorname{range}(A)$ we have 
\[
\bu^T A^\dagger \bu=\sup_{\bv}\left\{2\bu^T\bv-\bv^TA\bv\right\}.
\]
Since $\delta_T\in\operatorname{range}(C)$ also belongs to the $\operatorname{range}(\widetilde K_\alpha)$ (by the bound \eqref{eq:tildeKalpha_bound}), we can take $A=\widetilde K_\alpha,C$ and $u=\delta_T$ to get 
\[
\delta_T^T\widetilde K_\alpha^\dagger\delta_T=\sup_{\bv}\left\{2\delta_T^T\bv-\bv^T\widetilde K_\alpha\bv\right\}\leq \sup_{\bv}\left\{2\delta_T^T\bv-\bv^T\left(\gamma C\right)\bv\right\}=\gamma^{-1}\,\delta_T^TC^\dagger\delta_T.
\]
Taking the supremum over neighbouring pairs and finite test sets, using
$V_{D_-}(r)\le V_n(r)$, and plugging the expression for $\gamma$ gives the following final bound for the Gaussian R\'enyi mean term
\[
\frac{\alpha}{2}\delta_T^T\widetilde K_\alpha^\dagger\delta_T
\le
\frac{\alpha\left(V_n(r)+r^2\right)}
{2\left[
\sigma^2\left(r^2-(\alpha-1)V_n(r)\right)
+
\eta^2\left(V_n(r)+r^2\right)
\right]}
\,\Delta_n^2(r).
\]

Combining the determinant and mean bounds for
$D_\alpha\left(\widetilde\Pi_D(X_T)\|\widetilde\Pi_{D'}(X_T)\right)$ proves the claim.
\end{proof}

\section{BOUNDS FOR $V_n$, $\Delta_n$ AND SPECIAL CASES}
\label{app:Vn_Deltan_bounds}
We first present the general bounds for $V_n$.
\begin{proposition}[Upper bound for $V_n(r)$.]
\label{prop:general_positive_kernel_exact_Vn}
Let $r>0$ and let $k:\Omega_X\times\Omega_X\to\mathbb R$ be a positive semidefinite
kernel satisfying $\sup_{\bx\in\Omega_X} k(\bx,\bx)=1$. Suppose that $k$ admits a nonnegative lower bound $\kappa:=\inf_{\bx,\bx'\in\Omega_X} k(\bx,\bx')\ge 0$.  Then, for every $n\ge 1$
\[
V_n(r)\le  \bar V_n(r):=1-\frac{n-1}{n-1+r^2}\kappa^2.
\]
If, additionally $k(x,x)=1$ for all $x\in\Omega_X$, then this upper bound is exact, namely $V_n(r)=  \bar V_n(r)$.
\end{proposition}

\begin{proof}
By the assumptions, we have $0\le k(\bx,\bx)\le 1$ for all $\bx\in\Omega_X$. Define 
\[
a_n:=\frac{n-1}{n-1+r^2}.
\]
The case $n=1$ is immediate. Since the common core $D_-$ is empty, $V_1(r)=\sup_{\bx\in\Omega_X}k(x,x)=1=1-a_1\kappa^2$, because $a_1=0$. Next, consider $n>1$ and denote $n_-:=n-1$. Then $D_-$ is of size $n_-$ and denote it's covariate set as $X_-=\{\bx_1,\dots,\bx_{n_-}\}$. For $\bx\in\Omega_X$, write
\[
K:=K(X_-,X_-),\qquad
\bk_x:=K(X_-,\bx),\qquad
A:=K+r^2I.
\]
Then, $k_{D_-}(\bx,\bx')
=
k(\bx,\bx')-\bk_x^TA^{-1}\bk_{x'}
$. We first prove the upper bound. Let $\mathbf 1\in\mathbb R^{n_-}$ denote
the all-ones vector. Since $A$ is positive definite, we can apply Cauchy-Schwarz inequality with respect to the $A$-inner product
$\langle \bu,\bv\rangle_A=\bu^TA\bv$ with $u=\mathbf 1$ and
$v=A^{-1}\bk_x$ which gives
\[
(\mathbf 1^T\bk_x)^2
\le
(\mathbf 1^TA\mathbf 1)(\bk_x^TA^{-1}\bk_x),
\]
and thus
\[
\bk_x^TA^{-1}\bk_x
\ge
\frac{(\mathbf 1^T\bk_x)^2}{\mathbf 1^TA\mathbf 1}
=
\frac{(\mathbf 1^T\bk_x)^2}{\mathbf 1^TK\mathbf 1+r^2 n_-}.
\]
By the kernel lower bound,
\[
\mathbf 1^Tk_x
=
\sum_{i=1}^{n_-} k(\bx_i,\bx)
\ge
n_-\kappa.
\]
Since $\kappa\ge0$, this implies $(\mathbf 1^Tk_x)^2\ge n_-^2\kappa^2$, thus
\[
\bk_x^TA^{-1}\bk_x \ge \kappa^2\frac{n_-^2}{\mathbf 1^TK\mathbf 1+r^2 n_-}.
\]
Next, since $k$ is positive semidefinite and normalised, we have $|k(\bx_i,\bx_j)|
\le
\sqrt{k(\bx_i,\bx_i)k(\bx_j,\bx_j)}
\le
1$, 
and hence
\[
\mathbf 1^TK\mathbf 1
=
\sum_{i,j=1}^{n_-} k(\bx_i,\bx_j)
\le
n_-^2.
\]
It follows that
\[
\bk_x^TA^{-1}\bk_x
\ge
\frac{n_-^2\kappa^2}{n_-^2+r^2 n_-}
=
\frac{n_-\kappa^2}{n_-+r^2}
=
a_n\kappa^2.
\]
Substituting into the expression for $k_{D_-}(x,x)$ gives
\[
k_{D_-}(x,x)
\le
k(x,x)-a_n\kappa^2
\le
1-a_n\kappa^2.
\]
Taking the supremum over $x$ and over all common cores gives $V_n(r)\le 1-a_n\kappa^2$.

Now assume that $k(x,x)=1$ for all $x\in\Omega_X$. We prove the reverse inequality. By definition of $\kappa$, there exists a sequence of pairs $\left(\bx_1^{(j)},\bx_2^{(j)}\right)$ such that
\[
k\left(\bx_1^{(j)},\bx_2^{(j)}\right)\xrightarrow{j\to\infty} \kappa.
\]
For each $j$, let the common core $D_-^{(j)}$ consist of $n_-$ copies of $\bx_1^{(j)}$. Then
\[
K\left(X_-^{(j)},X_-^{(j)}\right)=\mathbf 1\mathbf 1^T,\qquad
\bk_x=k\left(\bx,\bx_1^{(j)}\right)\mathbf 1,\qquad
A=\mathbf 1\mathbf 1^T+r^2I.
\]
By the Sherman--Morrison formula,
\[
\left(\mathbf 1\mathbf 1^T+r^2I\right)^{-1}
=
\frac{1}{r^2}I
-
\frac{1}{r^2(n_-+r^2)}\,\mathbf 1\mathbf 1^T.
\]
Consequently, for any $\bx,\bx'\in\Omega_X$,
\[
k_{D_-}^{(j)}(\bx,\bx') = k(\bx,\bx') - k\left(\bx,\bx_1^{(j)}\right)k\left(\bx',\bx_1^{(j)}\right)\mathbf 1^T \left(\mathbf 1\mathbf 1^T+r^2I\right)^{-1}\mathbf 1
=
k(\bx,\bx')-a_n k\left(\bx,\bx_1^{(j)}\right)k\left(\bx',\bx_1^{(j)}\right).
\]
Taking $\bx=\bx'=\bx_2^{(j)}$, and using $k\left(\bx_2^{(j)},\bx_2^{(j)}\right)=1$, gives
\[
k_{D_-}^{(j)}\left(\bx_2^{(j)},\bx_2^{(j)}\right)
=
1-a_n k\left(\bx_1^{(j)},\bx_2^{(j)}\right)^2
\xrightarrow{j\to\infty}
1-a_n\kappa^2.
\]
Thus, $V_n(r)\ge 1-a_n\kappa^2$. Combining this with the upper bound proves $V_n(r)=1-a_n\kappa^2$.
\end{proof}

The problem of finding bounds for $\Delta_n(r)$ is more subtle. Below, we treat several special cases which illustrate how $\Delta_n(r)$ can exhibit a range of different behaviours with respect to $n$. Throughout, we restrict our considerations to response models with the response bounded in absolute value by a constant $M_Y>0$. We start with the generic upper bound that holds for general bounded-response models.

\begin{lemma}[Generic bounded-response bound for $\Delta_n(r)$]
\label{lemma:boundedResp_Delta_bound}
Assume $k(\bx,\bx)\le 1$ for all $\bx\in\Omega_X$, and suppose that all
training responses satisfy $|y_i|\le M_Y$. Then
\begin{equation}\label{eq:delta_gen1}
\Delta_n(r)
\le \bar\Delta_n(r):=
2M_Y\max_{0\le u\le V_n(r)}
\frac{
\sqrt u+
\frac{\sqrt{n-1}}{r}u\sqrt{1-u}
}{r^2+u}\leq \frac{M_Y}{r}\left(1+\frac{\sqrt{n-1}}{2r}\right)=\mathcal O(\sqrt{n}).
\end{equation}
We also have
\begin{equation}\label{eq:delta_gen2}
\Delta_n(r)
\le \hat\Delta_n(r):=\frac{M_Y\sqrt{n-1}}{r}\frac{V_n(r)}{r^2+V_n(r)}+2M_Y\max_{0\le u\le V_n(r)}\frac{\sqrt{u}}{r^2+u}.
\end{equation}
Thus, $\Delta_n(r)\leq \min\{\bar\Delta_n(r),\hat\Delta_n(r)\}$ where due to the difference in scaling at small $r$, $\bar\Delta_n(r)$ is tighter for large $r$ and $\hat\Delta_n(r)$ is tighter for small $r$.
\end{lemma}

\begin{proof}
Fix neighbouring datasets $D=D_-\cup\{(\bx,y)\}$, $D'=D_-\cup\{(\bx',y')\}$. Let $X_-=\{\bx_1,\dots,\bx_{n-1}\}$ and $\by_-$ denote the common-core covariates and responses, and write for short $K_- = K(X_-,X_-)$, $\bk(\bx)=K(X_-,\bx)$ $\ba(\bx)=\left(K_-+r^2I\right)^{-1}\bk(\bx)$. From the definition of posterior covariance we have $h_{\bx}:=k_{D_-}(\bx,\cdot)=k(\bx,\cdot)-\ba(\bx)^T.\bk(\cdot)\in\mathcal{H}_k$.
A series of direct expansions gives
\begin{align*}
\|h_{\bx}\|_{\mathcal{H}_k}^2=k(\bx,\bx)-2\ba^T.\bk(\bx)+\ba^T.K_-\ba=k(\bx,\bx)-2\ba^T.\bk(\bx)+\ba^T.\left(K_-+r^2I-r^2I\right)\left(K_-+r^2I\right)^{-1}\bk(\bx)
\\
=k(\bx,\bx)-2\ba^T.\bk(\bx)-r^2\|\ba\|_2^2+\ba^T.\bk(\bx)=v(\bx)-r^2\|\ba\|_2^2,
\end{align*}
where we have used the notation $v(\bx)=k_{D_-}(\bx,\bx)$. Similarly, we get $\langle k(\bx,\cdot),k_{D_-}(\bx,\cdot)\rangle_{\mathcal{H}_k}=v(\bx)$. Cauchy-Schwarz in $\mathcal{H}_k$ and $k(\bx,\bx)\leq 1$ imply
\[
v(\bx)^2\leq\|k(\bx,\cdot)\|_{\mathcal{H}_k}^2\,\|h_{\bx}\|_{\mathcal{H}_k}^2=k(\bx,\bx)\, \left(v(\bx)-r^2\|\ba\|_2^2\right)\leq v(\bx)-r^2\|\ba\|_2^2.
\]
Equivalently, $r^2\|\ba(\bx)\|_2^2\leq v(\bx)(1-v(\bx))$. For any $\bx\in\Omega_X$, the common-core posterior mean reads $\mu_{D_-}(\bx)=\ba(\bx)^T.\by_-$. Using Cauchy-Schwarz and the fact that $\|\by_-\|^2\leq (n-1)M_Y^2$ we get for any $|y|<M_Y$
\[
|y-\mu_{D_-}(\bx)|\leq M_Y+ \|\ba(\bx)\|_2\, \|\by_-\|_2\leq M_Y\left(1+\frac{\sqrt{n-1}}{r}\sqrt{v(\bx)(1-v(\bx))}\right).
\]
By the one-step update formula \eqref{eq:update_formula} we have $\mu_D-\mu_{D'}
=
\eta\,k_{D_-}(\cdot,\bx)
-
\eta'\,k_{D_-}(\cdot,\bx'),
$ for suitable scalars $\eta,\eta'$. Putting $k_{D_-}(\bx,\bx)=v$ and $k_{D_-}(\bx',\bx')=v'$ and using the reproducing property in
$\mathcal H_-:=\mathcal H_{k_{D_-}}$ we have
\begin{align*}
\|\mu_D-\mu_{D'}\|_{\mathcal H_-} & \leq \|\eta k_{D_-}(\cdot,\bx)\|_{\mathcal H_-}+\|\eta' k_{D_-}(\cdot,\bx')\|_{\mathcal H_-}
=
|\eta|\,\sqrt{v}+|\eta'|\,\sqrt{v'}
\\
&=|y-\mu_{D_-}(\bx)|\frac{\sqrt{v}}{v+r^2}+|y'-\mu_{D_-}(\bx')|\frac{\sqrt{v'}}{v'+r^2}\\
&\le
M_Y
\frac{
\sqrt v+
\frac{\sqrt{n-1}}{r}v\sqrt{1-v}
}{r^2+v}
+
M_Y
\frac{
\sqrt{v'}+
\frac{\sqrt{n-1}}{r}v'\sqrt{1-v'}
}{r^2+v'}.
\end{align*}
where in the penultimate step we have applied the formulae for $\eta,\eta'$ from \eqref{eq:update_formula}. Since $0\le v,v'\le V_n(r)$, each of the two
summands is bounded by the same supremum. Consequently,
\[
\|\mu_D-\mu_{D'}\|_{\mathcal H_-}
\le
2M_Y\max_{0\le u\le V_n(r)}
\frac{
\sqrt u+
\frac{\sqrt{n-1}}{r}u\sqrt{1-u}
}{r^2+u}.
\]
Taking the supremum over neighbouring pairs gives the bound \eqref{eq:delta_gen1}.

Let us next prove the bound \eqref{eq:delta_gen2}. To this end, rewrite the update formula \eqref{eq:update_formula} as 
\[
\mu_D-\mu_{D'}=\frac{y}{r^2+v}h_{\bx}
-
\frac{y'}{r^2+v'}h_{\bx'}
-
(H_{\bx}-H_{\bx'})\mu_-,\quad H_{\bx}
:=
\frac{h_{\bx}\otimes h_{\bx}^*}{r^2+v(\bx)},
\]
where $h_{\bx}^*$ is the $\mathcal{H}_-$-dual of $h_{\bx}$ i.e. $(h_{\bx}\otimes h_{\bx}^*)f=\langle h_{\bx},f\rangle_{\mathcal{H}_-}\,h_{\bx}$. This decomposition implies
\begin{equation}\label{eq:delta_bnd_intermediate}
\|\mu_D-\mu_{D'}\|_{\mathcal{H}_-}\leq M_Y\left(\frac{\sqrt{v}}{r^2+v}+\frac{\sqrt{v'}}{r^2+v}\right)+\|(H_{\bx}-H_{\bx'})\mu_-\|_{\mathcal{H}_-}\leq 2M_Y\max_{0\le u\le V_n(r)}\frac{\sqrt{u}}{r^2+u}+\|H_{\bx}-H_{\bx'}\|_{op}\|\mu_-\|_{\mathcal{H}_-},
\end{equation}
where $\|\cdot\|_{op}$ is the operator norm. $H_{\bx}$, $H_{\bx'}$ are rank-one operators in $\mathcal{H}_-$ and their respective nonzero eigenvalues are
\[
\frac{\|h_{\bx}\|_{\mathcal{H}_-}^2}{r^2+v(\bx)}=\frac{v(\bx)}{r^2+v(\bx)}\leq c_n(r):=\frac{V_n(r)}{r^2+V_n(r)}.
\]
Thus, $0\preceq  H_{\bx},H_{\bx'}\preceq c_n(r) I$ which implies $-c_n(r) I\preceq  H_{\bx}-H_{\bx'}\preceq c_n(r) I$ and hence $\|H_{\bx}-H_{\bx'}\|_{op}\leq c_n(r)$. It remains to bound $\|\mu_-\|_{\mathcal{H}_-}$. To this end, rewrite
\begin{align*}
\mu_-(\bx)=k(\bx,X_-)\left(K_-+r^2I\right)^{-1}\by_-=\frac{1}{r^2}k(\bx,X_-)\left(K_-+r^2I\right)^{-1}\left[ \left(K_-+r^2I\right) - K_-\right]\by_-
\\
=\frac{1}{r^2}\left(k(\bx,X_-)-k(\bx,X_-)\left(K_-+r^2I\right)^{-1}K_-\right)\by_-=\frac{1}{r^2} k_{D_-}(\bx,X_-)\by_-.
\end{align*}
Similarly, we get the identity $k_{D_-}(X_-,X_-)= r^2 K_-(K_-+r^2I)^{-1}$. Thus, 
\[
\|\mu_-(\cdot)\|_{\mathcal{H}_-}^2=\frac{1}{r^4}\by_-^Tk_{D_-}(X_-,X_-)\by_-=\frac{1}{r^2}\by_-^TK_-(K_-+r^2I)^{-1}\by_-\leq \frac{\|\by_-\|_2^2}{r^2}\leq \frac{(n-1)M_Y^2}{r^2},
\]
where we have used the fact that $K_-(K_-+r^2I)^{-1}\preceq I$. Combining this with inequality \eqref{eq:delta_bnd_intermediate} and taking the supremum over neighbouring datasets yields the final result \eqref{eq:delta_gen2}.
\end{proof}

The Lemma below shows that the $\sqrt{n}$-growth of the worst-case upper bound is in fact tight.

\begin{lemma}
\label{lemma:Delta_sqrt_growth_example}
Let $\Omega_X=[0,1]$ and suppose that all training responses satisfy $|y_i|\le M_Y$. There exists a positive semidefinite kernel $k:\Omega_X\times\Omega_X\to\mathbb R$ satisfying $\sup_{x\in\Omega_X}k(x,x)\le 1$ for which $\Delta_n(r)$ has the following lower bound.
\[
\Delta_n(r)
\ge
\frac{
M_Y\left(1+r^2+\sqrt{n-1}\right)
}
{
r\sqrt{1+r^2}\,(2+r^2)
}.
\]
\end{lemma}

\begin{proof}
Let $(\be_i)_{i\ge1}$ denote the standard basis of $\ell^2$. Choose mutually
distinct points
\[
x_0,\quad a_1,a_2,\dots,\quad b_1,b_2,\dots
\]
in $[0,1]$. Define a feature map $\varphi:[0,1]\to\ell^2$ by
\[
\varphi(x_0):=0,\qquad
\varphi(a_i):=\be_i,\qquad
\varphi(b_m):=\frac{1}{\sqrt{m}}\sum_{i=1}^{m}\be_i,
\]
and set $\varphi(x):=0$ for all remaining $x\in[0,1]$. Now define the GP kernel through this feature map $k(x,x'):=\langle \varphi(x),\varphi(x')\rangle_{\ell^2}$. Then $k$ has bounded diagonal since  $k(x,x)=\|\varphi(x)\|_{\ell^2}^2\leq 1$ for every $x\in[0,1]$.

Fix $n\ge2$. Consider the common-core dataset $D_-:=\{(a_i,M_Y)\}_{i=1}^{n-1}.$ Then $K(X_-,X_-)=I$ and  $K(X_-,b_{n-1})=\frac{1}{\sqrt{n-1}}\mathbf 1$. Thus, the common-core posterior mean at $x=b_{n-1}$ is
\[
\mu_{D_-}(b_{n-1})
=
K(b_{n-1},X_-)\left(K(X_-,X_-)+r^2I\right)^{-1}\by_-
=
\frac{M_Y\sqrt {n-1}}{1+r^2}.
\]
The corresponding common-core posterior variance is
\begin{align*}
v
:=
k_{D_-}(b_{n-1},b_{n-1})
&=
k(b_{n-1},b_{n-1})
-
K(b_{n-1},X_-)\left(K(X_-,X_-)+r^2I\right)^{-1}K(X_-,b_{n-1})
\\
&=
1-\frac{1}{1+r^2}
=
\frac{r^2}{1+r^2}.
\end{align*}

Now define neighbouring datasets $D:=D_-\cup\{(b_d,-M_Y)\}$, $D':=D_-\cup\{(x_0,M_Y)\}$. Since $\varphi(x_0)=0$, we have $k(\cdot,x_0)\equiv0$ and hence
$k_{D_-}(\cdot,x_0)\equiv0$. Thus, by the one-step update formula \eqref{eq:update_formula}
\[
\mu_D-\mu_{D'}
=
\eta\,k_{D_-}(\cdot,b_{n-1}), \qquad \eta
=
\frac{-M_Y-\mu_{D_-}(b_{n-1})}{v(\bx)+r^2}.
\]
Using the expressions above, we have
\[
\eta
=
-M_Y\left(1+\frac{\sqrt {n-1}}{1+r^2}\right)\left(\frac{r^2}{1+r^2}+r^2\right)^{-1}
=
-\frac{M_Y(1+r^2+\sqrt {n-1})}
{r^2(2+r^2)}.
\]
By the RKHS property we have we have $\|\mu_D-\mu_{D'}\|_{\mathcal H_-}^2
=
\eta^2 k_{D_-}(b_{n-1},b_{n-1})
=
\eta^2v.
$
Substituting the expressions for $\eta$ and $v$ gives
\[
\|\mu_D-\mu_{D'}\|_{\mathcal H_-}
=
\frac{
M_Y\left(1+r^2+\sqrt{n-1}\right)
}
{
r\sqrt{1+r^2}\,(2+r^2)
}.
\]
The claimed lower bound follows by taking the supremum over neighbouring datasets $D\sim D'$.
\end{proof}

\begin{lemma}
\label{lemma:delta_lower_constant}
Let $k:\Omega_X\times\Omega_X\to\mathbb R$ be a positive semidefinite kernel
satisfying $\sup_{\bx\in\Omega_X}k(\bx,\bx)\le 1$. Assume that there exist admissible datapoints $(\bx_0,y_0)$ and
$(\bx_1,y_1)$ such that $y_0>0$, $y_1<0$ and 
\[
k(\bx_0,\bx_0)=1,
\qquad
k(\bx_0,\bx_1)>0,
\qquad
k(\bx_1,\bx_1)-k(\bx_0,\bx_1)^2>0.
\]
Then
\[
\Delta_n(r)
\ge
\frac{|y_1|\sqrt{k(\bx_1,\bx_1)-k(\bx_0,\bx_1)^2}}{1+r^2}>0.
\]
\end{lemma}
\begin{proof}
Let $D_-$ consist of $n-1$ repeated copies of $(\bx_0,y_0)$ and set $D=D_-\cup\{(\bx_0,y_0)\}$,
$D'=D_-\cup\{(\bx_1,y_1)\}$ and define
\[
a_n:=\frac{n-1}{n-1+r^2},
\qquad
\gamma_n:=\frac{n}{n+r^2}
\]
and let $X_T=\{\bx_1\}$. Let $D_-$ consist of $n-1$ repeated copies of $(\bx_0,y_0)$. Then, $K(X_-,X_-)=\mathbf 1 \mathbf 1^T$ and using the Sherman-Morrison formula it is straightforward to verify that
\[
\mu_{D_-}(\bx_1)=a_n k(\bx_0,\bx_1)\,y_0,
\quad
\mu_D(\bx_1)=\gamma_n k(\bx_0,\bx_1)\,y_0,
\quad 
v_n:=k_{D_-}(\bx_1,\bx_1)
=
k(\bx_1,\bx_1)-a_n k(\bx_0,\bx_1)^2>0.
\]
Also, because $a_n\ge0$ and $\sup_{\bx}k(\bx,\bx)\le1$, we have $v_n\le k(\bx_1,\bx_1)\le1$. For $D'$, the one-point update formula \eqref{eq:update_formula} gives
\begin{align*}
\mu_{D'}(\bx_1)
=
\mu_{D_-}(\bx_1)
+
k_{D_-}(\bx_1,\bx_1)
\frac{y_1-\mu_{D_-}(\bx_1)}
{k_{D_-}(\bx_1,\bx_1)+r^2}
& =
a_n k(\bx_0,\bx_1)\,y_0
+
v_n\frac{y_1-a_n k(\bx_0,\bx_1)\,y_0}{v_n+r^2}
\\
& =
\frac{a_n k(\bx_0,\bx_1)\,y_0\,r^2+v_n y_1}{v_n+r^2}.
\end{align*}
Hence,
\begin{align*}
\mu_D(\bx_1)-\mu_{D'}(\bx_1)
& =
\gamma_n k(\bx_0,\bx_1)\,y_0-\frac{a_n k(\bx_0,\bx_1)\,y_0\,r^2+v_n y_1}{v_n+r^2}
\\
& =
\frac{
v_n(\gamma_n k(\bx_0,\bx_1)\,y_0-y_1)
+
k(\bx_0,\bx_1)\,y_0\,r^2(\gamma_n-a_n)
}{v_n+r^2}.
\end{align*}
Since $k(\bx_0,\bx_1)>0$, $y_0>0$, $y_1<0$, and
\[
\gamma_n-a_n
=
\frac{r^2}{(n+r^2)(n-1+r^2)}
\ge0,
\]
both terms in the numerator are nonnegative. Thus
\[
\mu_D(\bx_1)-\mu_{D'}(\bx_1)
\ge
\frac{
v_n(\gamma_n k(\bx_0,\bx_1)\,y_0-y_1)
}{v_n+r^2}
\ge
\frac{v_n |y_1|}{v_n+r^2}.
\]

By the reproducing property in $\mathcal H_{-}$ and Cauchy-Schwarz
\[
|\mu_D(\bx_1)-\mu_{D'}(\bx_1)|=\left|\langle\mu_D-\mu_{D'},k_{D_-}(\bx_1,\cdot)\rangle_{\mathcal H_{-}}\right|
\le
\|\mu_D-\mu_{D'}\|_{\mathcal H_{-}}
\sqrt{k_{D_-}(\bx_1,\bx_1)}
=
\|\mu_D-\mu_{D'}\|_{\mathcal H_{-}}
\sqrt{v_n}.
\]
Therefore
\[
\|\mu_D-\mu_{D'}\|_{\mathcal H_{-}}^2
\ge
\frac{
\bigl(\mu_D(\bx_1)-\mu_{D'}(\bx_1)\bigr)^2
}{v_n}
\ge
|y_1|^2
\frac{v_n}{(v_n+r^2)^2}.
\]
Since $k(\bx_1,\bx_1)-k(\bx_0,\bx_1)^2\le v_n\le1$, we have
\[
\frac{v_n}{(v_n+r^2)^2}
\ge
\frac{k(\bx_1,\bx_1)-k(\bx_0,\bx_1)^2}{(1+r^2)^2}.
\]
Hence
\[
\|\mu_D-\mu_{D'}\|_{\mathcal H_{-}}
\ge
\frac{|y_1|\sqrt{k(\bx_1,\bx_1)-k(\bx_0,\bx_1)^2}}{1+r^2}.
\]
Taking the supremum over admissible neighbouring pairs gives proves the claim.
\end{proof}

On the other hand, there exist typical situations where $\Delta_n(r)$ is upper-bounded by a constant or even decreases with $n$.

\begin{lemma}[Constant kernel]
\label{lemma:Delta_constant_kernel}
Let $k(\bx,\bx')\equiv 1$ for all $\bx,\bx'\in\Omega_X$  and suppose that all admissible responses satisfy $|y|\le M_Y$. Then
\[
\Delta_n(r)
\le
\frac{2M_Y\sqrt{r^2+n-1}}{r(r^2+n)}=\mathcal O(1/\sqrt{n}).
\]
\end{lemma}

\begin{proof}
Let $D=D_-\cup\{(\bx,y)\}$,  $D'=D_-\cup\{(\bx',y')\}$ where $D_-$ is any common core of size $n-1$. Since the kernel is constant for any $X$ we have $K(X,X)=\mathbf 1 \mathbf 1^T$ and it is straightforward to verify (using Sherman–Morrison formula)  that for any dataset $D$ of size $n$ the posterior mean is a constant function and that the posterior covariance is also constant and given by
\[
\mu_D(\cdot)
=
\frac{\sum_{i=1}^n y_i}{r^2+n},
\qquad
k_{D_-}(\bu,\bv)
=
\frac{r^2}{r^2+n-1}.
\]
Thus, $\mu_D-\mu_{D'}
=
(y-y')/(r^2+n)$ and $\mathcal H_{-}$ consists of constant functions. A constant function $g$ has RKHS norm
\[
\|g\|_{\mathcal H_{-}}
= |g|\, \frac{\sqrt{r^2+n-1}}{r}.
\]
Hence
\[
\|\mu_D-\mu_{D'}\|_{\mathcal H_{-}}
=
\frac{|y-y'|}{r^2+n}
\frac{\sqrt{r^2+n-1}}{r}.
\]
Taking the supremum over $|y|,|y'|\le M_Y$ gives $|y-y'|\le2M_Y$ from which the final result follows.
\end{proof}

\begin{lemma}[Purely diagonal kernel]
\label{lemma:Delta_purely_diagonal_kernel}
Let $k(\bx,\bx')=q(\bx)$ if $\bx=\bx'$ with $0\le q(\bx)\le 1$ and zero otherwise. Suppose that all admissible responses satisfy $|y_i|\le M_Y$. Then
\[
\Delta_n(r)
\le
\begin{cases}
\dfrac{\sqrt{2}M_Y}{r}, & 0<r\le1,\\[2mm]
\dfrac{2\sqrt{2}M_Y}{1+r^2}, & r\ge1.
\end{cases}
\]
\end{lemma}

\begin{proof}
Let $D_-$ be the common core. For $\bu\in\Omega_X$, define the overlap index set 
\[
\mathcal I_\cap(\bu):=\{i:\bx_i=\bu,\;(\bx_i,y_i)\in D_-\} ,\qquad n_\cap(\bu):=\#\mathcal I_\cap(\bu),
\qquad
S_\cap(\bu):=\sum_{i\in \mathcal I_\cap(\bu)} y_i.
\]
For the purely diagonal kernel, the common-core posterior kernel is again
diagonal. Its diagonal value is
\[
v(\bu):=k_{D_-}(\bu,\bu)
=
\frac{q(\bu)r^2}{r^2+n_\cap(\bu)q(\bu)}\le q(\bu)\le1.
\]
The common-core posterior mean is bounded as
\[
|\mu_{D_-}(\bu)|
=
\frac{q(\bu)|S_\cap(\bu)|}
{r^2+n_\cap(\bu)q(\bu)}
\le
\frac{|S_\cap(\bu)|}
{n_\cap(\bu)}
\le M_Y.
\]
where we have used the fact that $|S_\cap(\bu)|\le n_\cap(\bu)M_Y$. Hence, for any additional admissible observation $(\bu,y)$ we have $|y-\mu_{D_-}(\bu)|\le 2M_Y$. Now let $D=D_-\cup\{(\bx,y)\}$ and $D'=D_-\cup\{(\bx',y')\}$. If $\bx\neq\bx'$, then $k_{D_-}(\bx,\cdot)$ and
$k_{D_-}(\bx',\cdot)$ are orthogonal in $\mathcal H_{-}$, because
$k_{D_-}$ is diagonal. Then, the one-point update formula \eqref{eq:update_formula} yields
\[
\|\mu_D-\mu_{D'}\|_{\mathcal H_{-}}
\le
2\sqrt{2}M_Y
\sup_{0\le v\le1}
\frac{\sqrt v}{r^2+v}.
\]
It is straightforward to check that the same bound holds also when $\bx=\bx'$. The final result is obtained from the fact that
\[
\sup_{0\le v\le1}\frac{\sqrt v}{r^2+v}
=
\begin{cases}
\dfrac{1}{2r}, & r^2\le 1,\\
\dfrac{1}{r^2+1}, & r^2>1.
\end{cases}
\]
\end{proof}

\begin{lemma}[RKHS Responses]
\label{lemma:Delta_rkhs_response_bound}
Let $k$ be a positive semidefinite normalised kernel with bounded diagonal
with RKHS $\mathcal H_k$ and let $r>0$. Suppose that the responses are
generated by a fixed function $f_*\in\mathcal H_k$ i.e., for every
admissible dataset $D=(X,\by)$ with $X=(\bx_1,\ldots,\bx_n)$, we have $y_i=f_*(\bx_i)$, $i=1,\ldots,n$. Let $\mu_D$ denote the GP posterior mean as defined in Section \ref{sec:setup}. Then 
\[
\Delta_n(r)
\le
\|f_*\|_{\mathcal H_k}\,
\frac{V_n(r)}{r^2+V_n(r)}.
\]
\end{lemma}

\begin{proof}
Fix neighbouring datasets $D=D_-\cup\{(\bx,f_*(\bx))\}$, $D'=D_-\cup\{(\bx',f_*(\bx'))\}$. We use the same notation and similar methodology as in the proof of
Lemma~\ref{lemma:boundedResp_Delta_bound}. In particular, let 
\[
K_-=K(X_-,X_-), \quad \bk(\bx)=K(X_-,\bx), \quad \ba(\bx)=(K_-+r^2I)^{-1}\bk(\bx)
\]
and write $h_{\bx}:=k_{D_-}(\bx,\cdot)$, $v(\bx)=k_{D_-}(\bx,\bx)=\|h_{\bx}\|_{\mathcal H_-}^2$. We first derive an explicit expression for the inner product in
$\mathcal H_-$. The spaces $\mathcal H_-$ and $\mathcal H_k$ agree as
sets and for $g,g'\in\mathcal H_k$ their inner products are related
by
\begin{equation}\label{eq:Hminus_inner_product}
\langle g,g'\rangle_{\mathcal H_-}
=
\langle g,g'\rangle_{\mathcal H_k}
+
\frac{1}{r^2}g(X_-)^Tg'(X_-).
\end{equation}
Indeed, since $h_{\bx}(X_-)
=
\bk(\bx)-K_-\ba(\bx)
=
r^2\ba(\bx)$ (see Lemma~\ref{lemma:boundedResp_Delta_bound} for the proof of the last equality), we have for every $g\in\mathcal H_k$ that
\begin{align*}
\langle g,h_{\bx}\rangle_{\mathcal H_-}
&=
\langle g,h_{\bx}\rangle_{\mathcal H_k}
+
\frac{1}{r^2}g(X_-)^Th_{\bx}(X_-)
=
g(\bx)-g(X_-)^T\ba(\bx)
+
g(X_-)^T\ba(\bx)
=
g(\bx).
\end{align*}
Thus, $\mathcal H_k$ equipped with the inner product on the RHS of~\eqref{eq:Hminus_inner_product} is a Hilbert space whose reproducing kernel is $k_{D_-}$. By uniqueness of the RKHS associated with a reproducing kernel, this space coincides with $\mathcal H_-=\mathcal H_{k_{D_-}}$ as a Hilbert space of functions and \eqref{eq:Hminus_inner_product} holds for every $g,g'\in\mathcal H_-$. In particular,
\begin{equation}\label{eq:Hk_Hminus_norm}
\|g\|_{\mathcal H_k}\leq \|g\|_{\mathcal H_-}.
\end{equation}

Consider now the residual difference $\rho_-:=f_*-\mu_-\in\mathcal{H}_k$. We claim that
\begin{equation}\label{eq:common_core_residual_bound}
\|\rho_-\|_{\mathcal H_-}
\leq
\|f_*\|_{\mathcal H_k}.
\end{equation}
To see this, put $\bb:=(K_-+r^2I)^{-1}f_*(X_-)$. Then $\mu_-(\cdot)=\bb^T\bk(\cdot)$ and $\rho_-(X_-)
=
f_*(X_-)-K_-\bb
=
r^2\bb
$.
By~\eqref{eq:Hminus_inner_product}, we get
\begin{align*}
\langle \rho_-,\rho_-\rangle_{\mathcal H_-}
=
\langle f_*-\mu_-,\rho_-\rangle_{\mathcal H_k}
+
\frac{1}{r^2}\rho_-(X_-)^T\rho_-(X_-)
=
\langle f_*,\rho_-\rangle_{\mathcal H_k}
-
\bb^T\rho_-(X_-)
+
\bb^T\rho_-(X_-)
=
\langle f_*,\rho_-\rangle_{\mathcal H_k}.
\end{align*}
Using~\eqref{eq:Hk_Hminus_norm} and Cauchy-Schwarz gives $\|\rho_-\|_{\mathcal H_-}^2
=
\langle f_*,\rho_-\rangle_{\mathcal H_k}
\leq
\|f_*\|_{\mathcal H_k}
\|\rho_-\|_{\mathcal H_k}
\leq
\|f_*\|_{\mathcal H_k}
\|\rho_-\|_{\mathcal H_-}
$
which proves~\eqref{eq:common_core_residual_bound}.

As in the proof of
Lemma~\ref{lemma:boundedResp_Delta_bound} define
\[
H_{\bx}
:=
\frac{h_{\bx}\otimes h_{\bx}^*}{r^2+v(\bx)},
\]
where $(h_{\bx}\otimes h_{\bx}^*)g
=
\langle h_{\bx},g\rangle_{\mathcal H_-}h_{\bx}$. By the reproducing property in $\mathcal H_-$,
\[
H_{\bx}\rho_-
=
\frac{\rho_-(\bx)}{r^2+v(\bx)}h_{\bx}
=
\frac{f_*(\bx)-\mu_-(\bx)}
{r^2+v(\bx)}h_{\bx}.
\]
Thus, the one-point update formula~\eqref{eq:update_formula} can be written as
\begin{equation}\label{eq:rkhs_operator_update}
\mu_D-\mu_{D'}
=
(H_{\bx}-H_{\bx'})\rho_-.
\end{equation}

As explained in the proof of
Lemma~\ref{lemma:boundedResp_Delta_bound}, the operators $H_{\bx}$ and $H_{\bx'}$ are positive semidefinite
rank-one operators in $\mathcal H_-$ and their respective nonzero
eigenvalues are
\[
\frac{\|h_{\bx}\|_{\mathcal H_-}^2}{r^2+v(\bx)}
=
\frac{v(\bx)}{r^2+v(\bx)}
\leq
c_n(r):=\frac{V_n(r)}{r^2+V_n(r)}.
\]
Hence, we have $0\preceq H_{\bx},H_{\bx'}\preceq c_n(r)I$ and thus $-c_n(r)I
\preceq
H_{\bx}-H_{\bx'}
\preceq
c_n(r)I
$
and 
$
\|H_{\bx}-H_{\bx'}\|_{\mathrm{op}}
\leq c_n(r)$. Combining this with~\eqref{eq:rkhs_operator_update} and
\eqref{eq:common_core_residual_bound} yields
\[
\|\mu_D-\mu_{D'}\|_{\mathcal H_-}
\leq
\|H_{\bx}-H_{\bx'}\|_{\mathrm{op}}
\|\rho_-\|_{\mathcal H_-}
\leq
\frac{V_n(r)}{r^2+V_n(r)}
\|f_*\|_{\mathcal H_k}.
\]
Taking the supremum over neighbouring datasets proves the result.
\end{proof}

\subsection{1D Exponential Kernel}

\begin{lemma}
\label{lemma:Delta_exp_1D_mu_bounded}
Let $\Omega_X\subseteq\mathbb R$ and let $k(x,x')=\exp\left(-|x-x'|/\ell\right)$ with $\ell>0$. Suppose that all admissible responses satisfy $|y_i|\le M_Y$. Then, for every admissible dataset $D$ and for all $x\in \mathbb R$ we have $|\mu_{D}(x)|\leq M_Y$.
\end{lemma}
\begin{proof}
For $D=(X,\by)$ write $X=(x_1,\dots,x_n)$ and assume that $x_1<x_2<\dots<x_n$. Denoting $K:=K(X,X)$ we have 
\begin{align*}
\mu_D(X) = K(K+r^2 I)^{-1}\by & = K K^{-1}(I+r^2 K^{-1})^{-1}\by=(I+r^2K^{-1})^{-1}\by.
\end{align*}
For $i<j$ we can represent the entries of $K$ as
\[
K_{ij}=\exp\left(\frac{x_i-x_j}{\ell}\right)=\prod_{k=1}^{j-1}\rho_k,\qquad \rho_k:=\exp\left(\frac{x_k-x_{k+1}}{\ell}\right).
\]
One can verify by a straightforward calculation that $K^{-1}$ is tridiagonal with entries given by
\begin{gather*}
\left[K^{-1}\right]_{11}=\frac{1}{1-\rho_1^2}, \quad \left[K^{-1}\right]_{nn}=\frac{1}{1-\rho_{n-1}^2}, \quad \left[K^{-1}\right]_{ii}=\frac{1}{1-\rho_{i-1}^2}+\frac{\rho_i^2}{1-\rho_{i}^2}, \quad 2\le i\le n-1,
\\
\left[K^{-1}\right]_{i,i+1}=\left[K^{-1}\right]_{i+1,i} = -\frac{\rho_i}{1-\rho_{i}^2}.
\end{gather*}
Using these formulas we also get that $K^{-1}$ has nonpositive off-diagonal entries and non-negative row-sums. Consequently,  $A:=I+r^2K^{-1}$ has nonpositive off-diagonal entries and row-sums $\ge 1$.

Now find index $i_*$ such that $\mu_D(x_{i_*})=\max_{i}\mu_D(x_{i})$. Then, for any $j$ we have $\mu_D(x_{j})\le \mu_D(x_{i_*})$ Since $A_{ij}\le 0$ for $i\neq j$, we also have $A_{i_*,j}\mu_D(x_{j})\ge A_{i_*,j}\mu_D(x_{i_*})$. Summing over $j$ yields
\[
y_{i_*} = \left[A\,\mu_D(X)\right]_{i_*} = \sum_{j}A_{i_*,j}\mu_D(x_{j})\ge \mu_D(x_{i_*})\sum_{j}A_{i_*,j}\geq \mu_D(x_{i_*}),
\]
since the row-sum satisfies $\sum_{j}A_{i_*,j}\ge 1$. By the response-boundedness, this yields $\mu_D(x_{j})\le\mu_D(x_{i_*})\le M_Y$. By picking $\underline i_*$ such that $-\mu_D(x_{\underline i_*})=\max_{i}\left(-\mu_D(x_{i})\right)$ and repeating the reasoning above, we get $-\mu_D(x_{j})\le-\mu_D(x_{\underline i_*})\le M_Y$, i.e. $\mu_D(x_{j})\ge -M_Y$ and consequently
\[
\left|\mu_D(x_{j})\right|\leq M_Y\quad\mathrm{for\ all}\quad x_j\in X.
\]
It remains to show that $|\mu_D(x)|\leq M_Y$ for any $x$ outside the training set $X$. Assume first that $x_i\le x\le x_{i+1}$ for some $1\le i<n$. We claim that $|\mu_D(x)|$ is convex on the interval $[x_i,x_{i+1}]$, thus
\[
\max_{x_i\le x\le x_{i+1}}\,|\mu_D(x)|=\max\{|\mu_D(x_i)|,|\mu_D(x_{i+1})|\}\leq M_Y.
\] 
To see this, note that $\mu_D(x)$ as a function of $x$ is a linear combination of $e^{-x/\ell}$ and $e^{x/\ell}$, so it satisfies the equation $\partial_x^2\mu_D(x) = \mu_D(x)/\ell^2$. In the region where $\mu_D(x)\ge 0$ this gives $\partial_x^2|\mu_D(x)|>0$ i.e. $|\mu_D(x)|$ is convex. In the region where $\mu_D(x)<0$ we have $|\mu_D(x)|=-\mu_D(x)$ and thus $\partial_x^2|\mu_D(x)|=-\mu_D(x)/\ell^2>0$ as well. The same conclusion holds on the exterior intervals (if they are present). Indeed, on the left exterior interval $\Omega_X\cap(-\infty,x_1)$, the function $\mu_D(x)$ is of the form $\mu_D(x)=a_- e^{x/\ell}$ for some constant $a_-$, and hence $|\mu_D(x)|$ is maximised at $x_1$ if $x_1$ is finite, while it tends to zero as $x_1\to-\infty$ when $\Omega_X$ is unbounded to the left. Analogous reasoning holds for the right exterior interval.

When the dataset $D$ contains repeated covariates, then we consider the collapsed dataset  $\bar D = (\bar X, \bar \by)$ of the unique covariates $\bar x_1<\bar x_2<\dots<\bar x_m$ where $\bar x_k$ appears $n_k$ times in the original $D$, $n_k\ge1$. Suppose that $y_{k,1},\dots, y_{k,n_k}$ are the responses of $\bar x_k$ in $D$. Then, we define the collapsed response of $\bar x_k$ as $\bar y_k:=\frac{1}{n_k}\sum_{i=1}^{n_k} y_{k,i}$. In this notation, at the unique training points we have
\[
\mu_D(\bar X)=\bar\mu_{\bar D}(\bar X):=(I+R \bar K^{-1})^{-1}\bar \by,
\]
 where $R$ is a diagonal matrix with entries $R_{kk}=r^2/n_k$ and $\bar K = K(\bar X, \bar X)$. The rest of the proof follows exactly the same as before with $\mu_D$ replaced with $\bar\mu_D$.
\end{proof}

\begin{lemma}[Exponential kernel in 1D]
\label{lemma:Delta_exp_1D}
Let \(\Omega_X\subseteq\mathbb R\) be a possibly unbounded interval and let $k(x,x')=\exp\left(-|x-x'|/\ell\right)$, 
$x,x'\in \Omega_X$ with $\ell>0$. Suppose that all admissible responses satisfy $|y_i|\le M_Y$. Then, for every $r>0$,
\[
\Delta_n(r)
\le 4M_Y \sqrt{\Phi_n(r)},\qquad 
\Phi_n(r)=
\sup_{0\le u\le V_n(r)}
\frac{u}{(u+r^2)^2}=
\begin{cases}
\dfrac{1}{4r^2}, & V_n(r)\ge r^2,\\[2mm]
\dfrac{V_n(r)}{\left(V_n(r)+r^2\right)^2}, & V_n(r)\le r^2.
\end{cases}
\]
\end{lemma}

\begin{proof}
Let $D-=(X_-,\by_-)$ be the common core of $D\sim D'$. Denote by $\mu_{D_-}$ and $k_{D_-}$ the corresponding posterior mean and covariance and set $v(\bx)=k_{D_-}(\bx,\bx)$. By the one-step update formula \ref{eq:update_formula} we have
\[
\mu_{D_-\cup\{(\bx,y)\}}-\mu_{D_-}
=
\frac{y-\mu_{D_-}(\bx)}{r^2+v(\bx)}\,k_{D_-}(\bx,\cdot).
\]
The reproducing property gives $\sqrt{v(\bx)}=\|k_{D_-}(\bx,\cdot)\|_{\mathcal H_-}$. The 1D-exponential kernel has the property that $|\mu_{D}(\bx)|\leq M_Y$ for any $D$ and any $\bx$, see Lemma \ref{lemma:Delta_exp_1D_mu_bounded}. Hence,
\[
\left|\frac{y-\mu_{D_-}(\bx)}{r^2+v(\bx)}\right|\leq \frac{2M_Y}{r^2+v(\bx)}.
\]
Set $D=D_-\cup\{(\bx,y)\}$ and $D'=D_-\cup\{(\bx',y')\}$. By the one-step update formula \ref{eq:update_formula} we have
\[
\|\mu_{D}-\mu_{D'}\|_{\mathcal H_-} \leq 2M_Y\left(\frac{\|k_{D_-}(\bx,\cdot)\|_{\mathcal H_-}}{r^2+v(\bx)} + \frac{\|k_{D_-}(\bx',\cdot)\|_{\mathcal H_-}}{r^2+v(\bx')}\right) \le 4M_Y \sup_{0\leq v \leq V_n(r)}\frac{\sqrt{v}}{r^2+v}=4M_Y\sqrt{\Phi_n(r)}.
\]
Taking the supremum over $D\sim D'$ yields the result.
\end{proof}

\begin{lemma}[Refined mean-term bounds]
\label{lem:refined_mean_term}
Suppose that a uniform pointwise sensitivity bound
\[
S\geq
\sup_{D\sim D'}\sup_{\bx\in\Omega_X}
|\mu_D(\bx)-\mu_{D'}(\bx)|
\]
is available. Then, for every $1<\alpha<1+\frac{r^2}{V_n(r)}$ one has for every test set $X_T$ and every $D\sim D'$
\[
\frac{\alpha}{2\sigma^2}
\delta_T^TK_\alpha^\dagger\delta_T
\leq
\frac{\alpha}{2\sigma^2}
\left[
\Delta_n^2(r)
+
\frac{\alpha
\min\left\{
V_n(r)\Delta_n^2(r),S^2
\right\}}
{r^2-(\alpha-1)V_n(r)}
\right],
\]
where $\delta_T=\mu_D(X_T)-\mu_{D'}(X_T)$ and $K_\alpha=\alpha K_{D'}+(1-\alpha)K_D$.
\end{lemma}
\begin{proof}
Let $D=D_-\cup\{(\bx,y)\}$, $D'=D_-\cup\{(\bx',y')\}$. Let $\widetilde X_T:=X_T\cup\{\bx,\bx'\}$ and denote the corresponding mean difference and mixture covariance by $\widetilde\delta_T$ and $\widetilde K_\alpha$. Applying the block-inverse formula to $\widetilde K_\alpha$ one can show that adding evaluation points to $X_T$ cannot decrease the associated mean-term i.e.
$
\delta_T^TK_\alpha^\dagger\delta_T
\leq
\widetilde\delta_T^T
\widetilde K_\alpha^\dagger
\widetilde\delta_T$. Thus, it suffices to prove the claimed bound for
$\widetilde X_T$. Relabelling $\widetilde X_T$ as $X_T$, we may thus
assume without loss of generality that $\{\bx,\bx'\}\subseteq X_T$.

For brevity, use the notation from the update formula \eqref{eq:update_formula} and write $v=v(\bx)$, $v'=v(\bx')$, $\bc=\bc(\bx)$, $\bc'=\bc(\bx')$. The update formula \eqref{eq:update_formula} for covariance gives
\begin{equation}\label{eq:K_alpha_succ_Mx}
K_\alpha = C
+
(\alpha-1)\frac{\bc.\bc^T}{r^2+v}
-
\alpha\frac{\bc'.{\bc'}^T}{r^2+v'}\succeq
C-\alpha\frac{\bc'.{\bc'}^T}{r^2+v'}
=:M_{\bx'}.
\end{equation}
As shown in the proof of Proposition~\ref{prop:rdp_main},
$\bc,\bc'$ and $\delta_T$ belong to $\operatorname{range}(C)$ and $K_\alpha$ has the same nullspace as
$C$ on the admissible range of $\alpha$. So, we may restrict all matrices to the common support $\operatorname{range}(C)$ on which
$C$ is positive definite. On this support $M_{\bx'}$ is positive definite and $\ker M_{\bx}=\ker C$. Indeed, as shown in Lemma~\ref{lem:local_common_core_sandwich} for any vector $\ba$ we have $\left(\ba^T\bc(\bx)\right)^2\leq v(\bx)\,\ba^TC\ba$. So, 
\begin{align*}
\ba^T.M_{\bx'}\ba=\ba^T.C\ba-\frac{\alpha}{r^2+v'}\left(\ba^T.\bc'\right)^2\geq \ba^T.C\ba\left(1-\frac{\alpha v'}{r^2+v'}\right)
\\
=\frac{r^2-(\alpha-1)v'}{r^2+v'}\ba^T.C\ba\geq \frac{r^2-(\alpha-1)V_n(r)}{r^2+V_n(r)}\ba^T.C\ba.
\end{align*}
On the admissible range of $\alpha$ the prefactor in the above expression is strictly positive. Thus, for every nonzero $\ba\in \operatorname{range}(C)$ we have $\ba^T.M_{\bx'}\ba>0$. This means that applying Moore-Penrose inverse to both sides of~\eqref{eq:K_alpha_succ_Mx} gives $K_\alpha^\dagger \preceq M_{\bx'}^\dagger$. Applying the Sherman-Morrison formula for the inverse or rank-one perturbation gives
\[
M_{\bx'}^\dagger
=
C^\dagger
+
\frac{\alpha}{r^2+v'-\alpha \bc'^T.C^\dagger \bc'}
C^\dagger\bc'{\bc'}^TC^\dagger=
C^\dagger
+
\frac{\alpha}{r^2-(\alpha-1) v'}
C^\dagger\bc'{\bc'}^TC^\dagger,
\]
where we have used that fact that ${\bc'}^TC^\dagger\bc'=v'$ which follows from the assumption that $\bx'\in X_T$. Hence,
\begin{align}
\delta_T^TK_\alpha^\dagger\delta_T=
\delta_T^TC^\dagger\delta_T
+
\frac{\alpha}{r^2-(\alpha-1) v'}
\left({\bc'}^TC^\dagger\delta_T\right)^2.
\label{eq:mean_direction_intermediate}
\end{align}
By the definition of $\Delta_n(r)$ we have $\delta_T^TC^\dagger\delta_T\leq \Delta_n(r)^2$. It remains to bound ${\bc'}^TC^\dagger\delta_T$. To this end, we use Cauchy-Schwarz for the positive semidefinite quadratic form
induced by $C^\dagger$ together with
${\bc'}^TC^\dagger\bc'=v'$ which gives
\begin{align*}
\left({\bc'}^TC^\dagger\delta_T\right)^2
\leq
\left({\bc'}^TC^\dagger\bc'\right)
\left(\delta_T^TC^\dagger\delta_T\right)
=
v'\delta_T^TC^\dagger\delta_T\leq V_n(r)\Delta_n(r)^2.
\end{align*}
On the other hand, $\bx'\in X_T$ implies that ${\bc'}^TC^\dagger\delta_T=\mu_D(\bx')-\mu_{D'}(\bx')$, so $\left({\bc'}^TC^\dagger\delta_T\right)^2\leq S^2$. Combining this with~\eqref{eq:mean_direction_intermediate} yields the result.
\end{proof}

\begin{corollary}[Refined RDP bound for 1D exponential kernel]
\label{coro:exp_1d}
Lemma~\ref{lemma:Delta_exp_1D_mu_bounded} provides $S=2M_Y$. Plugging this and the sensitivity bound from Lemma~\ref{lemma:Delta_exp_1D} into the refined mean-term bound from Lemma~\ref{lem:refined_mean_term} yields the following refined RDP bound.
\[
D_\alpha\left(\Pi_D(X_T)\middle\|\Pi_{D'}(X_T)\right)
\le
-\frac{\log\left(
1-\frac{\alpha(\alpha-1)\tau_n(r)^2}{1+\tau_n(r)}
\right)}{2(\alpha-1)}+
\frac{16M_Y^2\alpha}{2\sigma^2}
\left[
\Phi_n(r)
+
\frac{\alpha
\min\left\{
V_n(r) \Phi_n(r),1/4
\right\}}
{r^2-(\alpha-1)V_n(r)}
\right].
\]
This $S$-refinement improves the scaling of the RDP bound at small $r$.
\end{corollary}

\subsection{$\mathcal O(1)$ sensitivity under data separation condition}
\label{app:Delta_kernel_sum_bounds}

The generic bounded-response estimate in
Lemma~\ref{lemma:boundedResp_Delta_bound} does not exploit spatial decay
of the kernel and consequently grows as $\mathcal O(\sqrt n)$. We next
give an $n$-independent bound under uniform bounds on the relevant kernel
sums. This condition is related to the strict diagonal dominance used by
\citet{pmlr-v84-smith18a} in their supplementary discussion of private
training inputs. There, diagonal dominance is used to bound the infinity norm
of the inverse regularized Gram matrix. Here, we combine the same matrix
control with the one-point posterior-update representation to bound the sensitivity $\Delta_n(r)$ of a posterior draw.

Throughout this subsection, assume that $k(\bx,\bx)=1$ and that
the responses satisfy $|y|\le M_Y$.

\begin{proposition}[Sensitivity under kernel-sum bounds]
\label{prop:Delta_kernel_sum_bounds}
Suppose that there exist constants $R_0,R_1<\infty$, independent of $n$,
such that every admissible common-core design
$X_-=(\bx_1,\ldots,\bx_{n-1})$ and every admissible replacement location $\bx$ satisfy
\[
\max_{1\le i\le n-1}
\sum_{j\ne i}|k(\bx_i,\bx_j)|
\le R_0,
\qquad
\sum_{i=1}^{n-1}|k(\bx,\bx_i)|
\le R_1.
\]
If $\gamma:=1+r^2-R_0>0$, then
\[
\Delta_n(r)
\le
\frac{M_Y}{r}
\left(
1+\frac{R_1}{\gamma}
\right).
\]
In particular, $\Delta_n(r)=\mathcal O(1)$ for fixed $r$, $R_0$, and
$R_1$.
\end{proposition}

\begin{proof}
Fix neighbouring datasets
$D=D_-\cup\{(\bx,y)\}$ and
$D'=D_-\cup\{(\bx',y')\}$, and let
$A:=K(X_-,X_-)+r^2I$. Since $A_{ii}=1+r^2$ and its off-diagonal row
sums are bounded by $R_0$, the matrix $A$ is strictly diagonally
dominant with margin at least $\gamma$. Hence
$\|A^{-1}\|_\infty\le\gamma^{-1}$, see Supplementary Material in \citet{pmlr-v84-smith18a}. Since $A$ is symmetric, we also have
$\|A^{-1}\|_1\le\gamma^{-1}$. Then, for every admissible replacement location $\bx$,
\[
\begin{aligned}
|\mu_{D_-}(\bx)|
&=
|k_{X_-}(\bx)^\top A^{-1}y_-|
\le
\|k_{X_-}(\bx)\|_1
\|A^{-1}\|_1
\|y_-\|_\infty 
\le
\frac{M_YR_1}{\gamma}.
\end{aligned}
\]
It follows that
$|y-\mu_{D_-}(\bx)|\le M_Y(1+R_1/\gamma)$, and the same bound holds
for $(\bx',y')$. Now let $v(\bx):=k_{D_-}(\bx,\bx)$. The one-point posterior update \eqref{eq:update_formula} gives
\[
\begin{aligned} \|\mu_D-\mu_{D'}\|_{\mathcal H_-} &\le |y-\mu_{D_-}(\bx)| \frac{\sqrt{v(\bx)}}{r^2+v(\bx)} + |y'-\mu_{D_-}(\bx')| \frac{\sqrt{v(\bx')}}{r^2+v(\bx')} 
\le \frac{M_Y}{r} \left( 1+\frac{R_1}{\gamma} \right)
\end{aligned}
\]
where we have used the fact that $\|k_{D_-}(\cdot,\bx)\|_{\mathcal H_{-}}=\sqrt{v(\bx)}$ and
$\sup_{v\ge0}\sqrt v/(r^2+v)=1/(2r)$. Taking the supremum over neighbouring datasets proves the result.
\end{proof}

The conditions of Proposition~\ref{prop:Delta_kernel_sum_bounds} can be
verified geometrically for standard decaying kernels. Suppose that
$\Omega_X\subseteq\mathbb R^d$ and that every admissible dataset is
$q$-separated, meaning that $\|\bx_i-\bx_j\|\ge q$ for all distinct
locations in the dataset. Let
$k(\bx,\bx')=\varphi(\|\bx-\bx'\|)$, where $\varphi(0)=1$ and
$\varphi$ is nonnegative and nonincreasing. Define
$N_{d,m}:=(2m+3)^d-(2m-1)^d$ and
\[
R_{d}(q)
:=
\sum_{m=1}^{\infty}
N_{d,m}\varphi(mq).
\]

\begin{corollary}[Separated data with decaying kernels]
\label{cor:Delta_separated_designs}
Suppose that every admissible dataset is $q$-separated and that
$R_{d}(q)<1+r^2$. Then
\[
\Delta_n(r)
\le
\frac{M_Y}{r}
\left(
1+
\frac{R_{d}(q)}
{1+r^2-R_{d}(q)}
\right).
\]
\end{corollary}

\begin{proof}
Fix a location $\bx$ and partition the remaining covariate locations into
the annuli
\[
mq\le\|\bx_i-\bx\|<(m+1)q,
\qquad m\ge1.
\]
The balls of radius $q/2$ centred at the points in the $m$th annulus
are disjoint and are contained between radii $(m-\frac12)q$ and
$(m+\frac32)q$ around $\bx$. Comparing volumes shows that the number
of points in this annulus is at most
\[
\frac{(m+\frac32)^d-(m-\frac12)^d}{(1/2)^d}
=
N_{d,m}.
\]
Let $I_m$ be the set of indices of points in the $m$-th annulus. Since $\varphi$ is nonincreasing,
\[
\sum_i k(\bx,\bx_i) = \sum_i \phi(\|\bx-\bx_i\|)=\sum_{m=1}^\infty \sum_{i\in I_m}\phi(\|\bx-\bx_i\|)
\le
\sum_{m=1}^{\infty}
|I_m|\varphi(mq)\le
\sum_{m=1}^{\infty}
N_{d,m}\varphi(mq)
=
R_{d}(q).
\]
The same argument bounds every off-diagonal row sum of the Gram matrix.
Thus, Proposition~\ref{prop:Delta_kernel_sum_bounds} applies with
$R_0=R_1=R_{d}(q)$.
\end{proof}

For $d=2$, one has $N_{2,m}=16m+8$. For the exponential kernel
$\varphi(s)=\exp(-s/\ell)$ and writing $\zeta:=\exp(-q/\ell)$ gives
\[
R_{2}^{(\exp)}(q)
\le
\sum_{m=1}^{\infty}(16m+8)\zeta^m
=
\frac{16\zeta}{(1-\zeta)^2}
+
\frac{8\zeta}{1-\zeta}.
\]
For the squared-exponential kernel
$\varphi(s)=\exp(-s^2/(2\ell^2))$ writing
$\zeta:=\exp(-q^2/(2\ell^2))$ and using $\zeta^{m^2}\le \zeta^m$ gives the same
sufficient upper bound
\[
R_{2}^{(\mathrm{SE})}(q)
\le
\frac{16\zeta}{(1-\zeta)^2}
+
\frac{8\zeta}{1-\zeta}.
\]
The series $R_{d}(q)$ is also finite for every Mat\'ern kernel
because its radial profile decays exponentially at infinity. Thus, for
exponential, squared-exponential and Mat\'ern kernels sufficiently
large $q/\ell$ or sufficiently large $r$ ensures
$R_{d}(q)<1+r^2$.

The separation condition restricts the admissible data universe by requiring all training datapoints to be at least distance $q$ apart. Besides excluding duplicate or near-duplicate observations and maintaining dataset diversity, such a condition may be imposed specifically to control privacy loss, since it bounds the aggregate kernel similarity and hence the sensitivity. For example when working with images one might enforce a minimum distance between different images. On an unbounded domain this condition permits arbitrarily large datasets. However, on a fixed
compact domain only finitely many $q$-separated points can fit (in moderate- or high-dimensional spaces this admissible dataset size may be very large), so
fixed-domain $n$-asymptotics require $q=q_n$ to decrease with $n$ and the bound remains uniform
when $q_n/\ell_n$ stays bounded away from zero. Note that Proposition~\ref{prop:Delta_kernel_sum_bounds} also applies whenever the
required kernel-sum bounds can be established without an explicit separation condition.

\subsection{A lower bound for covariance leakage}
\label{app:covariance_lower_bound}

The covariance contribution in Theorem~\ref{thm:main_rdp} is
$\mathcal O(1)$ in the dataset size. We next show that in general this cannot be improved. 

\begin{proposition}[Non-vanishing covariance leakage]
\label{prop:covariance_leakage_lower_bound}
Suppose that $k(x,x)=1$ for every $x\in\Omega_X$. Define $\kappa:=\inf_{x,x'\in\Omega_X} k(x,x')$ and assume $ 0\leq \kappa<1$. Define
\[
\Lambda_n(r,\kappa)
:=
\frac{
\bigl[r^2+n(1-\kappa^2)\bigr]
\bigl[r^4+nr^2+(n-1)(1-\kappa^2)\bigr]
}{
r^2(r^2+n)
\bigl[r^2+(n-1)(1-\kappa^2)\bigr]
}.
\]
Then, for every $n\geq1$ and every $\alpha>1$ satisfying $\alpha-(\alpha-1)\Lambda_n(r,\kappa)>0$ one has
\[
\sup_{D\sim D'}\sup_{|X_T|<\infty}
D_\alpha\!\left(
\Pi_D(X_T)\middle\|\Pi_{D'}(X_T)
\right)
\geq
\phi_\alpha\!\left(
1+\frac{1-\kappa^2}{r^2}
\right)
>0, \quad
\phi_\alpha(\lambda)
:=
-\frac12\log\lambda
-\frac{1}{2(\alpha-1)}
\log\!\left(\alpha-(\alpha-1)\lambda\right).
\]
\end{proposition}

\begin{proof}
Suppose that the infimum defining $\kappa$ is attained at points
$a,b\in\Omega_X$, so that $k(a,b)=\kappa$. Let the common core $D_-$
consist of $n-1$ copies of $(a,0)$ and define the neighbouring datasets
\[
D=D_-\cup\{(a,0)\},
\qquad
D'=D_-\cup\{(b,0)\}.
\]
Because all responses are zero, $\mu_D=\mu_{D'}=0$. Evaluate the released path at the single test point $X_T=\{b\}$. After
conditioning on the common core, the posterior variance at $b$ is
\[
k_{D_-}(b,b)
=
1-\kappa^2\,
\mathbf 1^\top
\left(\mathbf 1\mathbf 1^\top+r^2I\right)^{-1}
\mathbf 1
=
1-\frac{(n-1)\kappa^2}{n-1+r^2}
=:
v_{-,n}.
\]
Adding another observation at $a$ gives
\[
k_D(b,b)=1-\frac{n\kappa^2}{n+r^2}=:v_{+,n},
\]
whereas adding an observation at $b$ gives, by the one-point posterior
update \eqref{eq:update_formula},
\[
k_{D'}(b,b)
=
v_{-,n}-\frac{v_{-,n}^2}{v_{-,n}+r^2}
=
\frac{r^2v_{-,n}}{r^2+v_{-,n}}.
\]
Thus, for $X_T=\{b\}$ we have
\[
\Pi_D(\{b\})
=
\mathcal N(0,\sigma^2v_{+,n}),
\qquad
\Pi_{D'}(\{b\})
=
\mathcal N\!\left(
0,
\sigma^2\frac{r^2v_{-,n}}{r^2+v_{-,n}}
\right).
\]
The ratio of these variances is $\Lambda_n(r,\kappa)$, and the scale
$\sigma^2$ cancels from their R\'enyi divergence. The scalar Gaussian R\'enyi divergence
formula consequently gives
\[
D_\alpha\!\left(
\Pi_D(\{b\})\middle\|\Pi_{D'}(\{b\})
\right)
=
\phi_\alpha\!\left(\Lambda_n(r,\kappa)\right).
\]

Moreover, direct simplification gives
\[
\Lambda_n(r,\kappa)
-\left(1+\frac{q}{r^2}\right)
=
\frac{q(n-1)\kappa^2}
{(r^2+n)\left[r^2+(n-1)q\right]}
\geq 0.
\]
Since $\phi_\alpha'(\lambda)>0$ throughout the finite-divergence domain, the assumed condition implies
\[
\phi_\alpha\!\left(\Lambda_n(r,\kappa)\right)
\geq
\phi_\alpha\!\left(1+\frac{q}{r^2}\right)>0,
\]
which proves the claim. If the infimum defining $\kappa$ is not attained, the same conclusion follows by taking a sequence of pairs $(a_m,b_m)$ such that $k(a_m,b_m)\to\kappa$ and using continuity.
\end{proof}

\begin{remark}[Tightness in the strongly regularised regime]
Let $q:=1-\kappa^2$. As $r\to\infty$,
\[
\phi_\alpha\!\left(1+\frac{q}{r^2}\right)
=
\frac{\alpha q^2}{4r^4}
+\mathcal O(r^{-6}).
\]
The covariance term in Theorem~\ref{thm:main_rdp} has the same
$r^{-4}$ scaling. Thus the covariance-channel upper bound has the
correct dependence on both $n$ and $r$ up to constants in this regime.
\end{remark}

\subsection{GP-Thompson sampling: public covariates and adaptive RDP composition}
\label{app:public_covariates}
Thompson sampling is a classical Bayesian method for stochastic multi-armed bandit problems \citep{pmlr-v23-agrawal12}. GP-Thomson sampling (GP-TS, \citet{pmlr-v70-chowdhury17a}) extends this principle to continuous action spaces (the set of arms is continuous) by drawing a function from the current GP posterior and selecting its maximizer. In the finite-armed setting ordinary Thompson sampling has recently been shown to be intrinsically differentially private: posterior sampling itself supplies the required randomization without additional privacy noise or modification of the decision rule \citep{pmlr-v238-ou24a}. Here, we establish the corresponding intrinsic privacy guarantees for GP-TS including release of the full function-valued posterior-draw and action transcript. To our knowledge, this is the first result of this kind.

We start with a brief review of GP-TS whose full details can be found in \citet{pmlr-v70-chowdhury17a}. Consider the task of maximising an unknown function $f_*$. The Thompson sampling algorithm (TS) chooses at each round $t$ a {\bf{public}} action $\bx_t$ and subsequently observes the {\bf{private}} noisy value of the function $y_t=f_*(\bx_t)+\xi_t$. The action $\bx_t$ is selected depending on all the previous actions $\bx_1,\dots,\bx_{t-1}$ and the observed responses. We denote this history by $H_{t-1}:=\{(\bx_1,y_1)\dots,(\bx_{t-1},y_{t-1})\}$. We assume that $f_*$ has bounded RKHS norm $\|f_*\|_{\mathcal{H}_k}\leq B$ and that $\xi_i\stackrel{\text{i.i.d.}}{\sim} P_\xi$ are supported on a bounded interval $[-M_\xi,M_\xi]$. We assume bounded-diagonal kernel $k(\bx,\bx)\le 1$ and define $M_Y:=B + M_\xi$.

The goal of GP-TS is to minimise regret which is the loss incurred due to not knowing $f_*$’s maximiser. Let $\bx_*:=\arg \max_{\bx\in\Omega_X}f_*(\bx)$ be a maximiser of $f_*$. The instantaneous regret is defined as $f(\bx_*)-f(\bx_t)$ and total regret is $R_T:=\sum_{t=1}^T (f(\bx_*)-f(\bx_t))$. GP-TS proceeds as follows \citep{pmlr-v70-chowdhury17a}. Let $\mu_{t-1}$ and $k_{t-1}$ be the GP posterior mean and covariance obtained by conditioning the GP prior on the training set $H_{t-1}$. At round $t$ the algorithm draws $f_t\sim \mathcal{GP}(\mu_{t-1},\sigma_t^2 k_{t-1})$ and selects $\bx_t=\arg \max_{\bx\in\Omega_X}f_t(\bx)$. The covariance scale $\sigma_t$ depends on the round $t$ and this schedule is optimised to minimise the regret, see \citet{pmlr-v70-chowdhury17a} for the exact expression.

Thus, the privacy model relevant to TS is the response-only privacy model in which the covariates
$X=(\bx_1,\ldots,\bx_n)$ are public and only the response vector $y$ is
private. Neighbouring datasets are thus $D=(X,y)$ and
$D'=(X,y')$, where $y$ and $y'$ differ in one coordinate. Since the
posterior covariance depends only on $X$, one has $k_D=k_{D'}$, and
the covariance-leakage term in Theorem~\ref{thm:main_rdp} vanishes. Let us next derive RDP bounds for this model.

Let $K:=K(X,X)$ and $A:=K+r^2I$. For $j\in\{1,\ldots,t\}$, define the
kernel ridge leverage score \citep{elalaoui2015fast}
\[
h_j(X)
:=
\bigl[K(K+r^2I)^{-1}\bigr]_{jj}
=
1-r^2[A^{-1}]_{jj}.
\]
Since the eigenvalues of
$K(K+r^2I)^{-1}$ belong to $[0,1]$, one has
\[
0\le h_j(X)\le 1.
\]

\begin{proposition}[Posterior sampling with public covariates]
\label{prop:public_X_rdp}
Suppose that $D=(X,y)$ and $D'=(X,y')$ with public $X$. Then, $D\sim D'$ means that $y$ and $y'$ differ only in one response data point $j$. If $|y|\le M_Y$, then the release of one GP posterior sample path is $(\alpha,\alpha\rho_j)$-RDP for
every $\alpha>1$, where 
\[
\rho_j\le \frac{2M_Y^2}{\sigma^2r^2}h_j(X).
\]
\end{proposition}

\begin{proof}
Let $A:=K+r^2I$ and 
$
\bk_X(\bx)
:=
\bigl(k(\bx,\bx_1),\ldots,k(\bx,\bx_t)\bigr)^\top
$
and write
$
g:=\mu_{D'}-\mu_D.
$
Since $D$ and $D'$ have the same covariates, their posterior covariance
kernels coincide and we denote this common kernel by $k_D$. The posterior-mean
formula gives
\[
\begin{aligned}
g(\bx)
&=
\bk_X(\bx)^\top A^{-1}(y'-y) 
=
(y_j'-y_j)\,\bk_X(\bx)^\top A^{-1}e_j.
\end{aligned}
\]

Fix any finite test set $X_T$, and define $\delta_T:=g(X_T)$, $K_{D,T}:=k_D(X_T,X_T)$.
The corresponding finite-dimensional posterior distributions have the same
covariance matrix $\sigma^2K_{D,T}$. Hence the common-covariance Gaussian
R\'enyi-divergence formula gives
\[
D_\alpha\!\left(
\Pi_D(X_T)\,\middle\|\,\Pi_{D'}(X_T)
\right)
=
\frac{\alpha}{2\sigma^2}
\delta_T^\top K_{D,T}^{\dagger}\delta_T.
\]
By the RKHS inequality
$
\delta_T^\top K_{D,T}^{\dagger}\delta_T
\le
\left\| g\right\|_{\mathcal H_{k_D}}^2.
$
It remains to calculate this RKHS norm. The posterior covariance kernel is
$
k_D(\bx,\bx')
=
k(\bx,\bx')
-
\bk_X(\bx)^\top A^{-1}\bk_X(\bx').
$
Since $\bk_X(\bx_j)=Ke_j$, we have
\[
\begin{aligned}
k_D(\bx,\bx_j)
&=
\bk_X(\bx)^\top e_j
-
\bk_X(\bx)^\top A^{-1}Ke_j 
=
\bk_X(\bx)^\top
\bigl(I-A^{-1}K\bigr)e_j 
=
r^2\bk_X(\bx)^\top A^{-1}e_j,
\end{aligned}
\]
where the last equality follows from
$
I-A^{-1}K
=
A^{-1}(A-K)
=
r^2A^{-1}.
$
Consequently,
$
g
=
\frac{y_j'-y_j}{r^2}\,k_D(\cdot,\bx_j).
$
By the reproducing property,
\[
\begin{aligned}
\lVert g\rVert_{\mathcal H_{k_D}}^2
&=
\frac{(y_j'-y_j)^2}{r^4}
\left\lVert
k_D(\cdot,\bx_j)
\right\rVert_{\mathcal H_{k_D}}^2 
=
\frac{(y_j'-y_j)^2}{r^4}
k_D(\bx_j,\bx_j).
\end{aligned}
\]
Moreover,
\[
\begin{aligned}
k_D(\bx_j,\bx_j)
&=
e_j^\top
\bigl(K-KA^{-1}K\bigr)e_j 
=
r^2e_j^\top KA^{-1}e_j 
=
r^2h_j(X).
\end{aligned}
\]
Hence,
\[
\lVert g\rVert_{\mathcal H_{k_D}}^2
=
\frac{(y_j'-y_j)^2}{r^2}h_j(X).
\]
Combining all the above results, uniformly over all finite test
sets $X_T$,
\[
D_\alpha\!\left(
\Pi_D(X_T)\,\middle\|\,\Pi_{D'}(X_T)
\right)
\le
\alpha
\frac{(y_j'-y_j)^2}{2\sigma^2r^2}h_j(X)
=
\alpha\rho_j.
\]
The uniform finite-dimensional bound gives the stated sample-path RDP
guarantee. Finally if $|y_i|\le M_Y$ then
$|y_j'-y_j|\le2M_Y$ which proves the remaining bounds.
\end{proof}

The influence of an existing response cannot increase when further
design points are added.

\begin{lemma}[Monotonicity of leverage score]
\label{lem:influence_monotonicity}
Let $X=(\bx_1,\dots,\bx_t)$ and $X_-=(\bx_1,\dots,\bx_{t-1})$. Then, for every $j\le t-1$, $h_j(X)\le h_j(X_-)$. In particular,
\[
h_j(X)
\le
\frac{
k(\bx_j,\bx_j)
}{
k(\bx_j,\bx_j)+r^2
}.
\]
For a normalized kernel satisfying $k(\bx,\bx)=1$ this becomes
$
h_j(X)
\le
1/(1+r^2)
$. If $k(\bx,\bx)\leq 1$ then $
h_j(X)
\le\min\{1,
1/r^2\}
$.
\end{lemma}
\begin{proof}
Let
$
A_-:=K(X_-,X_-)+r^2I
$ and $
a:=k(X_-,\bx_t)
$
and define the positive Schur complement
$
\gamma
:=
k(\bx_t,\bx_t)+r^2-a^\top A_-^{-1}a
>0.
$
The block-inverse formula gives
\[
\left[
\bigl(K(X,X)+r^2I\bigr)^{-1}
\right]_{1:t-1,1:t-1}
=
A_-^{-1}
+
\frac{A_-^{-1}aa^\top A_-^{-1}}{\gamma}.
\]
Hence, for every $j\le t-1$,
\[
h_j(X)
=
h_j(X_-)
-
\frac{
r^2\bigl(e_j^\top A_-^{-1}a\bigr)^2
}{\gamma}
\le
h_j(X_-).
\]
The final bound follows by starting from the single-point dataset
$(\bx_j)$ for which
\[
h_j\bigl((\bx_j)\bigr)
=
\frac{k(\bx_j,\bx_j)}
{k(\bx_j,\bx_j)+r^2},
\]
and adding the remaining data points one at a time.
\end{proof}

Because the bound in Proposition~\ref{prop:public_X_rdp} is linear in
$\alpha$ and valid for every $\alpha>1$, repeated releases admit a
simple analytic conversion to approximate DP.

\begin{corollary}[$T$ rounds of GP-Thompson sampling]
\label{cor:T_round_TS}
$T$ rounds of GP-Thompson sampling together satisfy $(\alpha,\alpha \rho_T)$-RDP with 
\[
\rho_T \leq \frac{2M_Y^2}{r^2}\sup_{\bx\in\Omega_X} \frac{
k(\bx,\bx)
}{
k(\bx,\bx)+r^2
}\sum_{t=2}^T \frac{1}{\sigma_t^2}\leq (T-1)\frac{2M_Y^2}{\bar \sigma^2 r^2}\sup_{\bx\in\Omega_X} \frac{
k(\bx,\bx)
}{
k(\bx,\bx)+r^2
},\qquad \bar \sigma:=\min_{t=2,\dots,T} \sigma_t.
\]
Consequently, $T$ rounds of GP-Thompson sampling are 
\[
\left(
\rho_T+2\sqrt{\rho_T\log(1/\delta)},
\delta
\right)\text{-DP}.
\]
\end{corollary}

\begin{proof}
By Proposition~\ref{prop:public_X_rdp} and Lemma~\ref{lem:influence_monotonicity} one round of GP-TS is $(\alpha,\alpha \rho)$-RDP. The composition guarantees of RDP show that the adaptive composition of $T$ mechanisms that each satisfy $(\alpha,\gamma(\alpha))$-RDP is $(\alpha,T\gamma(\alpha))$-RDP \citep[see][Proposition 1]{rdp}. Thus, $T$ rounds of GP-TS is $(\alpha,\alpha T\rho)$-RDP. The standard conversion
to approximate DP gives $\varepsilon(\alpha)
=
\alpha \rho_T
+
\log(1/\delta)/(\alpha-1)
$.
The minimum over $\alpha>1$ is attained at
$
\alpha_*
=
1+
\sqrt{
\frac{\log(1/\delta)}{\rho_T}
}
$ which gives the claimed bound.
\end{proof}

This result mirrors the DP analysis of finite-armed Thompson sampling
in \citet[Appendix~C.2]{pmlr-v238-ou24a}. There, one round is
bounded by $(\alpha,\alpha/4)$-RDP, and adaptive composition over $T$
rounds gives $(\alpha,\alpha T/4)$-RDP and both guarantees have identical $T$-scaling. 

\paragraph{Sublinear growth of the RDP bound for Mat\'{e}rn kernels} Interestingly, under a suitable variance schedule $\sigma_t$ one can obtain sublinear growth of the DP/RDP bounds with $T$. Such a schedule can be obtained as a lower bound for the variance schedule of \citet{pmlr-v70-chowdhury17a} which in our notation reads $\sigma_t = B + M_\xi \sqrt{2(\gamma_{t-1}+1+\log(2/\delta_{conf}))}$ where $\gamma_t$ is the so-called maximum information gain and $\delta_{conf}\in ]0,1[$ is a free confidence parameter used by the algorithm. The key is to use known lower bounds for $\gamma_t$. Namely, \citet[Section~1.1]{pmlr-v130-vakili21a} explain that for Mat\'{e}rn-$\nu$ kernels on compact $d$-dimensional domains there exists $c>0$ such that  $\gamma_t\geq c\, t^{d/(2\nu+d)}$. Consequently, we have
\[
\sigma_t^2\ge 2M_\xi^2 \gamma_{t-1} \ge c\,M_\xi^2\,(t-1)^{d/(2\nu+d)}.
\]
Plugging this into the bound from Corollary~\ref{cor:T_round_TS} and using $k(\bx,\bx)=1$ we get 
\[
\rho_T \le \frac{2M_Y^2}{r^2(1+r^2)}
\sum_{t=2}^{T}\frac{1}{\sigma_t^2}\le \frac{2M_Y^2}{c\,M_\xi^2\,r^2(1+r^2)}
\sum_{t=1}^{T-1}t^{-d/(2\nu+d)}= \mathcal O\!\left( T^{2\nu/(2\nu+d)} \right).
\]

\subsection{Number of releasable posterior paths}
\label{app:composition}
To illustrate the effect of composition, we define
$L_{\max}(r,\sigma)$ as the largest number of independent posterior paths
whose composed guarantee satisfies $\varepsilon_L\le \varepsilon_{budget}$ at
$\delta=10^{-3}$. Figure~\ref{fig:lmax_contours} shows that regularisation
and prior scale can permit tens to thousands of exact posterior draws while using specialised bounds such as
the RKHS-response assumption further increases $L_{\max}$ relative to the
generic bounded-response model.

\begin{figure}[t]
    \centering
    \includegraphics[width=\linewidth]{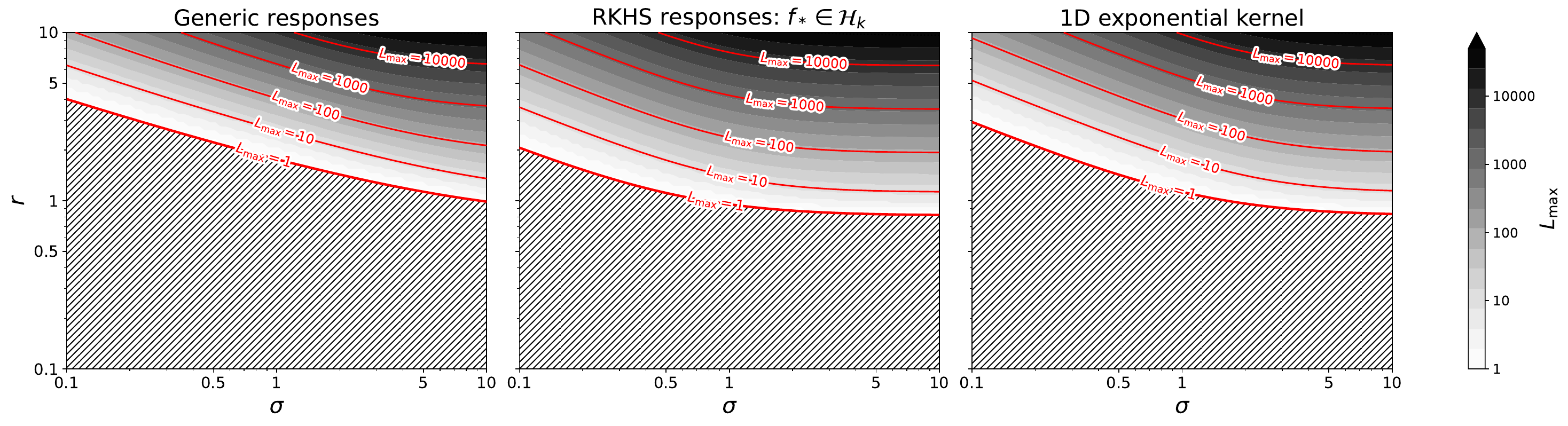}
    \includegraphics[width=\linewidth]{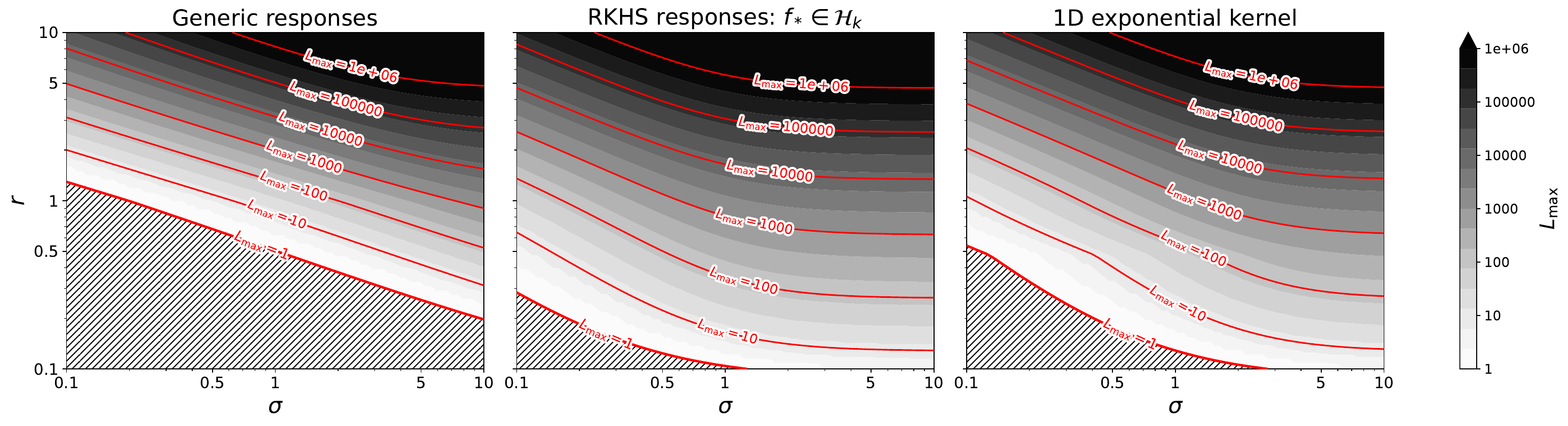}
    \caption{
    Maximum number $L_{\max}$ of independent posterior sample paths that can
    be released under the given privacy budget
    $(\varepsilon,\delta)$, as a function of the regularisation
    parameter $r$ and prior scale $\sigma$. The left column uses the generic
    bounded-response assumption $|y_i|\le M_Y=1$ (Lemma~\ref{lemma:boundedResp_Delta_bound}) and the middle column uses
    the RKHS assumption $y_i = f_*(\bx_i)$ with $\|f_*\|_{\mathcal H_k}\le1$ (Lemma~\ref{lemma:Delta_rkhs_response_bound}). The right column uses one-dimensional exponential kernel with bounded responses $|y_i|\le M_Y=1$ (Corollary~\ref{coro:exp_1d}). Hatched regions do not permit
    even one path under the given privacy budget. {\textbf{Top row:}} Privacy budget $\varepsilon_{budget}=10$. {\textbf{Bottom row:}} Privacy budget $\varepsilon_{budget}=10^3$ which is more illustrative for protecting against training data reconstruction, see the discussion in Section \ref{sec:discussion}. Every panel has
    $n=100$, $\delta=10^{-3}$ and $\kappa=\sqrt{0.1}$. 
    }
    \label{fig:lmax_contours}
\end{figure}

\section{FURTHER DETAILS OF THE MEMBERSHIP ATTACK}\label{app:mia}

We evaluate membership leakage using a LiRA-style shadow-model attack. The attack targets a fixed point \(x_0\) and distinguishes between two hypotheses:
\[
    H_{\rm in}: x_0 \in D,
    \qquad
    H_{\rm out}: x_0 \notin D .
\]
For each hypothesis, the adversary generates shadow datasets from the same data-generating distribution used in the experiment (uniform on $\Omega_X=[0,1]$). In the step-function experiment, responses are generated from
\[
    f_{step}(x)=
    \begin{cases}
    -1, & x<1/2,\\
    1, & x\geq 1/2,
    \end{cases}
\]
with bounded additive label noise $y_i=(1-M_\xi)f_{step}(x_i)+\xi_i$, $\xi_i\sim {\rm Unif}[-M_\xi,M_\xi]$. Under $H_{\rm in}$, the target point $x_0$ is inserted into the dataset and the remaining points are sampled from the uniform distribution on $[0,1]$. Under $H_{\rm out}$, all $n$ points are sampled from the uniform distribution on $[0,1]$. For a shadow dataset $D=(X,y)$ the posterior mean and normalized posterior variance at $x_0$ are
\[
    \mu_D(x_0)
    =
    k(x_0,X)\bigl(K(X,X)+r^2 I\bigr)^{-1}y,
\]
and
\[
    v_D(x_0)
    =
    1-
    k(x_0,X)\bigl(K(X,X)+r^2 I\bigr)^{-1}k(X,x_0),
\]
where $k(x,x')=\exp(-|x-x'|/\ell)$ is the exponential kernel function. The released object consists of \(L\) posterior sample-path values at the target point,
\begin{equation}\label{eq:posterior_draw}
    f_D^{(\ell)}(x_0)
    =
    \mu_D(x_0)
    +
    \sigma \sqrt{v_D(x_0)} Z_l,
    \qquad
    Z_l\sim N(0,1),
    \qquad
    l=1,\ldots,L.
\end{equation}
The attack uses two scalar statistics based on $L$ released values:
\[
    \widehat f_D(x_0)
    =
    \frac{1}{L}\sum_{\ell=1}^L f_D^{(\ell)}(x_0),
\qquad 
    \widehat v_D(x_0)
    =
    \frac{1}{L\sigma^2}
    \sum_{\ell=1}^L
    \left(
        f_D^{(\ell)}(x_0)-\widehat f_D(x_0)
    \right)^2 .
\]
Thus $\widehat f_D(x_0)$ captures the posterior mean signal available from the released paths, while $\widehat v_D(x_0)$ captures the posterior covariance signal. For $L=1$, the empirical variance statistic is undefined and the attack uses only $\widehat f_D(x_0)$.

The LiRA score is obtained by fitting separate density models for the in and out shadow distributions. Let $\rho_{\rm in}^{(f)},\ \rho_{\rm out}^{(f)}$ denote the fitted densities of $\widehat f_D(x_0)$ under $H_{\rm in}$ and $H_{\rm out}$ respectively. The sample-mean statistic is modelled by a two-component latent Gaussian mixture. Specifically, the model introduces a latent Gaussian mixture variable
\[
    \phi \sim
    \sum_{j=1}^2 \omega_j N(m_j,\tau_j^2),\quad N(m_j,\tau_j^2):\ \phi\mapsto \rho_{Gauss}\left(\phi;m_j,\tau_j^2\right):=\frac{1}{\tau_j\sqrt{2\pi}}
\exp\left[-(\phi-m_j)^2/(2\tau_j^2)\right]
\]
and models the observed statistic as
\[
    \widehat f_D(x_0)=u+\eta, \qquad u=\tanh(\phi/2),
    \qquad
    \eta\sim N(0,s^2).
\]
This latent model is motivated by Equation~\ref{eq:posterior_draw} where $\mu_D(x_0)$ takes values in $[-1,1]$ (by Lemma~\ref{lemma:Delta_exp_1D_mu_bounded}) and thus is modelled by $\tanh(\phi/2)$ whereas the contribution $\sigma \sqrt{v_D(x_0)} Z_l$ is modelled by the random normal variable $\eta$. Consequently, the fitted density is obtained by marginalising over the distribution of the latent variable $\phi$
\[
    \rho^{(f)}(t) = \int_{-\infty}^\infty\rho(t\mid \phi) \rho(\phi)d\phi
    =
    \sum_{j=1}^2
    \omega_j
    \int_{-\infty}^\infty
    \rho_{Gauss}\left(t;\tanh(\phi/2),s^2\right)
    \rho_{Gauss}\left(\phi;m_j,\tau_j^2\right)
    \,d\phi ,
\]
with the parameters fitted separately for the in- and out- shadow distributions by maximum likelihood. The integral is evaluated numerically.

For the nonnegative sample-variance statistic $\widehat v_D(x_0)$ the attack instead fits a two-component Gaussian mixture after the logarithmic transformation i.e.,
\[
    \phi_v
    =
    \log\widehat v_D(x_0),
\qquad
    \phi_v
    \sim
    \sum_{j=1}^2 \tilde\omega_j\, N\left(\tilde m_j,\tilde \tau_j^2\right).
\]

Given a fresh evaluation release, the fitted likelihood-ratio scores are
\[
    S_f
    =
    \log \rho_{\rm in}^{(f)}\!\left(\widehat f_D(x_0)\right)
    -
    \log \rho_{\rm out}^{(f)}\!\left(\widehat f_D(x_0)\right),
\]
and, for \(L>1\),
\[
    S_v
    =
    \log \rho_{\rm in}^{(v)}\!\left(\widehat v_D(x_0)\right)
    -
    \log \rho_{\rm out}^{(v)}\!\left(\widehat v_D(x_0)\right).
\]
The combined attack score is $S = S_f + S_v$. Equivalently, this is a plug-in product approximation to the joint likelihood ratio of
$(\widehat f_D(x_0),\widehat v_D(x_0))$. When $L=1$ the combined score reduces to $S=S_f$. Figure~\ref{fig:lira_histogram} shows that the latent models model the empirical in- and out- histograms well.

Attack performance is evaluated on independent in- and out- evaluation datasets. We report the ROC curve and true-positive rates at fixed false-positive rates. For a score threshold $t$ the Neyman--Pearson test declares membership when
$S\geq t$. This gives false-positive and true-positive rates
\[
    {\rm FPR}(t)=\mathbb P_{\rm out}\{S\geq t\},
    \qquad
    {\rm TPR}(t)=\mathbb P_{\rm in}\{S\geq t\}.
\]
To report ${\rm TPR}@\alpha$ we vary $t$ until the false-positive rate under
the out distribution reaches the target level $\alpha$, and then evaluate the
true-positive rate under the in distribution at the same threshold. Thus, for
a threshold $t_\alpha$ satisfying ${\rm FPR}(t_\alpha)=\alpha$ we report
$
    {\rm TPR}@\alpha
    =
    {\rm TPR}(t_\alpha)
    =
    \mathbb P_{\rm in}\{S\geq t_\alpha\}
$.
This evaluation measures how well an adversary equipped with shadow samples from the two membership hypotheses and knowledge of the release parameters $(\ell,r,\sigma,L)$ can distinguish whether the target point was included in the training dataset. Figure~\ref{fig:mia2} presents results for $\mathrm{TPR}@\mathrm{FPR}=0.01$.
\begin{figure}[t]
  \centering
 \includegraphics[width=.7\columnwidth]{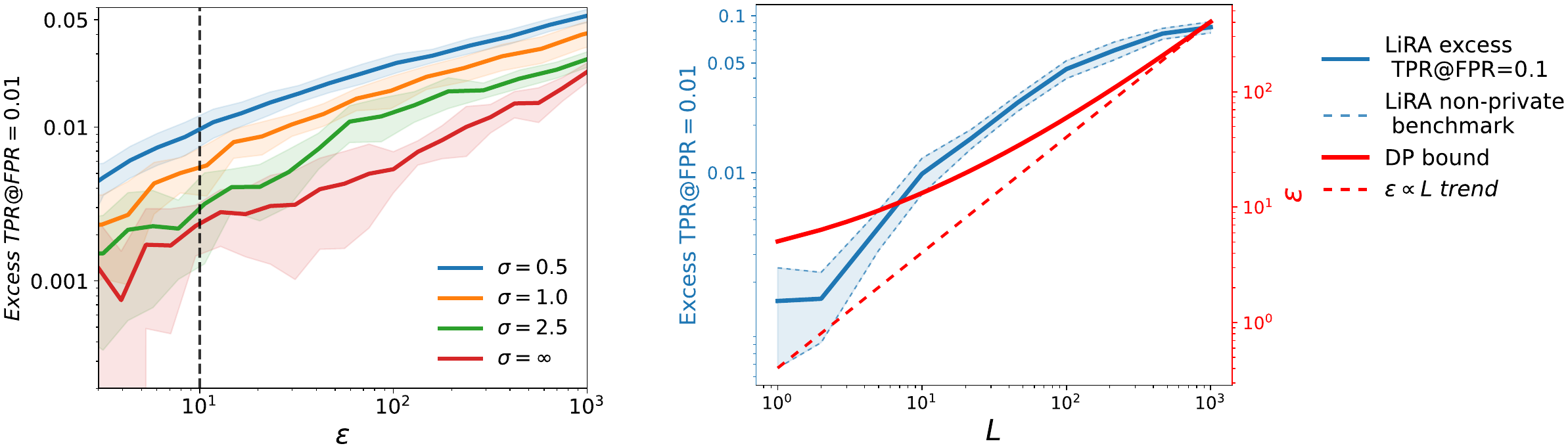}
 \caption{
MIA success and $(\varepsilon,\delta)$-DP bounds ($n=10$, $\delta=0.05$). 
{\bf Left:} Excess $\mathrm{TPR}@\mathrm{FPR}=0.01$ vs. $\varepsilon$ for different $\sigma$ ($L=1$). Error bands show standard deviation over $10$ random seeds. {\bf Right:} Effect of releasing $L$ posterior sample paths. The LiRA attack starts with random-guess effectiveness ($TPR\approx FPR$) at $L=1$ and strengthens with $L$ approaching the non-private benchmark while the $(\varepsilon,\delta)$-DP bound grows sub-linearly with $L$. Parameters: $r=1$, $\sigma=5$, $n=10$, $\delta=0.05$. Error bands show standard deviation over $10$ random seeds.
}
  \label{fig:mia2}
\end{figure}

\begin{figure}[t]
  \centering
 \includegraphics[width=.75\textwidth]{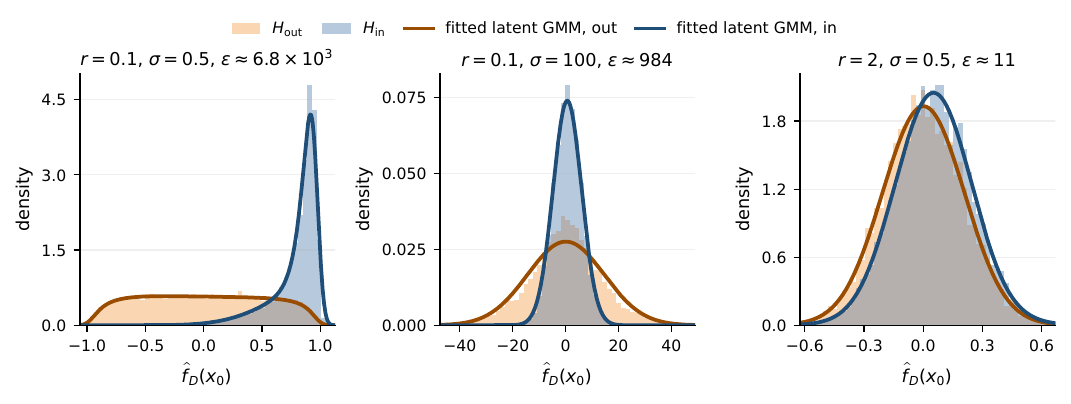}
 \includegraphics[width=.75\textwidth]{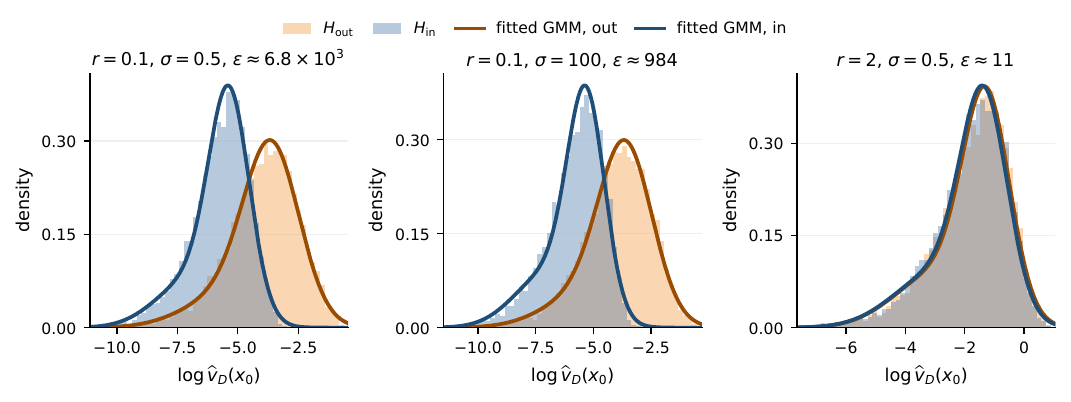}
  \caption{Shadow-distribution histograms for the LiRA-style attack at $x_0=1/2$ showing the sample-path mean statistic $\widehat f_D(x_0)$ in the top row and the log-variance statistic $\log \widehat v_D(x_0)$ in the bottom row. Each column corresponds to a different choice of the effective ridge parameter $r$ and kernel scale $\sigma$ with the corresponding approximate DP value $\varepsilon$ shown in the figure title. The comparison illustrates that privacy depends sharply on effective ridge regularisation. When $r$ is small, the in- and out-shadow distributions remain well separated, and even a very large posterior sampling scale $\sigma$ does not yield a small privacy parameter. In contrast, increasing $r$ substantially reduces the distinguishability of the two hypotheses: in the right column the in/out histograms and fitted densities nearly overlap, corresponding to much stronger membership privacy protection. Dataset size $|D|=10$, $\delta=0.001$, the LiRA attack uses $L=3$ independent posterior draws. The responses are generated as noiseless ($M_\xi=0$).
}
  \label{fig:lira_histogram}
\end{figure}

\section{FURTHER DETAILS OF THE PRIVACY-UTILITY EXPERIMENTS}
\label{app:utility}

\paragraph{Details of the 1D excursion-set utility experiment.}
For each experimental configuration we generate a pool of $N_{\rm pair}=10^3$
independent pairs $\left(f_*^{(j)},D^{(j)}\right)$, $j=1,\dots,N_{\rm pair}$ on $\Omega_X=[0,1]$. For pair $j$ the covariates
$x^{(j)}_1,\ldots,x^{(j)}_n$ are sampled i.i.d. from
$\mathrm{Unif}[0,1]$. The latent target is sampled from a centred GP with
exponential kernel
$\widetilde k(x,x')=\exp(-|x-x'|/\ell_{\rm gen})$ with $\ell_{\rm gen}=1$. The sampled path is rescaled to have
supremum norm $1-M_\xi$ and the observations for pair $j$ are generated as
\[
    y_i^{(j)}=f_*^{(j)}\left(x_i^{(j)}\right)+\xi_i^{(j)},
    \qquad
    \xi_i^{(j)}\overset{\rm i.i.d.}{\sim}
    \mathrm{Unif}[-M_\xi,M_\xi].
\]
More precisely, if the unnormalised prior GP draw is denoted by $\tilde f_*^{(j)}$,
then
$f_*^{(j)}=(1-M_\xi)\tilde f_*^{(j)}/\left\|\widetilde f_*^{(j)}\right\|_\infty$.
Thus the amplitude noise-to-signal ratio is $NSR=M_\xi/(1-M_\xi)$. In all reported experiments we take the excursion threshold $t=0$, so the target excursion set is
\[
    \Omega_t\left(f_*^{(j)}\right)
    =
    \left\{x\in\Omega_X:f_*^{(j)}(x)\ge t\right\}.
\]
Norm $\|\cdot\|_\infty$ is approximated on a fixed uniform evaluation grid
$x^{\rm grid}_1,\ldots,x^{\rm grid}_m$ with $m=800$ and all integrals over $\Omega_X$ are approximated using the trapezoidal rule on the same grid. We write
\[
    s_*(x):=\mathbf 1\left\{f_*\left(x\right)\ge t\right\}.
\]
where $\mathbf 1\{\cdot\}$ is the indicator function.

For a candidate hyperparameter triple $\theta=(\ell,r,\sigma)$, we fit the
GP posterior for each pair $\left(f_*^{(j)},D^{(j)}\right)$ and compute the posterior
excursion probability
\[
p_{D,\theta}(x)
:=
\mathbb P_{f\sim\Pi_{D,\theta}}\!\left[f(x)\ge t\right]
=
\Phi\!\left(
\frac{\mu_D(x)-t}
{\sigma\sqrt{k_D(x,x)}}
\right),
\]
where $\Phi$ is the CDF of the univariate normal distribution. The hyperparameter-search objective is the prior-average integrated BCE
\[
    \widehat{\mathcal L}_{\rm BCE}(\theta)
    =
    \frac1{N_{\rm pair}}
    \sum_{j=1}^{N_{\rm pair}}\mathcal L_{\rm BCE}\left(\Pi_{D^{(j)},\theta},f_*^{(j)}\right)
    , \quad \mathcal L_{\rm BCE}\left(\Pi_{D,\theta},f_*\right):=\int_{\Omega_X} \ell_{\rm BCE}\!\left( s_*(x), p_{D}(x) \right)\,dx.
\]
The unconstrained benchmark hyperparameters are selected by
\[
    \theta_*
    \in
    \arg\min_\theta
    \widehat{\mathcal L}_{\rm BCE}(\theta),
\]
whereas the privacy-feasible hyperparameters are selected by the constrained
search
\[
    \theta^{\rm priv}_L
    \in
    \arg\min_{\theta:\,\varepsilon_L(\theta)<\varepsilon_0}
    \widehat{\mathcal L}_{\rm BCE}(\theta).
\]
Thus the constrained search differs from the unconstrained search only through
the privacy feasibility constraint. Both searches optimise the same BCE
objective. The DP bound $\varepsilon_L(\theta)$ is evaluated using the
exponential-kernel sensitivity formula from Lemma~\ref{lemma:Delta_exp_1D} together with Theorem \ref{thm:main_rdp} and the RDP composition rule \eqref{eq:rdp_composition} numerically optimised over $\alpha$ with $L$ being the number of released posterior sample paths. The hyperparamater search is performed on a coarse grid over
$(\ell,r,\sigma)$ followed by two rounds of grid refinement around the top-$3$ best average-BCE
candidates with the private refinement restricted to feasible candidates
satisfying $\varepsilon_L(\theta)<\varepsilon_0$.

After selecting $\theta_*$ and $\theta^{\rm priv}_L$, we evaluate both the
posterior probability map and the actual randomized release. The deterministic
 estimate associated with a candidate $\theta$ is
\[
\widehat s_\theta(x):=\mathbf 1\left\{p_{D,\theta}(x)\ge C\right\},
\]
where the cutoff $C$ is selected on a separate validation set by maximising
expected IoU where
\[
{\rm IoU}\left(\widehat\Omega_\theta,\Omega_t\left(f_*\right)\right)=\frac{\left|\widehat\Omega_\theta\cap\Omega_t\left(f_*\right)\right|}{
\left|\widehat\Omega_\theta\cup\Omega_t\left(f_*\right)\right|}=\frac{\int_{\Omega_X}dx\, s_*(x)\widehat s_\theta(x)}{\int_{\Omega_X}dx\, \mathbf 1\left\{\widehat s_\theta(x)=1
    \ \text{or}\ s_*(x)=1\right\}}.
\]
The finite-grid approximation of ${\rm IoU}\left(\widehat\Omega_\theta,f_*\right)$ is the reported baseline IoU. It is a deterministic post-processing of the posterior excursion probabilities and is not the randomized posterior-sample release.

The utility of the randomized released excursion set is evaluated by Monte Carlo simulation of the actual $L$-path mechanism where we draw $B$ independent sets of $L$ posterior sample paths. Namely, for each stored pair $\left(f_*^{(j)},D^{(j)}\right)$ and Monte Carlo repetition $b=1,\ldots,B$, we draw independent posterior paths $f^{(1)}_{j,b},\ldots,f^{(L)}_{j,b}
    \overset{\rm i.i.d.}{\sim}
    \Pi_{D^{(j)},\theta}$
on the evaluation grid. In the experiments we take $B=50$. For a fixed collection of posterior sample paths $f^{(1)},\dots, f^{(L)}$ the indicator function of the released (randomised) excursion set is
\[
    \widehat s^{(L)}_{\theta}(x)
    :=
    \mathbf 1\!\left\{
    \frac1L\sum_{\ell=1}^L
    \mathbf 1\left\{f^{(\ell)}(x)\ge t\right\}
    \ge c
    \right\},
\]
where the vote-fraction cutoff $c$ is also selected on a separate validation
set by maximising expected IoU. The released excursion set for this collection of posterior sample paths is denoted by $\widehat\Omega_\theta^{(L)}$.

For fixed pair $(f_*,D)$, the Monte Carlo mean over the $B$ sample-path set realisations
\[
    \widehat {\rm IoU}^{(L)}_{MC}
    =
    \frac1B\sum_{b=1}^B {\rm IoU}\left(\widehat\Omega_{b,\theta}^{(L)},\Omega_t\left(f_*\right)\right)
\]
estimates the expected IoU of the randomized release $\widehat\Omega_\theta^{(L)}$ given
$(f_*,D)$. The reported released-set IoU is obtained by summarising the median and IQR of the values $\widehat {\rm IoU}^{(L)}_{MC}$ over the stored pairs $\left(f_*^{(j)},D^{(j)}\right)$. Similarly, the reported $L$-path IoU SD is the distribution of the MC standard deviations over the stored pairs $\left(f_*^{(j)},D^{(j)}\right)$. Such a standard deviation of the released-set IoU isolates the
randomness due to posterior sampling for a fixed $\left(f_*^{(j)},D^{(j)}\right)$-pair.

The BCE increase reported in the final table is computed pairwise by comparing
the privacy-feasible and unconstrained BCE-selected candidates on the same
$\left(f_*^{(j)},D^{(j)}\right)$-pairs:
\[
    \frac{
    \mathcal L_{\rm BCE}\left(\Pi_{D^{(j)},\theta^{\rm priv}_L},f_*^{(j)}\right)
    }{
    \mathcal L_{\rm BCE}\left(\Pi_{D^{(j)},\theta_*},f_*^{(j)}\right)
    }
    -1.
\]
The tables report median and interquartile range over the stored pairs for the
DP bound, effective dimension
$d_{\rm eff}=\operatorname{tr}\{K(K+r^2I)^{-1}\}$, BCE increase, baseline IoU from the unconstrained hyperparameter search,
released-set IoU, and the conditional standard deviation of the released-set
IoU.

Table~\ref{tab:noise_transition_full} shows that as $\rm NSR$ increases the unconstrained utility optimum becomes naturally more
regularized: both $\varepsilon$ and $d_{\rm eff}$ decrease sharply. Consequently, the additional utility cost of the constrained search decreases as $\rm NSR$ grows. The results complement Table~\ref{tab:noise_transition} from the main text.

\begin{table}[t]
\centering
\setlength{\tabcolsep}{4pt}
\begin{tabular}{lcccc}
\toprule
 $M_\xi$ & $0.1$ & $0.3$ & $0.5$ & $0.6$ \\
\midrule
NSR & 0.11 & 0.43 & 1.00 & 1.50 \\
$\varepsilon$, unconstrained & $2460$ & $270$ & $36$ & $22$ \\
$d_{\rm eff}$, unconstrained & $51.2\,[50.5,52.0]$ & $19.4\,[19.2,19.6]$ & $8.39\,[8.32,8.44]$ & $5.75\,[5.71,5.78]$ \\
$d_{\rm eff}$, $\varepsilon<9$ & $3.09\,[3.07,3.10]$ & $2.26\,[2.25,2.27]$ & $2.18\,[2.17,2.19]$ & $2.16\,[2.15,2.17]$ \\
relative BCE increase & $116\%\,[76\%,197\%]$ & $37.7\%\,[15.0\%,71.9\%]$ & $12.0\%\,[0.5\%,25.8\%]$ & $5.9\%\,[0.0\%,15.4\%]$ \\
IoU, unconstrained & $0.949\,[0.925,0.975]$ & $0.916\,[0.872,0.954]$ & $0.871\,[0.812,0.920]$ & $0.844\,[0.760,0.905]$ \\
$1$-path IoU mean, $\varepsilon<9$ & $0.841\,[0.776,0.888]$ & $0.828\,[0.757,0.876]$ & $0.797\,[0.720,0.855]$ & $0.765\,[0.684,0.831]$ \\
relative IoU  gap & $10.6\%\,[7.3\%,16.5\%]$ & $8.8\%\,[5.5\%,14.2\%]$ & $8.2\%\,[4.6\%,14.7\%]$ & $8.6\%\,[4.8\%,15.9\%]$ \\
\bottomrule
\end{tabular}
\caption{The effect of varying the $\mathrm{NSR}$ set size on the released 1D excursion set utility at fixed training set size $|D|=100$. Median [interquartile range] over $10^3$ draws of $(D,f_*)$-pairs across a range $\mathrm{NSR}$ values. The private candidate is the best average-BCE-selected setting satisfying $\epsilon<9$ and $\delta=10^{-3}$. The $1$-path IoU mean reported for a sample of $50$ posterior draws per each $(D,f_*)$-pair.}
\label{tab:noise_transition_full}
\end{table}

\subsection{Synthetic excursion set experiment in 2D}
\label{app:utility_2D}
We consider a synthetic spatial excursion-set problem on $\Omega_X=[0,1]^2$,
intended to mimic a smooth pollution field with a small number of broad
high-concentration regions. We generate a random pollution field by drawing a small number of anisotropic Gaussian-type
plumes with random centres, widths and amplitudes and forming their
superposition $h(x)$. More specifically, we draw $K$ anisotropic Gaussian plumes and set
\[
    h(x)
    =
    \sum_{j=1}^{K}
    a_j
    \exp\!\left\{
        -\frac12
        (x-x_{c_j})^\top
        \Sigma_j^{-1}
        (x-x_{c_j})
    \right\},
    \qquad
    \Sigma_j
    =
    R_{\phi_j}
    {\rm diag}(w_{j,1}^2,w_{j,2}^2)
    R_{\phi_j}^{\top}.
\]
Here $x_{c_j}\in[0,1]^2$, $a_j>0$, $w_{j,1},w_{j,2}>0$, $\phi_j\in[0,2\pi]$ are sampled randomly and $R_\phi$ denotes the rotation by angle $\phi$. In the presented expetiments we take $K=2$. To obtain high- and low-pollution regions, we apply the sigmoid transformation
\[
    \widetilde g_*(x)
    =
    \frac{1}{1+\exp[-\gamma(h(x)-b_h)]},
\]
where $b_h$ is chosen so that $65\%$ grid values of $h$ are below $b_h$.
The resulting field is then normalised to give $g_*:\Omega_X\to[0,1]$ on $\Omega_X$. Since the
GP prior is centred at zero, we finally use the transformed field
\[
    f_*(x)=2 g_*(x)-1\in[-1,1]
\]
as the latent response function.

Given $n$ uniformly sampled sensor locations $x_i\in\Omega_X$, observations are
generated as
\[
    y_i=(1-M_\xi)f_*(x_i)+\xi_i,
    \qquad
    \xi_i\sim{\rm Unif}[-M_\xi,M_\xi].
\]
The target set is the excursion set of the attenuated signal,
\[
    \Omega_t
    =
    \{x\in\Omega_X:(1-M_\xi)f_*(x)\ge t\},
    \qquad
    t=(1-M_\xi)(2q-1),
\]
where $q\in(0,1)$ is the excursion threshold on the original $[0,1]$ scale of
$g_*$. In the presented experiments we take $q=1/2$ i.e. $t=0$. For each dataset we fit a GP posterior with exponential kernel
$k(x,x')=\exp(-\|x-x'\|/\ell)$, effective ridge parameter $r$, and posterior
scale $\sigma$. As in the 1D experiment the hyperparameters $(\ell,r,\sigma)$ are selected by
minimising the mean binary cross-entropy between the true excursion
labels and the posterior excursion probabilities
\[
    p_D(x)
    =
    \Phi\!\left(
    \frac{\mu_D(x)-t}{\sigma\sqrt{k_D(x,x)}}
    \right).
\]
We perform this selection both without a privacy constraint and under the privacy constraint $\varepsilon<\varepsilon_0$ using a coarse
grid followed by a local refinement around the best feasible hyperparameter candidates. 

For privacy bounds we use Theorem~\ref{thm:main_rdp} together with Proposition~\ref{prop:general_positive_kernel_exact_Vn} which gives $V_n(r)=\bar V_n(r)$ with $\kappa=\exp(-\sqrt{2}/\ell)$ and Lemma~\ref{lemma:boundedResp_Delta_bound} that provides the generic $\mathcal O(\sqrt{n})$-bound for $\Delta_n(r)$. We emphasise that this latter bound is not tailored to the present spatial model and is likely conservative -- it controls the posterior mean sensitivity uniformly over all bounded-response datasets without exploiting the smooth plume structure or the specific GP kernel form. Consequently, the privacy-constrained hyperparameter search may over-regularise the posterior making the observed privacy-utility tradeoff more pronounced than what would be obtained from a sharper problem-specific sensitivity analysis. This is to be compared with the 1D setting where tighter bounds on $\Delta_n(r)$ lead to substantially less pessimistic privacy calibration.

After hyperparameter selection a single probability threshold $C$
is chosen to maximise the mean IoU of the decision sets
$\{x:p_D(x)\ge C\}$ over the sampled set-aside validation data. We also evaluate the actual
one-path private release: for each realisation $(D,f_*)$ posterior sample paths are drawn
on a grid, smoothed by a fixed Gaussian filter and thresholded to form
released excursion sets. This smoothing is data-independent post-processing and
therefore does not affect the privacy guarantee. We report the non-private posterior-probability excursion estimate $\{x:p_D(x)\ge C\}$ and the IoU
of the smoothed one-path private releases against the true excursion set.

Table~\ref{tab:noise_transition_2d} and Figures~\ref{fig:2d_excursion_NSR1}--\ref{fig:2d_excursion_lowNSR} show that the privacy-constrained posterior can still yield useful released excursion sets, even when its posterior-probability calibration is substantially degraded. The effect of the privacy constraint is most visible in the BCE values: the private hyperparameters satisfying $\varepsilon<9$ have much smaller effective dimension than the unconstrained posterior and this leads to a large relative increase in BCE, especially in the low-noise regime. However, the IoU degradation is considerably milder. At low noise the non-private posterior-probability set is almost exact, so the private one-path releases necessarily show a visible loss relative to this very strong baseline. Nevertheless, the released boundaries remain concentrated around the true excursion boundary and are still highly informative. As the noise level increases, the unconstrained excursion set itself becomes less accurate and the additional cost of privacy becomes less pronounced in relative IoU terms. This transition is reflected both quantitatively in the table and qualitatively in the figures: for $\mathrm{NSR}=0.11$ the non-private boundary nearly coincides with the truth, whereas for $\mathrm{NSR}=1$ the non-private estimate is already visibly imperfect, making the smoothed private one-path releases a reasonable approximation despite the DP-induced randomisation.

\begin{table}[ht]
\centering
\setlength{\tabcolsep}{4pt}
\begin{tabular}{lccc}
\toprule
 $M_\xi$ & $0.1$ & $0.3$ & $0.5$ \\
\midrule
NSR & 0.11 & 0.43 & 1.00 \\
$\varepsilon$, unconstrained & $3\cdot 10^6$ & $4\cdot 10^4$ & $10^3$ \\
$d_{\rm eff}$, unconstrained & $211.8\, [210.8,212.8]$ & $96.63\, [96.24,97.00]$ & $39.55\, [39.42,39.67]$  \\
$d_{\rm eff}$, $\varepsilon<9$ & $7.38\, [7.37,7.39]$ & $8.24\, [8.23,8.25]$ & $7.87\, [7.86,7.88$  \\
relative BCE increase & $440.0\%\, [329.2\%,542.6\%]$ & $165.5\%\, [130.1\%,205.4\%]$ & $73.91\%\, [54.78\%,93.28\%]$  \\
IoU, unconstrained & $0.979\, [0.974,0.983]$ & $0.950\, [0.939,0.959]$ & $0.908\, [0.888,0.927]$  \\
$1$-path IoU mean, $\varepsilon<9$ & $0.846\, [0.822,0.868]$ & $0.820\, [0.793,0.847]$ & $0.778\, [0.747,0.815]$  \\
$1$-path IoU SD, $\varepsilon<9$ & $0.029\, [0.024,0.033]$ & $0.035\, [0.030,0.041]$ & $0.043\, [0.037,0.050]$ \\
relative IoU  gap & $13.6\%\, [11.4\%,15.8\%]$ & $13.6\%\, [11.1\%,16.1\%]$ & $13.9\%\, [10.4\%,16.9\%]$  \\
\bottomrule
\end{tabular}
\caption{The effect of varying the $\mathrm{NSR}$ set size on the released 2D excursion set utility at fixed training set size $|D|=400$. Median [interquartile range] over $10^3$ draws of $(D,f_*)$-pairs across a range $\mathrm{NSR}$ values. The private candidate is the best average-BCE-selected setting satisfying $\epsilon<9$. The $1$-path IoU mean, SD reported for a sample of $50$ posterior draws per each $(D,f_*)$-pair.}
\label{tab:noise_transition_2d}
\end{table}

\begin{figure*}[t]
  \centering
 \includegraphics[width=.9\textwidth]{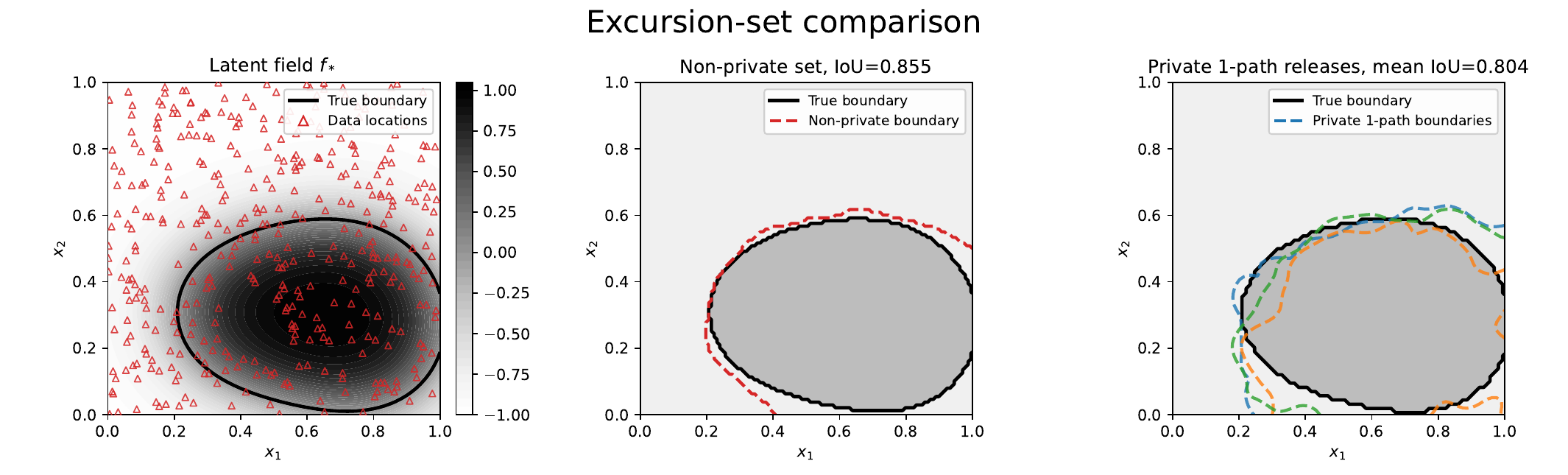}
 \includegraphics[width=.9\textwidth]{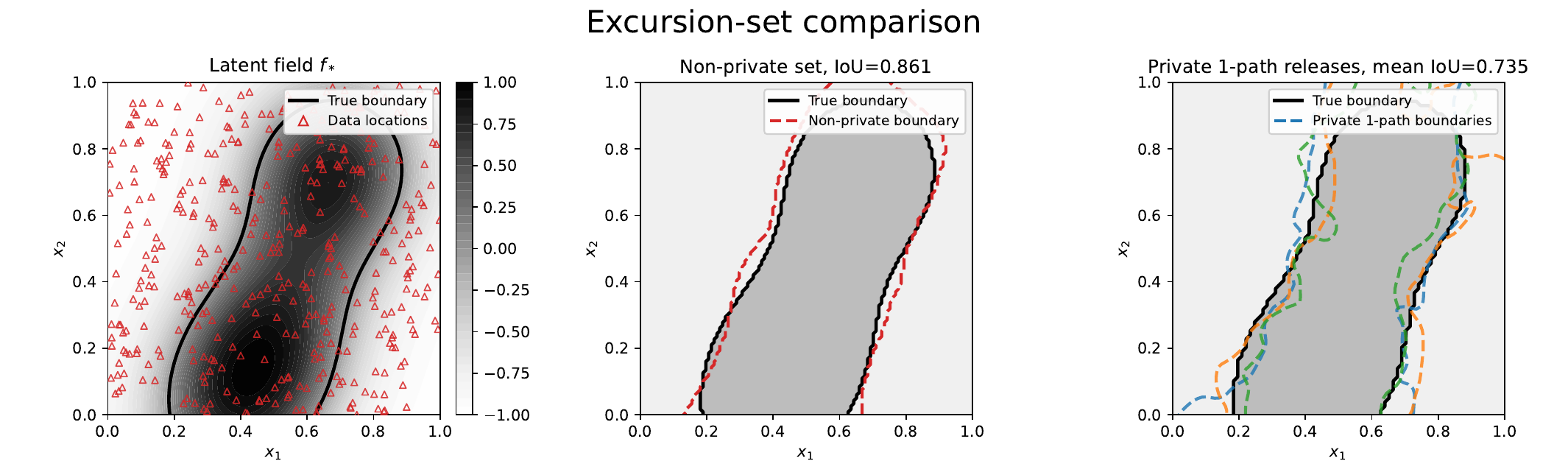}
  \caption{Representative examples of the 2D excursion-set experiment at $\rm NSR=1$. Each row uses one independently sampled latent field $f_*$ and dataset $D$. {\bf Left}: latent field $f_*$ and training data locations with the black curve indicating the true excursion boundary. {\bf Middle}: non-private posterior-probability excursion estimate $\{x:p_D(x)\ge C\}$ overlaid on the true excursion set shown in grey. {\bf Right}: private one-path excursion releases obtained by drawing posterior sample paths, applying a smoothing post-processing step, and thresholding at the natural excursion level. The private $(\varepsilon_L<9,\, \delta=|D|^{-1.1},\, L=1)$ one-path boundaries showcase the DP randomisation and associated IoU tradeoff while still tracking the true excursion boundary well. Dataset size $|D|=400$, $\delta=0.001$.
}
  \label{fig:2d_excursion_NSR1}
\end{figure*}

\begin{figure*}[t]
  \centering
 \includegraphics[width=.9\textwidth]{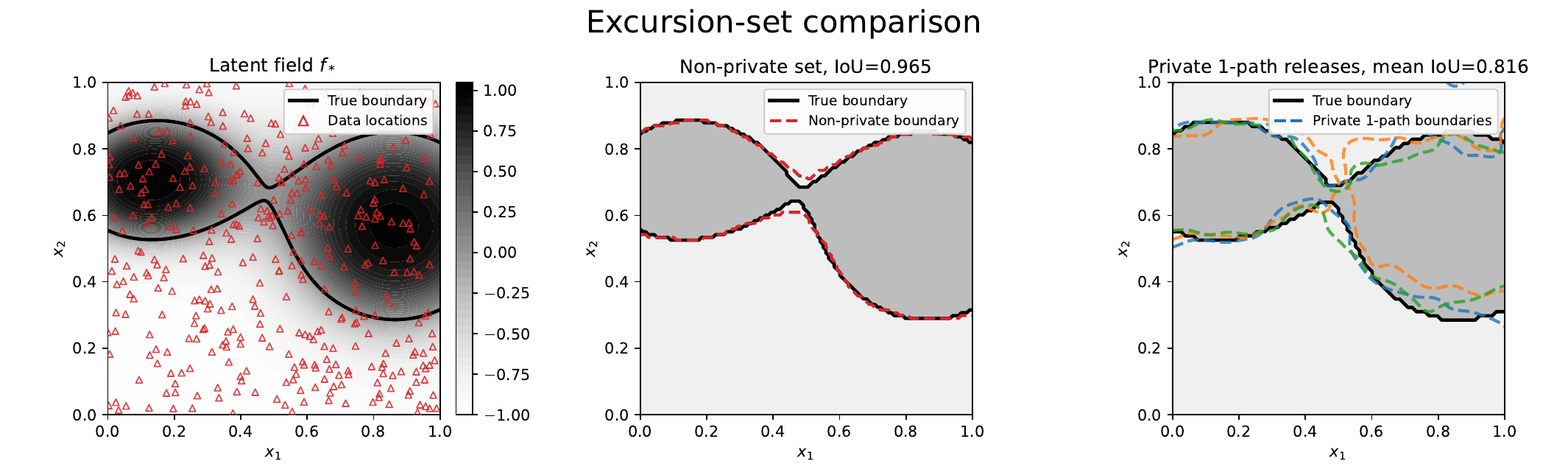}
  \caption{A representative example of the 2D excursion-set experiment at $\rm NSR=0.11$. {\bf Left}: latent field $f_*$ and training data locations with the black curve indicating the true excursion boundary. {\bf Middle}: non-private posterior-probability excursion estimate $\{x:p_D(x)\ge C\}$ overlaid on the true excursion set shown in grey. {\bf Right}: private one-path excursion releases obtained by drawing posterior sample paths, applying a smoothing post-processing step, and thresholding at the natural excursion level. The private $(\varepsilon_L<9,\, \delta=|D|^{-1.1},\, L=1)$ one-path boundaries showcase the DP randomisation and associated IoU tradeoff while still tracking the true excursion boundary well. Dataset size $|D|=400$, $\delta=0.001$.
}
  \label{fig:2d_excursion_lowNSR}
\end{figure*}

\subsection{Greater London property sales study}\label{app:london}
We also evaluate the method on a real spatial dataset derived from the HM Land
Registry Price Paid Data. We restrict attention to Greater London leasehold-flat
transactions in 2018 and join transaction postcodes to a coarse grid. Transactions are aggregated
onto a hexagonal lattice of Greater London and each non-empty hexagon is
assigned the median log sale price of the transactions falling in that cell. For
a fixed threshold $t$, we define the centred response
\[
    y_h^{\rm raw}
    =
    \operatorname{median}_{i\in h}\log P_i - t,
\]
and the corresponding excursion label
\[
    s_h
    =
    \mathbf 1\{y_h^{\rm raw}>0\}.
\]
The GP is fitted to the symmetrically clipped and rescaled response
\[
    y_h
    =
    \frac{
    \operatorname{clip}\left(y_h^{\rm raw},-B,B\right)
    }{B}
    \in[-1,1],
\]
so that the response bound entering the privacy calculation is $M_Y=1$. This
transformation preserves the zero excursion threshold.

We use the same exponential-kernel posterior model as in the synthetic
experiments, with hexagon-centre coordinates normalised to $[0,1]^2$. For the
non-private baseline hyperparameters $(\ell,r,\sigma)$ are selected by
validation binary cross-entropy (using the observed clipped labels $y_h$ derived from the aggregated responses from separate transaction data from 2017), and the probability cutoff $C$ defining the
excursion estimate
\[
    \widehat\Omega_C
    =
    \{x:p_D(x)\ge C\}
\]
is selected by validation IoU. Using marginal negative log-likelihood instead of BCE for hyperparameter selection gave qualitatively similar excursion boundaries. For the private version, the same selection
procedure is constrained to hyperparameter settings satisfying the
R\'enyi-DP-based privacy bound at the prescribed privacy level. The resulting
private visualisations are produced by drawing sample paths from the
DP-constrained posterior and thresholding them at zero.

The resulting privacy guarantee should be interpreted at the level of the
coarsened hexagon records used by the GP mechanism. Exact transaction locations
are discarded before fitting: postcode-level coordinates are assigned to the coarse hexagon cells and the mechanism operates only on the resulting
hexagon-level median responses. In Figure~\ref{fig:london} the non-private
probability boundary $p_D(x)=C$ is compared with three private one-path
excursion boundaries. The private draws recover the dominant high-price region
in central London, but tend to smooth out or miss smaller high-price islands
outside central London, reflecting the additional randomisation and
regularisation induced by the privacy constraint.

\subsection{Diminishing utility tradeoff with the increase of training set size}
\label{app:privacyForFree}

Table~\ref{tab:n_transition} shows that the cost of the privacy constraint 
decreases with sample size in the synthetic-data 1D excursion set experiment. When $n=|D|$
is small, each record has high influence on predictions and enforcing single-digit privacy
requires strong regularisation. At $|D|=50$ imposing $\varepsilon<9$ reduces
$d_{\rm eff}$ from $5.72$ to $1.53$ and yields a relative one-path IoU gap of
$15.1\%$. As $|D|$ grows, the privacy-feasible posterior remains substantially
regularised but loses less downstream utility: the relative BCE increase drops
from $22.9\%$ at $|D|=50$ to $3.3\%$ at $|D|=400$, while the relative IoU gap
drops from $15.1\%$ to $4.2\%$. The randomized excursino set release also becomes more stable:
the one-path IoU standard deviation decreases from $0.082$ to $0.038$ over the
same range. In particular, at $|D|=400$ the privacy-feasible posterior reduces
$d_{\rm eff}$ from $11.1$ to $8.94$, yet the one-path IoU remains high at
$0.88$.
\begin{table}[t]
\centering
\setlength{\tabcolsep}{4pt}
\begin{tabular}{lcccc}
\toprule
 $|D|$ & $50$ & $100$ & $200$ & $400$ \\
\midrule
$\varepsilon$, unconstrained  & $44.7$ & $36.0$ & $32.0$ & $13.6$ \\
$d_{\rm eff}$, unconstrained  & $5.72\,[5.67,5.79]$ & $8.39\,[8.32,8.44]$ & $12.4\,[12.3,12.4]$ & $11.1\,[11.0,11.1]$ \\
$d_{\rm eff}$, $\varepsilon<9$  & $1.53\,[1.52,1.54]$ & $2.18\,[2.17,2.19]$ & $4.12\,[4.11,4.13]$ & $8.94\,[8.93,8.97]$ \\
relative BCE increase  & $22.9\%\,[5.7\%,42.0\%]$ & $12.0\%\,[0.5\%,25.8\%]$ & $5.0\%\,[0.2\%,11.8\%]$ & $3.3\%\,[0.0\%,10.5\%]$ \\
IoU, unconstrained  & $0.845\,[0.753,0.915]$ & $0.871\,[0.812,0.920]$ & $0.897\,[0.843,0.931]$ & $0.926\,[0.877,0.956]$ \\
$1$-path IoU mean, $\varepsilon<9$  & $0.716\,[0.612,0.790]$ & $0.797\,[0.720,0.855]$ & $0.837\,[0.774,0.885]$ & $0.880\,[0.818,0.921]$ \\
$1$-path IoU SD, $\varepsilon<9$ & $0.082\,[0.055,0.109]$ & $0.056\,[0.040,0.080]$ & $0.051\,[0.039,0.069]$ & $0.038\,[0.027,0.049]$ \\
relative IoU  gap  & $15.1\%\,[7.6\%,23.0\%]$ & $8.2\%\,[4.6\%,14.7\%]$ & $5.8\%\,[3.2\%,9.2\%]$ & $4.2\%\,[2.8\%,8.0\%]$ \\
\bottomrule
\end{tabular}
\caption{The effect of varying the training set size on the released 1D excursion set utility at fixed $NSR=1$. Median [interquartile range] over $10^3$ draws of $(D,f_*)$-pairs
across a range of training dataset sizes. The private candidate is the best
average-BCE-selected setting satisfying $\varepsilon<9$ and $\delta=10^{-3}$. The relative IoU gap
is computed with respect to the unconstrained IoU. The $1$-path IoU mean, SD reported for a sample of $50$ posterior draws per each $(D,f_*)$-pair.}
\label{tab:n_transition}
\end{table}

\paragraph{Privacy for free from posterior risk consistency}

The key to understanding why the utility cost of imposing higher regularisation decreases with $n$ is to realise that the GP posterior risk stays consistent even when the effective regularisation grows with $n$, i.e., $r_n\propto n^\gamma$ with appropriate $\gamma>0$. Let $D_n$ be a training set of size $n$ where the covariates are i.i.d. from $P_X$ and let $\Pi_n(\cdot\mid D_n)$ be the posterior distribution of the sampled functions given $D_n$. Denote by $f_*$ the regression function. The posterior risk is defined as 
\[
R_n(D_n):=\EE_{g\sim \Pi_n(\cdot\mid D_n)}\left[\|g-f_*\|^2_{L^2(P_X)}\mid D_n\right]=\int d\Pi_n(g\mid D_n)\int dP_X(\bx)\left(g(\bx)-f_*(\bx)\right)^2.
\]
We define GP \textit{posterior-risk consistency} as the convergence in probability $R_n(D_n)\xrightarrow{n\to \infty} 0$ with respect to a family of data-prior probability measures $P_{XY}$ from which $D_n\sim P_{XY}^{\otimes n}$. GP posterior risk consistency has been proved in the matched-Gaussian case where the responses are generated by adding i.i.d. gaussian noise to $f_*$ which is sampled from a GP prior \citep{JMLR:v12:vandervaart11a}. However, in the literature we have not found a suitable GP posterior risk consistency result stated directly in this paper's exact setup i.e. for bounded-response models and regularisation $r\equiv r_n$ growing with $n$. Fortunately, the desired result is readily available as a direct combination of Theorem 12 from \cite{10.3150/07-BEJ5102} and Theorem 2 from \cite{LeGratiet2015} which we explain below. We work under the following assumptions.
\begin{enumerate}[({A.}1)]
\item The covariate domain $\Omega_X\subset \RR^d$ is compact.
\item The GP kernel $k:\ \RR^d\times\RR^d\to \RR$ is universal (see Definition 11 in \cite{10.3150/07-BEJ5102}), continuous as a function on $\RR^d\times\RR^d$, has bounded diagonal and by this paper's convention it is normalised so that $\sup_{\bx\in\RR^d}k(x,x)=1$.
\item The GP kernel $k$ is non-degenerate, i.e., the Mercer eigenvalues of the integral operator $\hat T_k(f)(\bx):=\int dP_X(\bx)k(\bx,\bx')f(\bx')$ on $L^2(P_X)$ are nonzero.
\item The covariates are i.i.d. from $P_X$ and $y_i=f_*(\bx_i)+\xi_i$ where $f_*\in L^2(P_X)$, $\EE[\xi_i|\bx_i]=0$ and $|\xi_i|\leq M_\xi$.
\end{enumerate}

\begin{proposition}\label{prop:risk_consistency}
    Assume (A.1)-(A.4) and that the effective regularisation grows as $r_n\propto n^\gamma$ with $3/8<\gamma<1/2$. Then, the GP posterior risk is consistent, i.e., $R_n(D_n)\xrightarrow{n\to \infty} 0$ in probability with respect to the training data distribution $P_{XY}^{\otimes n}$.
\end{proposition}
\begin{proof}
    Let $D_n=(X_n,\by_n)$ with $X_n=(\bx_1,\dots,\bx_n)$ and $\by_n=(y_1,\dots,y_n)$. Let us first evaluate the posterior integral
    \[
    \int d\Pi_n(g\mid D_n)\left(g(\bx)-f_*(\bx)\right)^2=(f_*(\bx)-\mu_{D_n}(\bx))^2+\sigma^2\,k_{D_n}(\bx,\bx),
    \]
    where we have used the facts that the mean and variance of $g(\bx)$ are $\mu_{D_n}(\bx)$ and $\sigma^2\,k_{D_n}(\bx,\bx)$ respectively. Consequently,
    \[
    R_n(D_n)=\left\|f_*-\mu_{D_n}\right\|^2_{L^2(P_X)}+\sigma^2\int dP_X(x)k_{D_n}(\bx,\bx).
    \]
    Theorem 12 from \cite{10.3150/07-BEJ5102} deals with the term $\left\|f_*-\mu_{D_n}\right\|^2_{L^2(P_X)}$. As stated originally, the theorem concerns kernel ridge regression (KRR). By the standard equivalence between GP posterior mean and and KRR solution (see e.g. Proposition 3.6 in \cite{kanagawa2018gaussianprocesseskernelmethods}) we have that the function $\mu_{D_n}$ minimises the following KRR loss defined on the RKHS $\mathcal{H}_k$
    \[
    L(f) = \frac{1}{n}\sum_{i=1}^n\left(f(\bx_i)-y_i\right)^2+\eta_n\|f\|_{\mathcal{H}_k}^2,
    \]
    where the correspondence between KRR's $\lambda_n$ (which we denote here by $\eta_n$) and this paper's $r_n$ is 
    \[
    \eta_n = \frac{r_n^2}{n}.
    \]
    Under assumptions (A.1), (A.2) and (A.4), Theorem 12 in \cite{10.3150/07-BEJ5102} states that $\left\|f_*-\mu_{D_n}\right\|^2_{L^2(P_X)}\xrightarrow{n\to\infty} 0$ in probability if $\eta_n\to 0$ and $n\eta_n^4\to \infty$ which if $r_n\propto n^\gamma$ requires that $3/8<\gamma<1/2$. Indeed, the theorem's $p$ is equal to $2$ since the relevant loss is quadratic. What is more, in the theorem's notation
    \begin{align*}
    R_{L,P}=\EE_{(\bx,y)\sim P_{XY}}\left[\left(\mu_{D_n}(\bx)-y\right)^2\right]=\int dP_X(\bx)\left(\mu_{D_n}(\bx)^2-2\mu_{D_n}(\bx)\EE_\xi\left[y\mid X=\bx\right]+\EE_\xi\left[y^2\mid X=\bx\right]\right)
    \\
    =\EE_\xi\left[\xi^2\right]+\int dP_X(\bx)(\mu_{D_n}(\bx)-f_*(\bx))^2=R^*_{L,P}+\left\|f_*-\mu_{D_n}\right\|^2_{L^2(P_X)},
    \end{align*}
    where $R^*_{L,P}=\EE_\xi\left[\xi^2\right]$ is the optimal risk. Theorem 12 in \cite{10.3150/07-BEJ5102} states that $R_{L,P}\to R^*_{L,P}$ in probability which is equivalent to convergence of $\left\|f_*-\mu_{D_n}\right\|^2_{L^2(P_X)}$ to zero in probability. Note also that the theorem's moment condition denoted by $|P|_p<\infty$ (the $p$th moment of the response is required to be finite) is automatically satisfied by the boundedness of the noise and the fact that $f_*\in L^2(P_X)$.

    It remains to show how Theorem 2 of \cite{LeGratiet2015} implies $\int dP_X(x)k_{D_n}(\bx,\bx)\to 0$ in probability with respect to the distribution $P_X^{\otimes n}$. To this end, let us denote
    \[
    k_{D_n,\eta}(\bx,\bx')=k(\bx,\bx')-k(\bx,X_n)\left(K(X_n,X_n)+n\eta I\right)^{-1}k(X_n,\bx').
    \]
    Theorem 2 of \cite{LeGratiet2015} states that if $\eta$ does not depend on $n$ ($r_n=\sqrt{\eta\,n}$) then under (A.3)
    \begin{equation}\label{eq:int_variance_limit}
    I_n(\eta):=\int dP_X(x)k_{D_n,\eta}(\bx,\bx)\xrightarrow{n\to\infty}S(\eta):=\sum_{j\geq 0}\frac{\eta\lambda_j}{\eta+\lambda_j}
    \end{equation}
    in probability with respect to $P_X$ where $\lambda_j$ is $j$th Mercer eigenvalue of the operator $\hat T_k$ from (A.3). The goal is to show that by allowing $\eta_n$ to change with $n$ so that $\eta_n\to 0$ (meaning that $r_n\propto n^\gamma$ with $\gamma<1/2$), then \eqref{eq:int_variance_limit} implies that $I_n(\eta_n)\xrightarrow{P_X} 0$. To see this, note first that $\frac{\eta\lambda_j}{\eta+\lambda_j}\leq \lambda_j$, thus $S(\eta_n)\le\mathrm{tr}\hat T_k<\infty$. This allows us to use Tannery's theorem to interchange limit with summation as follows:
    \begin{equation}\label{eq:S_convergence}
    \lim_{\tau\to0}S(\eta)=\sum_{j\geq 0}\lim_{\tau\to0}\frac{\eta\lambda_j}{\eta+\lambda_j}=0.
    \end{equation}
    What is more, if $\eta'\le\eta$ we have $I_n(\eta')\le I_n(\eta)$. Indeed, $\left(K(X_n,X_n)+n\eta' I\right)^{-1}\succeq \left(K(X_n,X_n)+n\eta I\right)^{-1}$ which implies $k_{D_n,\eta'}(\bx,\bx)\le k_{D_n,\eta}(\bx,\bx)$.
    
    Now fix $\epsilon>0$. Choose $\eta>0$ such that $S(\eta)<\epsilon/2$ by~\eqref{eq:S_convergence}. There exists $n_0$ such that for every $n>n_0$ $\eta_n<\eta$. By the monotonicity of $I_n(\eta)$ w.r.t. $\eta$ we have for every $n>n_0$
    \[
    P_X^{\otimes n}\left[I_n(\eta_n)>\epsilon\right]\leq P_X^{\otimes n}\left[I_n(\eta)>\epsilon\right]\leq P_X^{\otimes n}\left[\left|I_n(\eta)-S(\eta)\right|>\frac{\epsilon}{2}\right]\xrightarrow{n\to\infty}0
    \]
    by \eqref{eq:int_variance_limit} where in the second inequality we have used the fact that $I_n(\eta)>\epsilon$ implies $I_n(\eta)-S(\eta)>\epsilon/2$.
    \end{proof}
    Intuitively, Proposition \ref{prop:risk_consistency} says that the posterior sample paths approximate $f_*$ well as $n\to\infty$ and $r_n\propto n^\gamma$ with $3/8<\gamma<1/2$. Crucially, this growth rate of $r_n$ implies that privacy $\epsilon\to 0$ as $n\to \infty$ even in the generic bounded-repsonse model, thus asymptotically enabling privacy for free.
    \begin{lemma}[Privacy for free asymptotically]
    Let the effective GP regularisation grow as $r_n\propto n^\gamma$ with $3/8<\gamma<1/2$ and assume (A.4) with $\sup_{\bx\in\Omega_X}f_*(\bx)\leq B<\infty$. Then, releasing a single posterior sample path is $(\epsilon(n), \delta(n))$-DP where $\epsilon(n)\xrightarrow{n\to\infty}0$ for any $\delta(n)$ decreasing polynomially with $n$.
    \end{lemma}
    \begin{proof}
    Let us first simplify the RDP bound from Theorem~\ref{thm:main_rdp} on the half-range $1\leq \alpha\le 1+\frac{1}{2\tau_n(r)}$. Using $-\log(1-u)\leq u/(1-u)$ we get
    \[
    -\frac{\log\left(
1-\frac{\alpha(\alpha-1)\tau_n(r)^2}{1+\tau_n(r)}
\right)}{2(\alpha-1)}\leq \frac{\alpha \tau_n(r)^2}{2(1-(\alpha-1)\tau_n(r))(1+\alpha\tau_n(r))}\leq \frac{\alpha \tau_n(r)^2}{1+\tau_n(r)},
    \]
    where in the last inequality we have used $1\leq \alpha\le 1+\frac{1}{2\tau_n(r)}$. Applying similar reasoning to the mean-term of the RDP bound we obtain
    \begin{equation}\label{eq:rdp_simple}
    \varepsilon_\alpha\leq \alpha\, a_n(r_n),\quad a_n(r) = (1+\tau_n(r))\left(\tau_n(r)^2+\left(\frac{\Delta_n(r)}{\sigma}\right)^2\right).
    \end{equation}
    
    Next, bound the sensitivity $\hat\Delta_n(r)$ from Lemma~\ref{lemma:boundedResp_Delta_bound} as follows.
    \[
    \Delta_n(r)\leq \hat\Delta_n(r)=\frac{M_Y\sqrt{n-1}}{r}\frac{V_n(r)}{r^2+V_n(r)}+2M_Y\max_{0\le u\le V_n(r)}\frac{\sqrt{u}}{r^2+u}\leq M_Y\left(\frac{2}{r^2}+\frac{\sqrt{n}}{r^3}\right),
    \]
    where $M_Y=M_\xi+B$. Using $r_n=C n^\gamma$ with $\gamma<1/2$ we get that for sufficiently large $n$, $\Delta_n(r)\le 2 M_Y\sqrt{n}/r^3$. Plugging this to \eqref{eq:rdp_simple} and using further simplifications we get that for sufficiently large $n$
    \[
    \varepsilon_\alpha\leq \alpha \frac{2}{C^6}
\left(1+\frac{4M_Y^2}{\sigma^2}\right)
n^{1-6\gamma}.
    \]
    Note that under the assumption $\gamma>3/8$ putting $\alpha_n=\sqrt{n}$ is a valid choice of RDP order for sufficiently large $n$. Then, the RDP-to-DP conversion from \cite{rdp} yields
    \begin{align}
\varepsilon_n^{\mathrm{DP}}
&\leq
\varepsilon_{\alpha_n}^{\mathrm{RDP}}
+
\frac{\log(1/\delta(n))}{\alpha_n-1}\
\leq
\frac{2}{C^6}
\left(1+\frac{4M_Y^2}{\sigma^2}\right)
n^{3/2-6\gamma}
+
\frac{\log(1/\delta(n))}{\sqrt n-1}.
\end{align}
    Since $\gamma>3/8$, the first term converges to zero. The second term also converges to zero whenever $\delta(n)$ decreases polynomially with $n$.
\end{proof}

Finally, we show that posterior risk consistency implies that one-path IoU tends to $1$ in probability, therefore asymptotically providing privacy for free in the excursion-set experiments.
\begin{lemma}
    Assume the GP posterior risk is consistent, i.e., $R_n(D_n)\xrightarrow{n\to \infty} 0$ in probability with respect to the training data distribution $P_{XY}^{\otimes n}$. Assume also that $P_X$ has density $\rho_X$ which is bounded from below by $\rho_{\min}>0$, the level set $\{\bx:\ f_*(\bx)=t\}$ is of measure zero and that the true excursion set $S_* = \{\bx:\ f_*(\bx)\geq t\}$ has nonzero volume. Then, posterior-expected one-path IoU of the excursion sets tends to $1$ in probability, i.e.
    \[
    \EE_{g\sim\Pi_n(\cdot\mid D_n)}\left[\mathrm{IoU}(f_*,g)\right]=\EE_{g\sim\Pi_n(\cdot\mid D_n)}\left[\frac{\left|\{\bx:\ f_*(\bx)\geq t\mathrm{\ and\ } g(\bx)\geq t\}\right|}{\left|\{\bx:\ f_*(\bx)\geq t\mathrm{\ or\ } g(\bx)\geq t\}\right|}\right]\xrightarrow[P_{XY}]{n\to\infty }1.
    \]
\end{lemma}
\begin{proof}
    Denote $S_g = \{\bx:\ g(\bx)\geq t\}$. Using the disjoint decomposition 
    \[
    S_g\cup S_*=(S_g\cap S_*)\sqcup \left[(S_g\setminus S_*)\cup (S_*\setminus S_g) \right]
    \]
    we get
    \begin{equation}\label{eq:1minusIoU}
    1-\mathrm{IoU}(S_g,S_*)=\frac{|S_g\cup S_*|-|S_g\cap S_*|}{|S_g\cup S_*|}=\frac{|(S_g\setminus S_*)\cup (S_* \setminus S_g)|}{|S_g\cup S_*|}\leq \frac{|(S_g \setminus S_*)\cup (S_* \setminus S_g)|}{|S_*|}.
    \end{equation}
    Next, we note that for any $u>0$ we have
    \[
    (S_g \setminus S_*)\cup (S_* \setminus S_g)\subset S_u:=\{\bx:\ |f_*(\bx)-t|\leq u\}\cup \{\bx:\ |f_*(\bx)-g(\bx)|> u\}.
    \]
    To see this, fix $u>0$ and assume $\bx\in S_g \setminus S_*$, i.e., $g(\bx)\geq t$ and $f_*(\bx) < t$. Then, we have either $|f_*(\bx)-t|\leq u$ or $|f_*(\bx)-t|> u$. If $|f_*(\bx)-t|\leq u$ then $\bx\in S_u$. If $|f_*(\bx)-t|> u$ then because $f_*(\bx) < t$ we necessarily have $t-f_*(\bx)>u$ and since $g(\bx)\geq t$
    \[
    g(\bx)-f_*(\bx)\geq t-f_*(\bx) > u,
    \]
    i.e., $|g(\bx)-f_*(\bx)|>u$ and so $\bx \in S_u$. Analogous resoning shows that if $\bx\in S_*/S_g$ then also $\bx \in S_u$.

    Since the density $\rho_X(\bx)$ is bounded from below by $\rho_{\min}$ we have for any measurable set $|S|\leq P_X[S]/\rho_{\min}$. By the definition of $S_u$, we have 
    \[
    |S_u|\leq |\{\bx:\ |f_*(\bx)-t|\leq u\}|+|\{\bx:\ |f_*(\bx)-g(\bx)|> u\}|.
    \]
    Furthermore,
    \[
    |\{\bx:\ |f_*(\bx)-g(\bx)|> u\}|\leq \frac{1}{\rho_{\min}}P_X\left[|g(\bx)-f_*(\bx)|>u\right] \leq \frac{1}{u^2 \rho_{\min}}\|g-f_*\|_{L^2(P_X)}^2,
    \]
    where in the second step we have used Markov's inequality $P[Z>u^2]\leq E[Z]/u^2$ for $Z=|g(X)-f_*(X)|^2$. Thus, we get
    \[
    \EE_{g\sim\Pi_n(\cdot\mid D_n)}\left[1-\mathrm{IoU}(S_g,S_*)\right]\leq \frac{1}{|S_*|}\EE_{g\sim\Pi_n(\cdot\mid D_n)}\left[|S_u|\right]\leq\frac{1}{|S_*|}\left(|\{\bx:\ |f_*(\bx)-t|\leq u\}|+\frac{1}{u^2 \rho_{\min}}R_n(D_n)\right).
    \]
    By the posterior risk consistency, the second term tends to zero in probability. The first term can be made arbitrarily small by taking $u$ small enough since we're assuming that the level set $\{\bx:\ f_*(\bx)=t\}$ is of measure zero.
\end{proof}

\end{document}